%% file: main.tex
\documentclass[sigconf]{acmart}

\AtBeginDocument{%
  }

\setcopyright{none}
\copyrightyear{2026}
\acmYear{2026}
\acmDOI{XXXXXXX.XXXXXXX}
\acmConference[Conference acronym 'XX]{Make sure to enter the correct
  conference title from your rights confirmation email}{June 03--05,
  2018}{Woodstock, NY}
\acmISBN{978-1-4503-XXXX-X/2018/06}
\renewcommand\footnotetextcopyrightpermission[1]{}

\usepackage{mathtools}
\usepackage{amsthm}
\usepackage{subcaption}
\usepackage{multirow}
\usepackage{colortbl}
\usepackage{framed}
\usepackage{fvextra}
\usepackage{algorithm}
\usepackage{algorithmic}
\usepackage{float}
\usepackage[capitalize,noabbrev]{cleveref}
\usepackage{makecell}
\usepackage{pifont}
\usepackage{xspace}
\newcommand{\Require}{\REQUIRE}

\newcommand{\State}{\STATE}
\newcommand{\Return}{\RETURN}
\newcommand{\Comment}[1]{\COMMENT{#1}}
\newcommand{\For}[1]{\FOR{#1}}
\newcommand{\EndFor}{\ENDFOR}
\newcommand{\If}[1]{\IF{#1}}
\newcommand{\EndIf}{\ENDIF}
\newcommand{\Else}{\ELSE}
\newcommand{\Repeat}{\REPEAT}
\newcommand{\Until}[1]{\UNTIL{#1}}
\theoremstyle{plain}
\newtheorem{theorem}{Theorem}[section]
\newtheorem{proposition}[theorem]{Proposition}

\newtheorem{corollary}[theorem]{Corollary}
\theoremstyle{definition}
\newtheorem{definition}[theorem]{Definition}
\newtheorem{assumption}[theorem]{Assumption}
\theoremstyle{remark}

\definecolor{cpagentblue}{rgb}{0.9,1.0,1.0}
\definecolor{promptbg}{RGB}{245,245,245}
\definecolor{takeawaybg}{HTML}{F6E3BA}

\newcommand{\promptboxsep}{4pt}
\newcommand{\promptboxfont}{\ttfamily\fontsize{7pt}{8pt}\selectfont}
\newcommand{\promptverbfont}{\scriptsize}
\DefineVerbatimEnvironment{verbatim}{Verbatim}{
  fontsize=\promptverbfont,
  breaklines=true,
  breakanywhere=true,
  breakindent=0pt,
  breakautoindent=false,
  breaksymbolleft={}
}
\makeatletter
\newenvironment{promptboxapp}[1][]{%
  \MakeFramed{\advance\hsize-\width\FrameRestore}%
  \promptboxfont
  \setlength{\parindent}{0pt}%
  \setlength{\emergencystretch}{2em}%
  \raggedright\sloppy}
  {\endMakeFramed}
\newenvironment{takeawaybox}{%
  \par\noindent\begin{minipage}{\linewidth}%
  \MakeFramed{\advance\hsize-\width\FrameRestore}%
  \small\setlength{\parindent}{0pt}}
  {\endMakeFramed\end{minipage}\par}
\makeatother

\newcommand{\model}{\textsc{CP-Agent}\xspace} 

\newcommand{\cfg}[4]{%
{\scriptsize
\ifnum#1=1\textcolor{green!55!black}{\ding{51}}\else\textcolor{red!65!black}{\ding{55}}\fi
\ifnum#2=1\textcolor{green!55!black}{\ding{51}}\else\textcolor{red!65!black}{\ding{55}}\fi
\ifnum#3=1\textcolor{green!55!black}{\ding{51}}\else\textcolor{red!65!black}{\ding{55}}\fi
\ifnum#4=1\textcolor{green!55!black}{\ding{51}}\else\textcolor{red!65!black}{\ding{55}}\fi
}}
\newcommand{\rt}[2]{#1{\scriptsize\,(#2)}}

\newcommand{\ind}{\mathbf{1}}
\newcommand{\clip}[1]{\bigl[#1\bigr]_{+}}

\title{\model: A Calibrated Risk-Controlled Agent for Feedback-Driven Competitive Programming}
\author{Peisong Wang}
\affiliation{%
  \institution{The Hong Kong University of Science and Technology (Guangzhou)}
  \city{Guangzhou}
  \country{China}}
\author{Bowen Liu}
\affiliation{%
  \institution{The Hong Kong University of Science and Technology (Guangzhou)}
  \city{Guangzhou}
  \country{China}}
\author{Zehua Li}
\affiliation{%
  \institution{The Hong Kong University of Science and Technology (Guangzhou)}
  \city{Guangzhou}
  \country{China}}
\author{Yuyao Wang}
\affiliation{%
  \institution{The Hong Kong University of Science and Technology (Guangzhou)}
  \city{Guangzhou}
  \country{China}}
\author{Zhiwei Ma}
\affiliation{%
  \institution{The Hong Kong University of Science and Technology (Guangzhou)}
  \city{Guangzhou}
  \country{China}}
\author{Yuhan Li}
\affiliation{%
  \institution{The Hong Kong University of Science and Technology (Guangzhou)}
  \city{Guangzhou}
  \country{China}}
\author{Jia Li}
\authornote{Corresponding author.}
\affiliation{%
  \institution{The Hong Kong University of Science and Technology (Guangzhou)}
  \city{Guangzhou}
  \country{China}}
\renewcommand{\shortauthors}{Wang et al.}

\begin{document}

\begin{abstract}
Large language models still struggle with contest-level programming, while many agentic remedies rely on massive inference-time sampling or expensive multi-stage post-training. We study when execution feedback reliably helps an LLM CP solver and which mechanisms govern the gains. We model feedback-driven solving as a calibrated stopped process and identify three quantities: false-admission risk, program-level evidence against bad programs, and the active-state success hazard. Under held-out trace calibration and selection from a pre-declared finite controller manifest, the resulting structural certificate lower-bounds the clean success probability before false admission. We instantiate mechanisms targeting these quantities as Dual-Granularity Verification, Test Augmentation, and Experience-Driven Self-Evolving, yielding \model. Without updating any parameters, \model raises Pass@1 from 25.8\% to 48.5\% on LiveCodeBench Pro and improves Refine@5 by 11.0\% on ICPC-Eval. Across three LLM backbones, \model lies on the cost--accuracy efficiency frontier, and ablations show that each component primarily affects its corresponding certificate quantity.
\end{abstract}



\maketitle

\begin{figure}[t]
    \centering
    \includegraphics[width=\linewidth]{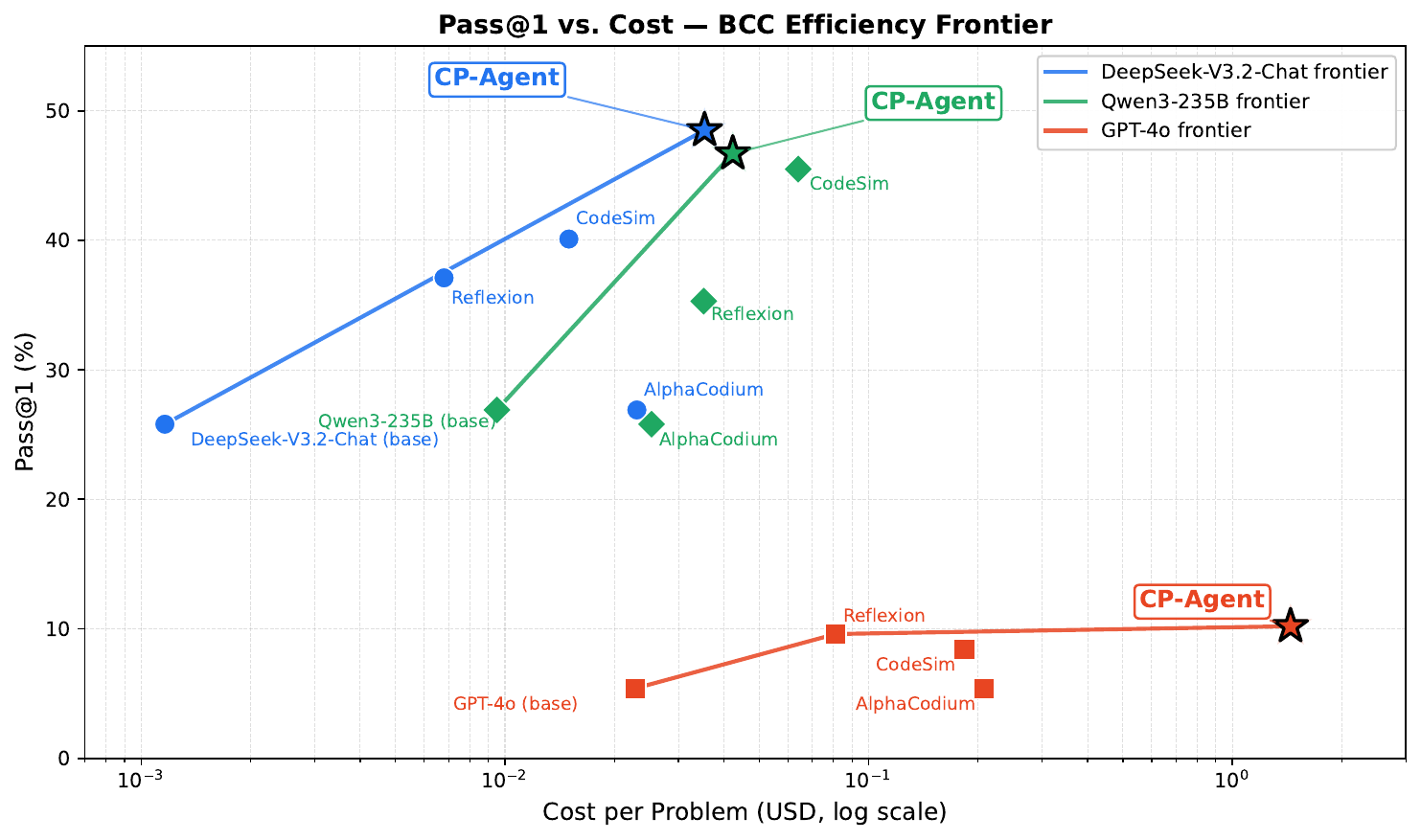}
    \caption{Multi-backbone Pass@1 versus cost on LiveCodeBench Pro for DeepSeek-V3.2-Chat, GPT-4o, and Qwen3-235B-A22B-Instruct, with the VRS/BCC (DEA) efficiency frontier shown as the upper envelope of each backbone. Cost is the average API cost per problem on a log scale. \model lies on the efficiency frontier for all three backbones.}
    \Description{A scatter plot with average API cost per problem on a logarithmic x-axis ranging from about 0.001 to 0.05 US dollars, and Pass@1 accuracy on the y-axis ranging from 0 to about 55 percent. Points are grouped into three backbones using different colors and marker shapes: DeepSeek-V3.2-Chat, GPT-4o, and Qwen3-235B-A22B-Instruct. For each backbone, a piecewise-linear upper-envelope curve connects its most cost-efficient methods. Within every backbone the CP-Agent point sits on this envelope in the upper-right region; baselines such as Reflexion lie on the envelope only for DeepSeek-V3.2-Chat and fall strictly below it for the other backbones, indicating that CP-Agent attains the highest Pass@1 at comparable per-problem cost across all three backbones.}
    \label{fig:efficiency}
\end{figure}
\section{Introduction}
Large language models (LLMs) now perform well across a broad range of downstream tasks, including code generation and everyday software-engineering assistance. Contest-level competitive programming (CP), however, remains a sharp failure case: on ICPC/IOI-style benchmarks~\cite{jain2025livecodebench,xu2026icpc,zheng2026livecodebench}, even strong LLMs reach below 30\% Pass@1 (Section~\ref{sec:experiment}). CP is judged by hidden correctness and asymptotic feasibility under strict time and memory limits, so plausible implementations earn no credit when the algorithm or an edge case is wrong, and software-engineering virtues such as maintainability or modular design do not compensate.

Recent work has explored two directions: (i) tackling competitive programming problems through massive code generation combined with heuristic-based inference and execution-based pruning, exemplified by AlphaCode~\cite{li2022competition}, AlphaCode2~\cite{leblond2023alphacode}, and o1-ioi~\cite{el2025competitive}; (ii) enhancing the intrinsic reasoning ability of LLMs, including large reasoning models~\cite{guo2025deepseek,el2025competitive,openai2025_o4mini_high,liu2025deepseek} that are often trained with reinforcement learning. While these directions improve performance, they often come with substantial computational cost-large-scale sampling during inference, or expensive rollouts during training.

CP offers rich, executable feedback: compilation errors, test failures, and counterexamples directly expose concrete failure modes~\cite{shinn2023reflexion}. This makes agentic solving a natural fit. Two gaps, however, remain for agentic CP\@. First, the literature disagrees on which mechanisms reliably help. MapCoder~\cite{islam2024mapcoder} reports that additional test-case generation is unhelpful while AlphaCodium~\cite{ridnik2024code} reports gains; MapCoder and Oi-Assistant~\cite{yuhua2025oi} find retrieval of similar problems useful, while AlgoSimBench~\cite{li2025algosimbench} finds it unstable. Second, recent strong agentic CP systems increasingly invest in multi-stage agentic post-training~\cite{team2026grandcode}. In contrast, for training-free LLM agents on CP, it remains unclear whether tool orchestration alone improves performance and which mechanisms drive its gains. 

These gaps converge to one essential question: which controllable mechanisms turn execution feedback into measurable gains on contest-level CP? We answer it through a calibrated risk-control view of feedback-driven solving.
We analyze feedback-driven CP solving as a calibrated, risk-controlled stopped process. Its state reduces to three certificate quantities: the residual bad-candidate envelope $\rho_t$, the accumulated program-level evidence $I_t$, and the active-state success hazard $\underline h_t$. Each is moved by exactly one abstract action class: $\textsf{RiskProbe}$ on $\rho_t$, $\textsf{EvidenceAcquire}$ on $I_t$, and $\textsf{ContextAcquire}$ on $\underline h_t$. The calibrated false-admission UCB $\bar q_t \approx \bar\rho_t \exp(-\mathcal I_t)$ follows as a derived quantity. This one-to-one reduction yields a unified design-and-selection framework: for any pre-declared finite manifest of controllers, it supplies a post-selection-valid lower bound on clean-before-false success.
Guided by this framework, we design $\model$ with one controller per action class: Dual-Granularity Verification contracts $\rho_t$, Test Augmentation adds to $I_t$, and Experience-Driven Self-Evolving raises $\underline h_t$. A frozen admission gate and a refinement loop complete the agent.
Extensive empirical results across three backbones on two contest-level benchmarks, together with theory-grounded mechanism ablations on the five-controller calibration class, jointly validate $\model$'s effectiveness and its calibrated certificate.

Our contributions are threefold:
\begin{itemize}
    \item \textbf{A calibrated risk-control theory for feedback-driven CP.}
    We prove an admission-risk theorem and a simultaneously valid finite-manifest selection rule over five abstract action classes (Section~\ref{sec:theory}).
    \item \textbf{A theory-induced agent.} We instantiate this controller as \model, a tool-using framework with no parameter updates (Section~\ref{sec:method}).\footnote{Our implementation is open-sourced at \url{https://github.com/NineAbyss/CP-Agent}.}
    \item \textbf{Trace-calibrated mechanism evidence.} On LiveCodeBench Pro and ICPC-Eval across three backbones, \model raises Pass@1 from 25.8\% to 48.5\% and Refine@5 by $+11.0$\%; calibration attributes the gains to TA (evidence), HP (bad-plan survival), and Experience (success hazard) (Section~\ref{sec:experiment}).
\end{itemize}

\begin{figure*}[t]
    \centering
    \includegraphics[width=\textwidth]{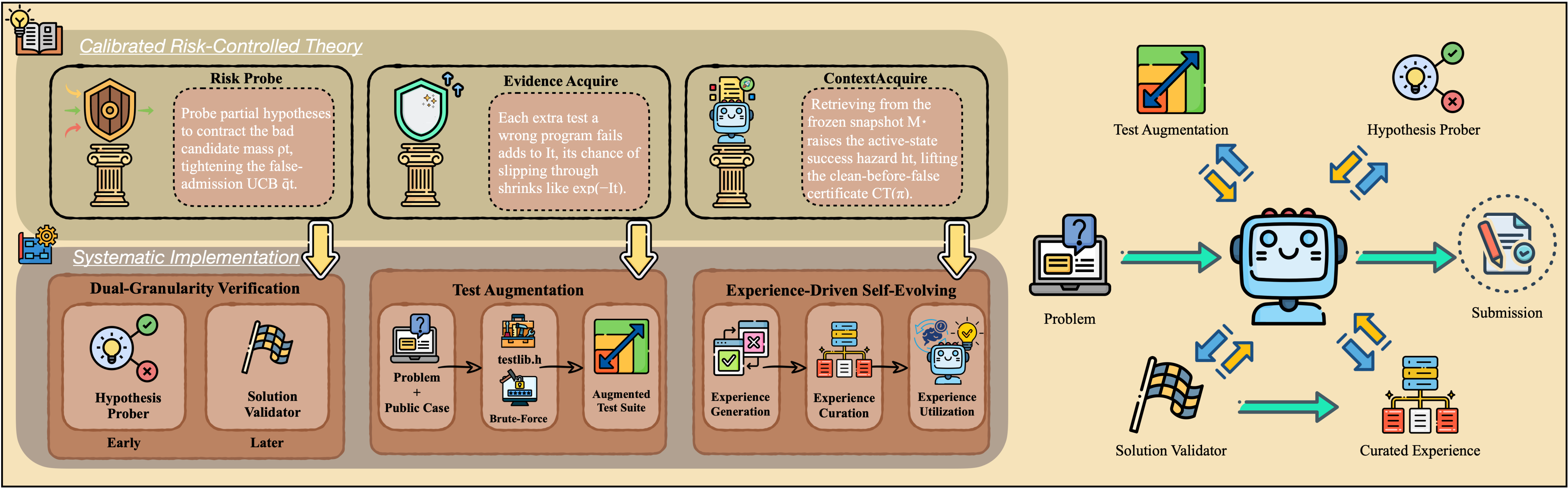}
    \caption{Overview of \model. The risk-control theory defines a certificate state and abstract action classes; the concrete \model controller is declared before target testing and audited on held-out traces through the calibrated lower-bound objective. \model implements these handles with Dual-Granularity Verification, Test Augmentation, and Experience-Driven Self-Evolving.
}
    \Description{A wide block diagram with three horizontal layers connected by arrows. The top layer is the theoretical certificate state, listing the three quantities rho_t (residual bad-candidate envelope), I_t (program-level evidence in nats), and underline h_t (clean-success hazard), each linked to one abstract action class: RiskProbe acts on rho_t, EvidenceAcquire acts on I_t, and ContextAcquire acts on underline h_t. The middle layer is the CP-Agent controller, with three modules drawn as boxes: a Hypothesis Prober that compiles and runs short C++ snippets on the left, a Solution Validator and Test Augmentation pair in the middle that runs candidates against public tests and consensus-filtered augmented tests, and an Experience-Driven Self-Evolving memory bank on the right that retrieves frozen Fail-to-Pass triplets. Arrows from each middle-layer module point up to its matching certificate quantity. The bottom layer shows the temporal data-split pipeline as three boxes in a row: D_hist_mem for memory construction, D_cal for held-out calibration and controller selection, and D_test for target evaluation, with horizontal arrows indicating that the controller manifest and the memory snapshot M-star are frozen before D_cal and D_test.}
    \label{fig:framework}
\end{figure*}

\section{A Calibrated Risk-Controlled Theory of Feedback Control}
\label{sec:theory}
A feedback-driven CP agent repeatedly faces the same decision: admit the current candidate, refine it, or acquire more information. We specify certificate quantities that audit this decision, abstract action classes that act on them, and a calibration rule that turns them into a frozen controller; the takeaway box summarizes the role of the analysis and Appendix~\ref{app:glossary} consolidates the Section~\ref{sec:theory} notation.

\begin{definition}[Controlled stopped process and certificate state]
\label{def:stopped-process}
Fix a horizon $T$, false-admission budget $\alpha$, and a frozen policy $\pi$.  Throughout, ``step $t$'' denotes a \emph{refinement step}---one response--revision attempt (one $\textsf{Refine}$ advance, equivalently one outer-loop iteration of Algorithm~\ref{alg:orchestration}), matched to Refine@$K$ and Pass@$1$. Within-trajectory interaction events such as individual LLM outputs or tool calls are termed \emph{interaction steps} and appear only in diagnostic figures (e.g., Figure~\ref{fig:verification_tool_freq}). Before refinement step $t$, the trajectory is active when
$A_t=\{\tau_S>t-1,\tau_F>t-1\}$.  On active trajectories, the implementation state is $X_t$ (problem text, candidate code, feedback, retrieved memory, and tool outputs).  The certificate exposes a lower-dimensional summary state
\[
Z_t:=\phi(X_t)=(\rho_t,I_t,\underline h_t,t),
\]
where $\rho_t$ is an upper envelope on residual bad-candidate mass, $I_t$ is accumulated program-level evidence, and $\underline h_t$ is a lower envelope on clean-success hazard.  The residual admission-risk score is
\begin{equation}
\label{eq:risk-score}
r_t(Z_t):=\rho_t\exp(-I_t).
\end{equation}
Here $\rho_t$ upper-bounds the bad-candidate probability after upstream screening and $\exp(-I_t)$ upper-bounds the probability that such a candidate survives the evidence gate, so their product upper-bounds false admission; $I_t$ is in nats so evidence accumulates additively as $I_t=\sum_{s\le t}\Delta I_s$.
Let $\Gamma_t\in\{0,1\}$ denote the controller's observable admission gate at step $t$, the binary admit/reject signal logged on every active trace. If the controller admits a candidate through gate $\Gamma_t$, hidden adjudication defines
\[
F_t:=\ind\{A_t,\Gamma_t=1,B_t^\epsilon\},\qquad
S_t:=\ind\{A_t,\Gamma_t=1,\neg B_t^\epsilon\},
\]
where $B_t^\epsilon$ denotes the event that the admitted candidate fails the deterministic hidden-judge on at least one hidden test, which we treat as the $\epsilon$-significant error event under the benchmark's hidden-test distribution. The stopping times are $\tau_F:=\inf\{t:F_t=1\}$ and $\tau_S:=\inf\{t:S_t=1\}$; the certified event is $\{\tau_S\le T,\tau_S<\tau_F\}$.
\end{definition}

\begin{definition}[Abstract action classes]
\label{def:abstract-actions}
We choose action classes to span the certificate: each certificate quantity in $Z_t$ admits one action that can move it in the direction of a larger lower bound---$\textsf{RiskProbe}$ on $\rho_t$, $\textsf{EvidenceAcquire}$ on $I_t$, and $\textsf{ContextAcquire}$ on $\underline h_t$---together with $\textsf{Stop}$ and $\textsf{Refine}$ for the commit/advance decision.  The controller chooses actions from
\begin{equation}
\label{eq:abstract-actions}
\begin{aligned}
\mathcal A(Z_t)=\{&\textsf{Stop},\textsf{Refine},\textsf{RiskProbe},\\
&\textsf{EvidenceAcquire}(m),\textsf{ContextAcquire}(k)\}.
\end{aligned}
\end{equation}
Each action acts on the certificate as follows:
\begin{align}
\textsf{Stop}:\ &\Gamma_t^\alpha=\ind\{r_t(Z_t)\le \hat\theta_{t,\alpha}\}, \label{eq:stop-action}\\
\textsf{RiskProbe}:\ &\rho_{t+}\le \lambda_t(Z_t)\,\rho_t,\quad \lambda_t(Z_t)\in[0,1], \label{eq:riskprobe-action}\\
\textsf{EvidenceAcquire}(m):\ &I_{t+}=I_t+\Delta I_t(m,Z_t), \label{eq:evidence-action}\\
\textsf{ContextAcquire}(k):\ &\underline h_{t+}\ge \underline h_t+\Delta h_t(k,Z_t), \label{eq:context-action}
\end{align}
where $\rho_{t+},I_{t+},\underline h_{t+}$ denote the post-action certificate components, and $\textsf{Refine}$ advances to the next candidate distribution without admitting the current candidate. The integers $m$ and $k$ are action intensities, e.g., evidence budget and context-retrieval budget.
\end{definition}

\begin{assumption}[Held-out exchangeability]
\label{assump:trace-calibration}
Calibration and evaluation problem trajectories are independent draws from the same benchmark distribution.  Independence is required across problem trajectories at a fixed step $t$; it is \emph{not} required across refinement steps within a trajectory.
\end{assumption}

\begin{definition}[Controller manifest]
\label{def:controller-manifest}
The controller manifest is
\[
\begin{aligned}
\pi:=\bigl(\Gamma^{\rm obs}_{1:T},\,\mathcal U_{\rm probe},\,m,\,\gamma_{\rm evid},\,k_{\rm gen},\,\ell_{\rm alg},\,R_{\rm route},\,D_{\rm dedup},\\
p_{\rm prompt},\,d_{\rm decode},\,T,\,M^\star,\,\omega_{\rm write}\bigr)
\end{aligned}
\]
containing: the observable admission gate $\Gamma^{\rm obs}_{1:T}$; the $\textsf{RiskProbe}$ tool interface $\mathcal U_{\rm probe}$ and its model-controlled invocation policy; the $\textsf{EvidenceAcquire}$ intensity $m$ and aggregation rule $\gamma_{\rm evid}$; the two $\textsf{ContextAcquire}$ retrieval channels---general-bank injection budget $k_{\rm gen}$ and algorithmic-bucket retrieval budget $\ell_{\rm alg}$ (so the action intensity $k$ in Eq.~\eqref{eq:context-action} is realized as the pair $(k_{\rm gen},\ell_{\rm alg})$); the retrieval routing $R_{\rm route}$; the deduplication rule $D_{\rm dedup}$; the prompt template $p_{\rm prompt}$; the decoding parameters $d_{\rm decode}$; the horizon $T$; the memory snapshot $M^\star$; and the write mode $\omega_{\rm write}$. For every certificate-bearing result, $\omega_{\rm write}=\textsf{audit-only}$, i.e., $M^\star$ is read-only and evaluation-time experiences are written only to the audit log $\widetilde M$. Each entry of $\pi$ is either pre-declared or selected on held-out traces, and all entries are frozen before final evaluation under the split $D_{\rm hist}^{\rm mem}\to D_{\rm cal}\to D_{\rm test}$ with $M^\star$ built from $D_{\rm hist}^{\rm mem}$.
\end{definition}

\paragraph{Calibration estimators.}
For each fixed $\pi$ and step $t$, let $n_t^\pi$ active calibration traces contain $f_t^\pi$ false admissions and $s_t^\pi$ clean successes.  Under Assumption~\ref{assump:trace-calibration} the counts are binomial conditional on $n_t^\pi$, so we use the exact one-sided Clopper--Pearson intervals~\cite{clopper1934use}
\begin{equation}
\label{eq:qh-calibration}
\bar q_{t,\mathrm{raw}}^\pi:=\mathrm{UCB}_{\rm Binom}\!\left(f_t^\pi,n_t^\pi,\tfrac{\delta}{2T}\right),\quad
\underline h_t^\pi:=\mathrm{LCB}_{\rm Binom}\!\left(s_t^\pi,n_t^\pi,\tfrac{\delta}{2T}\right)
\end{equation}
on the deployment active-state hazards $q_t^\pi:=\Pr_\pi(F_t{=}1\mid A_t)$ and $h_t^\pi:=\Pr_\pi(S_t{=}1\mid A_t)$, jointly valid for $t\le T$ by a union bound (proof of Theorem~\ref{thm:finite-class-selection}); operationally, both bounds share the same active-row denominator $n_t^\pi$ on $D_{\rm cal}$.

\begin{theorem}[Calibrated admission gate]
\label{thm:risk-gate}
Let $\Theta_t$ be a finite grid of admission thresholds for the risk score in Eq.~\eqref{eq:risk-score}.  On held-out calibration traces, let $n_t(\theta)$ be the number of active candidate admissions with $r_t(Z_t)\le\theta$, and let $f_t(\theta)$ be the number of such admissions that hidden evaluation rejects.  Define
\[
\bar R_t(\theta)=\mathrm{UCB}_{\rm Binom}\!\left(f_t(\theta),n_t(\theta),\frac{\delta_{\mathrm{gate}}}{T|\Theta_t|}\right)
\]
and select
\begin{equation}
\label{eq:theta-select}
\hat\theta_{t,\alpha}=\max\{\theta\in\Theta_t:\bar R_t(\theta)\le\alpha_t\},
\end{equation}
with $\hat\theta_{t,\alpha}{:=}-\infty$, $\Gamma_t^\alpha{\equiv}0$ if the set is empty, and the conservative conventions $\mathrm{UCB}_{\rm Binom}(0,0,\eta){:=}1$ and zero-probability conditioning evaluated as $0$.  Then, with probability at least $1-\delta_{\mathrm{gate}}$ over calibration, the gate $\Gamma_t^\alpha=\ind\{r_t(Z_t)\le\hat\theta_{t,\alpha}\}$ satisfies
\[
\Pr(F_t=1\mid A_t,\Gamma_t^\alpha=1)\le\alpha_t
\]
simultaneously for all $t\le T$ and all data-dependent $\hat\theta_{t,\alpha}\in\Theta_t$.
\end{theorem}

\paragraph{Controller certificate.}
Theorem~\ref{thm:risk-gate} controls risk-threshold stopping when a score threshold is part of the manifest, and audits the frozen observable gate otherwise.  To select information-acquisition actions, we compare two separately calibrated UCBs on the active false-admission hazard --- the \emph{raw} bound from admitted-failure counts and the \emph{mechanism-factorized} bound that multiplies probe-survival and evidence-gate-survival rates --- and take their minimum (the mechanism factor wins when the probe and evidence channel together reject most of the bad mass; the raw bound wins when calibration data is sparse, since unfit mechanism factors fall back to $1$):
\begin{align}
\bar q_{t,\mathrm{mech}}^\pi
&=\bar\rho_{t}^{\mathrm{probe},\pi}\exp\!\bigl[-\mathcal I_{t}^{\mathrm{evid},\pi}\bigr],\nonumber\\
\bar q_{t,\mathrm{ctrl}}^\pi
&=\min\!\left\{\bar q_{t,\mathrm{raw}}^\pi,\,\bar q_{t,\mathrm{mech}}^\pi\right\}.
\label{eq:false-risk}
\end{align}
Here $\bar q_{t,\mathrm{raw}}^\pi$ is the Clopper--Pearson UCB on $q_t^\pi$ from active calibration trajectories.  For the mechanism factor, let $W_t^{\rm probe},W_t^{\rm evid}\in\{0,1\}$ denote survival through the upstream risk-probe stage and the evidence gate, respectively.  Since $\{F_t{=}1\}\subseteq\{B_t^\epsilon,$ $W_t^{\rm probe}{=}1,$ $W_t^{\rm evid}{=}1\}$ on $A_t$, the chain rule of conditional probabilities (not independence) gives
{\footnotesize
\[
\Pr_\pi(F_t{=}1\mid A_t)\le
\underbrace{\Pr_\pi(B_t^\epsilon,W_t^{\rm probe}{=}1\mid A_t)}_{\le\,\bar\rho_t^{\mathrm{probe},\pi}\ \text{(Prop.~\ref{prop:probe-odds})}}
\cdot
\underbrace{\Pr_\pi(W_t^{\rm evid}{=}1\mid B_t^\epsilon,W_t^{\rm probe}{=}1,A_t)}_{\le\,\exp[-\mathcal I_t^{\mathrm{evid},\pi}]\ \text{(Prop.~\ref{prop:evidence-acquire})}}.
\]
}
We split $\delta=\delta_{\mathrm{raw}}+\delta_{\mathrm{gate}}+\delta_{\mathrm{probe}}+\delta_{\mathrm{evid}}$, allocating $\delta_{\mathrm{raw}}/(2T)$ to each of $\bar q_{t,\mathrm{raw}}^\pi$ and $\underline h_t^\pi$ (Eqs.~\eqref{eq:qh-calibration},~\eqref{eq:hazard}), $\delta_{\mathrm{gate}}$ to the threshold grid (Theorem~\ref{thm:risk-gate}), $\delta_{\mathrm{probe}}/T$ to $\bar\rho_t^{\mathrm{probe},\pi}$, and $\delta_{\mathrm{evid}}/T$ to $\bar b_t^\pi(m)$.  A single union bound makes all UCBs/LCBs jointly valid with calibration probability $\ge 1-\delta$, so $\bar q_{t,\mathrm{ctrl}}^\pi$ is a valid UCB on $q_t^\pi$.  If a mechanism factor is not separately calibrated, the corresponding term is set to $1$ and Eq.~\eqref{eq:false-risk} falls back to $\bar q_{t,\mathrm{raw}}^\pi$.  The context-conditioned success term is
\begin{equation}
\label{eq:hazard}
\underline h_t^\pi(M)=\mathrm{LCB}_{\rm Binom}\!\left(s_t^{M,\pi},n_t^{M,\pi},\tfrac{\delta_{\rm raw}}{2T}\right),
\end{equation}
and the symbol $\alpha$ in Definition~\ref{def:stopped-process} denotes the false-admission \emph{risk budget}, distinct from the calibration confidence $\delta$.  The controller-level certificate is
\begin{equation}
\label{eq:mechanism-certificate}
C_T^{\rm ctrl}(\pi)=
\clip{\prod_{t=1}^{T}(1-\bar q_{t,\mathrm{ctrl}}^\pi)-
\prod_{t=1}^{T}(1-\underline h_t^\pi(M))}.
\end{equation}

\begin{theorem}[Simultaneously valid finite-manifest calibration]
\label{thm:finite-class-selection}
Let $\Pi_{\alpha}$ be a finite, pre-declared class of frozen controllers generated from the abstract action classes in Eq.~\eqref{eq:abstract-actions}; each controller maps certificate states to stop/refine decisions, model-controlled risk-probe calls, evidence intensity $m$, context intensity $k$, prompts, routing/deduplication rules, decoding parameters, a retrievable memory snapshot $M^\star$, and write mode $\omega_{\rm write}$.  Allocate the calibration failure probability $\delta$ across the joint family of UCB/LCB statements as
\[
\delta=\delta_{\mathrm{raw}}+\delta_{\mathrm{gate}}+\delta_{\mathrm{probe}}+\delta_{\mathrm{evid}},
\]
with the per-component per-step splitting of Section~\ref{sec:theory} applied for each $\pi\in\Pi_\alpha$ separately when estimating Eq.~\eqref{eq:qh-calibration}, Eq.~\eqref{eq:false-risk}, and Eq.~\eqref{eq:hazard}.  The calibrated selection rule is
\begin{equation}
\label{eq:controller-select}
\pi^\star\in\arg\max_{\pi\in\Pi_{\alpha}}
C_T^{\rm ctrl}(\pi).
\end{equation}
Then, by a single union bound over $\Pi_\alpha$, threshold grids, mechanism factors, and steps, with probability $\ge 1-\delta$ over calibration the UCBs/LCBs in Eq.~\eqref{eq:mechanism-certificate} are simultaneously valid for every $\pi\in\Pi_\alpha$.  Consequently, any $\pi^\star$ selected by Eq.~\eqref{eq:controller-select}---and any pre-declared frozen $\pi$---inherits a post-selection-valid lower bound $C_T^{\rm ctrl}(\pi^\star)$; if its manifest includes a score-threshold admission rule, the threshold also satisfies Theorem~\ref{thm:risk-gate}.
\end{theorem}

\paragraph{Numerical instantiation.} Before $D_{\rm cal}$ opens we fix $\delta{=}0.10$ split equally as $\delta_{\mathrm{raw}}{=}\delta_{\mathrm{gate}}{=}\delta_{\mathrm{probe}}{=}\delta_{\mathrm{evid}}{=}0.025$, $|\Theta_t|{=}1$ (two-stage realization in \S\ref{sec:method}), $|\Pi_\alpha|{=}5$ (Table~\ref{tab:calibrated_certificate}), $T{=}5$ for ICPC Refine@5 (or $T{=}1$ for LCB-Pro Pass@1), and per-step $\alpha_t{=}0.10$.  The per-step per-controller Clopper--Pearson~\cite{clopper1934use} allocations for Eqs.~\eqref{eq:qh-calibration}--\eqref{eq:hazard}, Theorem~\ref{thm:risk-gate}, and Propositions~\ref{prop:probe-odds}--\ref{prop:evidence-acquire} are, respectively,
{\small
\[
\tfrac{\delta_{\mathrm{raw}}}{2T|\Pi_\alpha|},\ \tfrac{\delta_{\mathrm{gate}}}{T|\Theta_t||\Pi_\alpha|},\ \tfrac{\delta_{\mathrm{probe}}}{T|\Pi_\alpha|},\ \tfrac{\delta_{\mathrm{evid}}}{T|\Pi_\alpha|},
\]
}%
with $\mathrm{UCB}_{\rm Binom}(0,0,\eta){=}1$ and $\mathrm{LCB}_{\rm Binom}(0,0,\eta){=}0$ for empty cells.  Concretely, the raw per-row confidence on the Full controller at any single $t$ is $0.025/(2{\cdot}5{\cdot}5){=}5{\times}10^{-4}$.

\begin{proposition}[RiskProbe bad-and-survive UCB]
\label{prop:probe-odds}
Let $W_t^{\mathrm{probe}}$ be a $\{0,1\}$-valued probe-survival indicator at step $t$, set to $1$ when the probe is not invoked.  On held-out calibration traces, let $n_t^{\pi}$ be the number of trajectories active at step $t$ under controller $\pi$ (so $A_t$ holds), and let $f_t^{\mathrm{probe},\pi}$ count those for which the hidden-judge label $B_t^\epsilon$ holds and the recorded $W_t^{\mathrm{probe}}=1$.  Define
\[
\bar\rho_t^{\mathrm{probe},\pi}:=\mathrm{UCB}_{\rm Binom}\!\left(f_t^{\mathrm{probe},\pi},\,n_t^{\pi},\,\delta_{\mathrm{probe}}/T\right).
\]
Then $\bar\rho_t^{\mathrm{probe},\pi}$ is a valid one-sided UCB on $\Pr_\pi(B_t^\epsilon,W_t^{\mathrm{probe}}{=}1\mid A_t)$ for $t\le T$ with confidence $1-\delta_{\mathrm{probe}}$, factoring as $\sigma_t^{\mathrm{bad},\pi}\rho_t^{\pi}$ with $\sigma_t^{\mathrm{bad},\pi}:=\Pr_\pi(W_t^{\mathrm{probe}}{=}1\mid B_t^\epsilon,A_t)$ and $\rho_t^{\pi}:=\Pr_\pi(B_t^\epsilon\mid A_t)$, so the probe contracts bad-and-survive mass whenever bad candidates survive at a strictly lower rate than the prior.
\end{proposition}

\begin{proposition}[Candidate-level evidence from EvidenceAcquire]
\label{prop:evidence-acquire}
Let $W_t^{\mathrm{evid}}\in\{0,1\}$ denote the program-level evidence-gate survival at step $t$ (recorded as $1$ when the gate is not invoked).  On held-out calibration traces calibrated under the chain-rule conditional, let $\bar b_t^\pi(m)$ be a one-sided UCB on $\Pr_\pi(W_t^{\mathrm{evid}}=1\mid B_t^\epsilon,W_t^{\mathrm{probe}}=1,A_t)$ at level $\delta_{\mathrm{evid}}/T$ from active rows that are simultaneously bad and probe-surviving.  Then setting $\Delta I_t(m):=-\log \bar b_t^\pi(m)$ and $\mathcal I_t^{\mathrm{evid},\pi}:=\sum_{s\le t}\Delta I_s(m_s)$ yields, simultaneously for $t\le T$ with confidence $1-\delta_{\mathrm{evid}}$,
\[
\Pr_\pi\bigl(W_t^{\mathrm{evid}}=1\,\big|\,B_t^\epsilon,W_t^{\mathrm{probe}}=1,A_t\bigr)\le\exp\!\bigl[-\Delta I_t(m)\bigr].
\]
The bound is suite-level: $\bar b_t^\pi(m)$ is a Bernoulli rate over the consensus-filtered suite, with no independence required across the $m$ tests.
\end{proposition}

\begin{proposition}[ContextAcquire as calibrated hazard envelope]
\label{prop:experience-hazard}
Under the split $D_{\rm hist}^{\rm mem}\to D_{\rm cal}\to D_{\rm test}$, a $\textsf{ContextAcquire}(k)$ action with frozen $M^\star$ realizes Eq.~\eqref{eq:context-action} by assigning $\underline h_{t+}\leftarrow\underline h_t^\pi(M^\star)$, with $\underline h_t^\pi(M^\star)$ the LCB in Eq.~\eqref{eq:hazard} estimated on traces using $M^\star$. The increment $\Delta h_t(k,Z_t):=\underline h_t^\pi(M^\star)-\underline h_t^{\pi,\varnothing}$ is the snapshot-on/off LCB gap, reported as a deployment diagnostic. Both the controller-level certificate Eq.~\eqref{eq:mechanism-certificate} and the generic certificate Eq.~\eqref{eq:main-certificate} use only $\underline h_t^\pi(M^\star)$ as the LCB factor, read off after $M^\star$ is frozen.
\end{proposition}

\begin{corollary}[Stopped clean-before-false certificate]
\label{thm:structural-certificate}
Under Assumption~\ref{assump:trace-calibration}, suppose a frozen controller $\pi$ has simultaneous active-step bounds $q_t^\pi\le\bar q_t^\pi$ and $h_t^\pi\ge\underline h_t^\pi$ for all $t$.  Then, with probability at least $1-\delta$ over calibration, a fresh evaluation problem satisfies
\begin{equation}
\label{eq:main-certificate}
\Pr_\pi(\tau_S\le T,\tau_S<\tau_F)
\ge
C_T(\pi):=
\clip{\prod_{t=1}^{T}(1-\bar q_t^\pi)-\prod_{t=1}^{T}(1-\underline h_t^\pi)}.
\end{equation}
\end{corollary}

\begin{proof}
The product form follows from the chain rule over stopped events, not from step independence.  We assume the trajectory begins active, $\Pr(A_1)=1$, and adopt the convention that any conditional probability with a zero-probability conditioning event evaluates to $0$ (these conditioning events contribute $0$ to the bound).  We argue in three steps.

\textit{Step 1: no-false survival.} The events $\{\tau_F>t\}$ form a decreasing filtration with $\{\tau_F>t\}=\{\tau_F>t-1\}\cap\{F_t=0\}$. Since $F_t=1$ implies the trajectory is active at $t$, we have the chain $\{F_t=1\}\subseteq A_t\subseteq\{\tau_F>t-1\}$, and for events $E\subseteq B\subseteq C$ with $\Pr(B)>0$ the identities $\Pr(E\mid C)=\Pr(E)/\Pr(C)\le\Pr(E)/\Pr(B)=\Pr(E\mid B)$ give
\[
\Pr_\pi(F_t=1\mid\tau_F>t-1)\le \Pr_\pi(F_t=1\mid A_t)=q_t^\pi\le \bar q_t^\pi.
\]
Iterating the chain rule on $\Pr(\tau_F>T)=\prod_{t=1}^{T}\Pr(F_t=0\mid\tau_F>t-1)$ yields
\begin{equation}
\Pr_\pi(\tau_F>T)\ge\prod_{t=1}^{T}(1-\bar q_t^\pi).
\label{eq:cert-proof-1}
\end{equation}

\textit{Step 2: still-active upper bound.} The active event satisfies $A_{t+1}=A_t\cap\{S_t=0,F_t=0\}\subseteq A_t\cap\{S_t=0\}$, hence
\[
\Pr_\pi(A_{t+1}\mid A_t)\le \Pr_\pi(S_t=0\mid A_t)=1-h_t^\pi\le 1-\underline h_t^\pi.
\]
Iterating, using $A_{T+1}\subseteq A_T\subseteq\cdots\subseteq A_1$ and $\Pr(A_1)=1$,
\begin{equation}
\Pr_\pi(A_{T+1})=\prod_{t=1}^{T}\Pr_\pi(A_{t+1}\mid A_t)\le\prod_{t=1}^{T}(1-\underline h_t^\pi).
\label{eq:cert-proof-2}
\end{equation}

\textit{Step 3: clean-before-false event.} The certified event $\{\tau_S\le T,\tau_S<\tau_F\}$ is exactly the complement, inside $\{\tau_F>T\}$, of the still-active event $\{A_{T+1}\}$: trajectories with no false admission by $T$ either succeed by $T$ (clean stopping at some $\tau_S\le T<\tau_F$) or remain active at $T+1$. Therefore
\[
\Pr_\pi(\tau_S\le T,\tau_S<\tau_F)=\Pr_\pi(\tau_F>T)-\Pr_\pi(A_{T+1}).
\]
Combining \eqref{eq:cert-proof-1} and \eqref{eq:cert-proof-2}, and taking the positive part since probabilities are non-negative, yields Eq.~\eqref{eq:main-certificate}.  The decomposition includes the degenerate case $\tau_S=\tau_F=\infty$, which has $\tau_F>T$ and $\tau_S>T$ and therefore lies in $A_{T+1}$.  The validity of the per-step bounds $\bar q_t^\pi$ and $\underline h_t^\pi$ follows from Clopper--Pearson on independent active calibration rows at fixed $t$ (Assumption~\ref{assump:trace-calibration}); the union bound over $t$ is absorbed by the $\delta_{\mathrm{raw}}/(2T)$ allocation in Eq.~\eqref{eq:qh-calibration}.
\end{proof}

\begin{table}[t]
\centering
\footnotesize
\setlength{\tabcolsep}{3pt}
\renewcommand{\arraystretch}{1.05}
\begin{tabular}{@{}p{0.26\linewidth}p{0.32\linewidth}p{0.36\linewidth}@{}}
\toprule
Action class & Certificate effect & \model realization (control variable) \\
\midrule
$\textsf{Stop}$ & enforce calibrated risk gate / frozen observable gate & final candidate gate (admission threshold) \\
$\textsf{Refine}$ & move to next active candidate & solution generation / repair (horizon $T$, retry) \\
$\textsf{RiskProbe}$ & $\rho_t\!\mapsto\!\lambda_t\rho_t$ & Hypothesis Prober (model-decided) \\
$\textsf{EvidenceAcquire}(m)$ & $I_t\!\mapsto\!I_t+\Delta I_t(m)$ & Validator + Test Augmentation ($m$, $\gamma_{\rm evid}$) \\
$\textsf{ContextAcquire}(k)$ & $\underline h_t\!\mapsto\!\underline h_t+\Delta h_t(k)$ & Experience memory (budget $k$) \\
\bottomrule
\end{tabular}
\caption{Theory-to-system map. The theory derives abstract action classes from certificate quantities; \model instantiates each class with concrete tools and calibrated knobs.}
\label{tab:certificate-operationalization}
\end{table}

\begin{takeawaybox}
\textbf{(T1) Stopped-process state.} Feedback-driven CP solving compresses into three certificate quantities $(\rho_t,I_t,\underline h_t)$, with calibrated realizations $(\bar\rho_t^{\mathrm{probe},\pi},\mathcal I_t^{\mathrm{evid},\pi},\underline h_t^\pi(M))$ entering the certificate.
\textbf{(T2) One action per quantity.} Each quantity is moved by exactly one action class---RiskProbe on $\rho_t$, EvidenceAcquire on $I_t$, ContextAcquire on $\underline h_t$---which specifies what a useful tool must change.
\textbf{(T3) Audit target.} The calibration rule and stopped-process certificate select and audit frozen manifests on held-out traces; Section~\ref{sec:method} instantiates the three classes as Hypothesis Prober, Solution Validator / Test Augmentation, and frozen Experience retrieval.
\end{takeawaybox}
The detailed proofs of all theorems are provided in Appendix~\ref{app:proofs}.
\section{Method}
\label{sec:method}

\paragraph{Motivation \& Overview.}
\model is the instantiation of the selected controller $\pi^\star$: each abstract action class maps to a concrete tool (Table~\ref{tab:certificate-operationalization}).
All manifest entries---prompts, the observable admission gate, the HP and TA interfaces, the retrieval budget, the historical memory snapshot $M^\star$, the write mode, and the maximum horizon---are frozen before benchmark evaluation, either pre-declared or audited on held-out traces $D_{\rm cal}$ via Theorem~\ref{thm:finite-class-selection}.

\subsection{Dual-Granularity Verification}
\label{sec:verification}

$\textsf{RiskProbe}$ and $\textsf{EvidenceAcquire}(m)$ act at different granularities: the former screens a reasoning state before it is admitted, the latter gates a complete candidate before it is trusted.
\model realizes them as a Hypothesis Prober for the logged upstream gate $G_t^{\mathrm{HP}}$ and a Solution Validator for the program-level admission screen inside $\Gamma_t$.
Separating the two roles prevents the upstream screen from being double-counted as calibrated program-level evidence.

\paragraph{Hypothesis Prober.}
During early-stage reasoning, the agent often needs to verify local claims that are hard to establish purely in natural language---e.g., whether a recurrence holds on random small instances, whether a greedy rule fails on a constructed counterexample, or whether an invariant is preserved by a transition.
\model exposes a Hypothesis Prober that compiles and runs short C++ snippets for this purpose. Under the frozen prompt and decoding policy the backbone itself decides when to call HP from its observable reasoning state---no hand-written trigger grid---and every call/pass/fail outcome is logged for calibration.
A surviving HP call contracts the posterior odds of an upstream reasoning error (Proposition~\ref{prop:probe-odds}) and tightens the residual envelope $\bar\rho_t^{\mathrm{probe},\pi}$ entering Eq.~\eqref{eq:false-risk}; HP thus moves $\rho_t$ (and the aggregate $\bar q_t$) alone, not the program-level evidence channel.

\paragraph{Solution Validator.}
Once the agent produces a complete candidate implementation, the Solution Validator compiles and runs it against the local test suite: any public/local failure rejects the candidate and returns counterexample-level feedback.
By default SV is a deterministic hard screen and contributes no calibrated statistical evidence; only when the suite uses fresh sampled tests or an externally calibrated set is an all-pass event additionally folded into $\mathcal I_t$ as the binary survival channel of Proposition~\ref{prop:evidence-acquire}.

\subsection{Test Augmentation}
\label{sec:augmentation}

The action class $\textsf{EvidenceAcquire}(m)$ makes the admission gate more informative against public-test-passing errors and exposes the evidence-acquisition intensity as a frozen controller entry.
Public test suites $\mathcal{T}_{\text{basic}}$ are fixed and shallow, so passing them is a weak certificate: candidates routinely clear $\mathcal{T}_{\text{basic}}$ yet fail hidden tests in corner regimes.
Test Augmentation (TA) instantiates $\textsf{EvidenceAcquire}$ by appending consensus-filtered checks beyond $\mathcal{T}_{\text{basic}}$; an all-pass on the augmented suite enters $\mathcal I_t$ as the binary survival channel of~Proposition~\ref{prop:evidence-acquire}.

\paragraph{Augmentation pipeline.}

\model implements test augmentation via three stages.
\textbf{Input Synthesis.} Given the statement, the agent synthesizes a C++ input generator based on \texttt{testlib.h}\footnote{\url{https://github.com/MikeMirzayanov/testlib}} and command suites that sweep corner regimes such as min/max values, degenerate structures, adversarial patterns, and random sizes.
\textbf{Oracle Labeling.} Ignoring time and memory limits often makes it easier to write a correct brute-force solver than an efficient one~\cite{ridnik2024code,he2025hardtests}; \model therefore executes independently generated brute-force solvers on synthesized inputs to obtain candidate reference outputs.
\textbf{Consensus Filtering.} Because brute-force programs can still be imperfect, \model runs multiple solvers in parallel and majority-votes their outputs. Candidates are evaluated on augmented tests only after passing public tests, so augmentation serves as additional diagnostic evidence after the basic admission screen.
\textbf{Budget and retries.} The frozen TA manifest caps synthesis at 32 candidate augmented inputs per problem. This is an upper bound rather than a fixed suite size: the generator does not keep sampling until a nonempty suite is obtained, and there is no separate retention cap after filtering. Instead, \model retains all valid inputs that pass synthesis, compilation/execution, and consensus labeling within the cap; generator synthesis/compilation/execution may be retried up to five times. Consequently, GenRate measures the fraction of problems for which this capped protocol yields a nonempty suite, while the candidate-level evidence calibration treats the resulting suite as one binary survival channel.

\subsection{Experience-Driven Self-Evolving}
\label{sec:memory}
The action class $\textsf{ContextAcquire}(k)$ calls for a compact context prior $M$ that raises the active-state success hazard $\underline h_t(M)$ in Corollary~\ref{thm:structural-certificate}.  A naive instantiation---continuously updating $M$ across the entire problem stream---would let target-split outcomes leak into $M$ and invalidate the calibration premise of Theorem~\ref{thm:finite-class-selection}: the hazard lift would no longer be data-independent of the evaluation split.  \model resolves this with a temporal frozen-memory protocol,
\[
D_{\rm hist}^{\rm mem}\rightarrow D_{\rm cal}\rightarrow D_{\rm test}.
\]
In \emph{memory-building mode}, validated fail-to-pass traces from $D_{\rm hist}^{\rm mem}$ are written into a retrievable memory bank.  In \emph{certificate-evaluation mode}, the resulting snapshot $M^\star$ is frozen before $D_{\rm cal}$ and $D_{\rm test}$; both read $M^\star$ read-only, and target-split experiences are written only to the audit log, never retrieved by subsequent target problems.  Dataset instantiation of $D_{\rm hist}^{\rm mem}/D_{\rm cal}/D_{\rm test}$ is reported in \S\ref{sec:experiment}.
\paragraph{Memory-building stage.}
\emph{Notation.} On $D_{\rm hist}^{\rm mem}=(q_1,\ldots,q_n)$ we index problems by $r$ and within-problem attempts by $t$; the target-split index $i\in D_{\rm cal}\cup D_{\rm test}$ is disjoint from $r$.

The design question is what evidence each experience should carry so that the resulting prior can move the hazard but not the calibration premise.  We require each retrievable item to be a self-contained \emph{Fail$\to$Pass} triplet---a failed attempt, its observable failure evidence, and a later observable passing correction on the \emph{same} problem.  The triplet anchors every learnt pattern to an observable counterexample, so retrieval at evaluation time substitutes for an additional probe rather than supplying labels; and because it is built only on $D_{\rm hist}^{\rm mem}$, it leaves the active-row denominator $n_t^\pi$ on $D_{\rm cal}$ untouched.

Concretely, on problem $q_r$ the agent maintains a build-time memory state $M^{\rm build}_r$ distilled from validated fixes.  Candidate solutions are generated as
\begin{equation}
s_{r,t} = \mathrm{Gen}(q_r, M^{\rm build}_r, h_{r,t-1}),
\label{eq:gen_with_memory}
\end{equation}
where $h_{r,t-1}$ summarizes within-problem history.  The Solution Validator returns a pass indicator and public/local feedback,
\begin{equation}
(v_{r,t},\,f_{r,t}) = \Phi(q_r, s_{r,t}),\qquad v_{r,t}\in\{0,1\},
\label{eq:validator}
\end{equation}
with $f_{r,t}$ the actionable failure feedback (unused when $v_{r,t}{=}1$), matching the $(v,f)\gets\Phi(\cdot)$ call in Algorithm~\ref{alg:orchestration}.  For any failing attempt $t$, define the first later observable pass
\begin{equation}
t^+(r,t) := \min\{u>t \mid v_{r,u}=1\},
\label{eq:first_pass_after_fail}
\end{equation}
when it exists.  The extracted experience is
\begin{equation}
e_{r,t}=\mathrm{Exp}(q_r,s_{r,t},s_{r,t^+(r,t)},f_{r,t}),
\label{eq:experience_extraction}
\end{equation}
realising the Fail$\to$Pass triplet above as a reusable memory item.  The build-time memory update is
\begin{equation}
M^{\rm build}_{r+1}=\mathrm{Update}(M^{\rm build}_{r},e_{r,t}).
\label{eq:memory_update}
\end{equation}
Because both $s_{r,t^+(r,t)}$ and $f_{r,t}$ are observable on $D_{\rm hist}^{\rm mem}$, no hidden label from $D_{\rm cal}\cup D_{\rm test}$ enters $\mathrm{Exp}(\cdot)$, $\mathrm{Update}(\cdot)$, prompts, or routing.

\paragraph{Frozen certificate-evaluation stage.}
After historical memory construction, calibration chooses a snapshot $M^\star$ and controller manifest $\pi^\star$ using $D_{\rm cal}$ with memory fixed.  For every target calibration and test problem,
\begin{equation}
M_i^{\rm target}\equiv M^\star,\qquad i\in D_{\rm cal}\cup D_{\rm test}.
\label{eq:frozen_memory_eval}
\end{equation}
An extractable target-split experience is recorded only in an audit log $\widetilde M$ for optional post-benchmark deployment:
\begin{equation}
\widetilde M_{i+1}=\widetilde M_i\cup\{e_i\},\qquad
M_{i+1}^{\rm target}=M^\star.
\label{eq:audit_memory}
\end{equation}
Thus the calibrated certificate's memory prior is data-independent of the target evaluation outcomes---exactly the i.i.d.-given-controller premise that Theorem~\ref{thm:finite-class-selection} requires of the held-out rows.

\paragraph{Experience Curation.}
The hazard lift attributable to $M$ is bounded by the description length of the retrieved evidence (Proposition~\ref{prop:experience-hazard}), so curation is the part of the manifest that directly determines how much $\underline h_t$ can move.  We design three controllers, one per MDL component, and freeze all three into the manifest (Definition~\ref{def:controller-manifest}):

\emph{(i) Routing} restricts the active hypothesis bucket by splitting the snapshot into two repositories,
$M^\star := (\mathcal{M}^{\mathrm{gen}}, \mathcal{M}^{\mathrm{alg}})$, where $\mathcal{M}^{\mathrm{gen}}$ collects task-agnostic debugging patterns and $\mathcal{M}^{\mathrm{alg}}_{\tau_1,\tau_2}$ is indexed by a two-level taxonomy $(\tau_1,\tau_2)$ with 8 primary categories and 68 secondary tags.

\emph{(ii) LLM-mediated deduplication} merges redundant traces: on each insertion the curation prompt compares the candidate against existing summaries in the target bank and rejects entries judged duplicate or covered by an existing generalization, so retrieved memory carries non-redundant evidence rather than near-copies of the same pattern.

\emph{(iii) Capacity} controls the description length of retrieved memory by capping each bank, $|\mathcal{M}^{\mathrm{gen}}|\le 20$ and $|\mathcal{M}^{\mathrm{alg}}_{\tau_1,\tau_2}|\le K{=}10$; once a bank is at capacity the same semantic rule keeps at most the allowed number of representative experiences,
\[
\mathcal{B} \leftarrow \mathrm{Dedup}(\mathcal{B};K).
\]

\paragraph{Experience Utilization.}
At evaluation time the two repositories play distinct roles in the hazard.  The general bank supplies a \emph{task-agnostic prior shift} applied to every problem: a compact summary is appended to the system prompt,
\begin{align*}
\textsf{SysPrompt}_i = {} & \textsf{BaseSysPrompt} \\
& \Vert\ \mathrm{Format}(\mathcal{M}^{\mathrm{gen}}, k_{\text{gen}}),
\end{align*}
raising the baseline component of $\underline h_t(M^\star)$ uniformly across $q_i$.  The algorithmic bank supplies a \emph{task-conditional lift} that fires only when $q_i$ is routed to a populated bucket: \model predicts a taxonomy label and retrieves the corresponding frozen bucket,
\begin{subequations}\label{eq:alg_retrieval_pair}
\begin{align}
(\hat{\tau}_1,\hat{\tau}_2) &= \mathrm{ClassifyAlg}(q_i), \label{eq:alg_classify}\\
E_i^{\mathrm{alg}} &= \mathrm{RetrieveAlgExp}(\hat{\tau}_1,\hat{\tau}_2;\ell,M^\star), \label{eq:alg_retrieve}
\end{align}
\end{subequations}
contributing the residual, problem-specific portion of the lift.  Both pathways read $M^\star$ read-only, so the realized hazard $\underline h_t^\pi(M^\star)$ remains the calibrated quantity entering Corollary~\ref{thm:structural-certificate}, and the refinement-step-$1$ ($t{=}1$) acceptance gain in Table~\ref{tab:exp_steps} is its empirical instantiation.
\subsection{Tool Orchestration}
\label{sec:tool-orchestration}
\model runs a stage-wise refinement loop coordinating the three components. For each problem $q_i$, the system prompt is initialized with the general bank of $M^\star$ and, when helpful, retrieved algorithmic experiences $E_i^{\mathrm{alg}}$. The frozen backbone then drafts a solution and may call the Hypothesis Prober to falsify local hypotheses before committing; the Solution Validator compiles and runs each draft on public tests, whose feedback drives subsequent edits; once all public tests pass, Test Augmentation generates consensus-filtered cases. In deployment mode, attempts may be appended to retrievable memory via Eq.~\eqref{eq:memory_update}; in certificate-evaluation mode they are written only to the audit log $\widetilde M$ via Eq.~\eqref{eq:audit_memory}. The mapping to the abstract controller of Section~\ref{sec:theory} is direct: HP $\to\textsf{RiskProbe}$, SV/TA $\to\textsf{EvidenceAcquire}(m)$, Experience $\to\textsf{ContextAcquire}(k)$. Algorithm~\ref{alg:orchestration} gives the full loop, with an outer loop over the $T$ refinement steps of Def.~\ref{def:stopped-process} and an inner loop over within-step HP probes (the interaction steps of Fig.~\ref{fig:verification_tool_freq}), so a single refinement step---including $T{=}1$ on LCB-Pro---can host many HP calls before the candidate is committed.

\begin{table}[t]
\centering
\scriptsize
\resizebox{\columnwidth}{!}{%
\begin{tabular}{lccc}
\toprule
\multirow{2}{*}{Model} & LCB-Pro & \multicolumn{2}{c}{ICPC-Eval} \\
\cmidrule(lr){2-2}\cmidrule(lr){3-4}
& Pass@1 & Refine@5 & \#T \\
\midrule
\multicolumn{4}{l}{\textit{Frontier LLMs}} \\
DeepSeek-R1 & 39.5 & 14.4 & 2.1 \\
DeepSeek-V3.2-Reasoner & 58.1 & 33.9 & 1.6 \\
Gemini-3-pro & 71.9 & 48.3 & 1.4 \\
o4-mini-high & 59.3 & 42.4 & 1.2 \\
\midrule
\multicolumn{4}{l}{\textit{GPT-4o Based}} \\
GPT-4o & 5.4 & 3.4 & 1.8 \\
GPT-4o Reflexion & 9.6 & 7.7 & 2.4 \\
GPT-4o CodeSim & 8.4 & 5.1 & 1.5 \\
GPT-4o AlphaCodium & 5.4 & 5.9 & 2.7 \\
\rowcolor{cpagentblue} GPT-4o \model (ours) & 10.2 (+4.8) & 7.7 (+4.3) & 1.6 \\
\midrule
\multicolumn{4}{l}{\textit{Qwen3-235B-A22B-Instruct Based}} \\
Qwen3-235B-A22B-Instruct & 26.9 & 21.2 & 1.4 \\
Qwen3-235B-A22B Reflexion & 35.3 & 27.1 & 1.2 \\
Qwen3-235B-A22B CodeSim & 45.5 & 25.4 & 2.1 \\
Qwen3-235B-A22B AlphaCodium & 25.8 & 25.4 & 1.6 \\
\rowcolor{cpagentblue} Qwen3-235B-A22B \model (ours) & 46.7 (+19.8) & 26.3 (+5.1) & 1.6 \\
\midrule
\multicolumn{4}{l}{\textit{DeepSeek-V3.2-Chat Based}} \\
DeepSeek-V3.2-Chat & 25.8 & 22.9 & 1.2 \\
DeepSeek-V3.2-Chat Reflexion & 37.1 & 32.5 & 1.2 \\
DeepSeek-V3.2-Chat CodeSim & 40.1 & 29.7 & 1.2 \\
DeepSeek-V3.2-Chat AlphaCodium & 24.5 & 29.7 & 1.7 \\
\rowcolor{cpagentblue} DeepSeek-V3.2-Chat \model (ours) & 48.5 (+22.7) & 33.9 (+11.0) & 1.3 \\
\bottomrule
\end{tabular}
}
\caption{Performance of \model and strong baselines on competitive-programming benchmarks. LiveCodeBench Pro reports Pass@1 on the 167-problem 2025Q2 target split; ICPC-Eval reports Refine@5 and the average number of refinement turns \#T. Highlighted rows correspond to \model variants built upon the same backbone.}
\label{tab:mainresults}
\end{table}
\section{Experiment}
\label{sec:experiment}
\subsection{Experiment Setup}
\paragraph{Backbone LLMs}
To evaluate the general applicability of our method across different LLMs, we conducted experiments using three distinct backbone LLMs: GPT-4o~\cite{hurst2024gpt}, Qwen3-235B-A22B-Instruct~\cite{yang2025qwen3}, and DeepSeek-V3.2-Chat~\cite{liu2025deepseek} for all multi-agent workflows and \model variants.
\paragraph{Baselines}
We compare against two baseline groups: 
\textit{(i) Frontier LLMs:} DeepSeek-R1~\cite{guo2025deepseek}, DeepSeek-V3.2-Reasoner~\cite{liu2025deepseek}, Gemini-3-pro~\cite{pichai2025new}, and o4-mini-high~\cite{openai2025_o4mini_high}. 
\textit{(ii) Agent frameworks/workflows:} For a fair comparison, we evaluate three agent frameworks: Reflexion~\cite{shinn2023reflexion}, CodeSim~\cite{islam2025codesim}, and AlphaCodium~\cite{ridnik2024code}, each with the same three backbones, all targeting C++.

\paragraph{Evaluation datasets.}
We evaluate all baselines and \model variants on two competition-style, contamination-controlled benchmarks: LiveCodeBench Pro~\cite{zheng2026livecodebench} and ICPC-Eval~\cite{xu2026icpc}. The target split for LCB-Pro is the held-out 2025Q2 subset. The temporal partition that Experience uses on LCB-Pro and the controller transfer to ICPC-Eval are described in \S\ref{sec:calibration_protocol} (summarized in Table~\ref{tab:controller_choices}).

\paragraph{Metrics}
For ICPC-Eval, we use Refine@K~\cite{xu2026icpc} to measure iterative refinement under execution feedback, i.e., whether the model passes the hidden judge within $K$ response--revision attempts; we report Refine@5 ($K{=}5$) to reflect the interactive nature of competitive programming.
For LiveCodeBench Pro, we report Pass@1 to match its standard single-shot setting and prior work, and use per-refinement-step analyses ($t{=}1,2,3$) when studying the role of experience.

\paragraph{Implementations.}
The full frozen controller manifest---refinement horizon $T$, retrieval budgets $(k_{\rm gen},\ell_{\rm alg})$, TA intensity $m$ and consensus rule $\gamma_{\rm evid}$, decoding parameters, and prompt template---is given in Table~\ref{tab:controller_choices}. Baselines (Reflexion, CodeSim, AlphaCodium) use their official prompts and control flow with the same three backbones. HP, TA, and Experience are added only inside \model variants and never touch baselines.

\subsection{Calibration Protocol and Auditability}
\label{sec:calibration_protocol}
The controller is calibrated before final benchmark reporting and is never recomputed after observing evaluation outcomes. We use a temporal three-stage protocol,
\[
D_{\rm hist}^{\rm mem}\rightarrow D_{\rm cal}\rightarrow D_{\rm test}.
\]
First, the 387 historical pre-test problems in \texttt{biannual\_2024\_7\_12} construct candidate memory snapshots through the memory-building rule in Eq.~\eqref{eq:memory_update}. Second, for each candidate policy $\pi\in\Pi_\alpha$, we run the frozen agent with a fixed snapshot $M^\star$ on the disjoint 166-problem calibration split \texttt{quater\_2025\_1\_3}, record every active trace row, compute the one-sided bounds in Eq.~\eqref{eq:qh-calibration}, audit the observable admission gate using Eq.~\eqref{eq:theta-select}, and select the best declared manifest from $\Pi_\alpha$ by Eq.~\eqref{eq:controller-select}. Third, final benchmark runs on $D_{\rm test}$ use the selected controller in certificate-evaluation mode: the policy, memory snapshot, TA consensus threshold, HP tool interface, prompts, decoding parameters, maximum horizon, and write mode are frozen, and target-test experiences are written only to the audit log but not retrievable. For ICPC-Eval we transfer the LCB-Pro--selected controller and the same frozen snapshot $M^\star$ directly to the ICPC target split, under the identical certificate-evaluation mode.

Hidden-judge outcomes are used in calibration only as labels for $S_t$ and $F_t$ after a candidate has passed the observable gate $\Gamma_t$; they are not exposed to the agent during refinement and are never used for target-split memory writes, retrieval, prompt updates, routing, or controller selection. This separation is important: public/local execution can guide repair, whereas hidden outcomes define the calibration labels for the stopped-process certificate. A trace contributes to the denominator $n_t$ only if it is active at step $t$; a candidate that fails SV or TA contributes neither to $s_t$ nor to $f_t$ because it was not admitted. These conventions align the empirical hazards with Definition~\ref{def:stopped-process} and the admission risk with Theorem~\ref{thm:risk-gate}.

\begin{table}[t]
\centering
\scriptsize
\setlength{\tabcolsep}{3pt}
\renewcommand{\arraystretch}{1.05}
\begin{tabular}{@{}p{0.21\columnwidth}p{0.13\columnwidth}p{0.56\columnwidth}@{}}
\toprule
Controller entry & Symbol & Declared / fixed value before $D_{\rm cal}$ \\
\midrule
\multicolumn{3}{l}{\textit{Data splits (Assumption~\ref{assump:trace-calibration})}} \\
Temporal splits & --- & LCB-Pro: $|D_{\rm hist}^{\rm mem}|{=}387$, $|D_{\rm cal}|{=}166$, $|D_{\rm test}|{=}167$ (97/57/13 Easy/Med/Hard). \\
\midrule
\multicolumn{3}{l}{\textit{Manifest entries (Definition~\ref{def:controller-manifest})}} \\
Admission gate & $\Gamma^{\rm obs}_{1:T}$ & SV/TA hard screen; risk-score grid $|\Theta_t|{=}1$. \\
Probe interface & $\mathcal U_{\rm probe}$ & C++ snippet HP tool; backbone-decided invocation. \\
TA intensity & $m$ & $\le 32$ synthesized inputs/problem; $\le 5$ generation retries. \\
TA aggregation & $\gamma_{\rm evid}$ & 16 brute-force samples at sampling temperature $1.0$; accept on $\ge 9/16$ agreement. \\
General retrieval & $k_{\rm gen}$ & Inject $k_{\rm gen}{=}10$ general-bank items; cap $20$. \\
Algorithmic retrieval & $\ell_{\rm alg}$ & Retrieve $\ell_{\rm alg}{=}5$ items per bucket; per-bucket cap $K{=}10$. \\
Routing & $R_{\rm route}$ & Two-level taxonomy (8 primary, 68 secondary tags). \\
Deduplication & $D_{\rm dedup}$ & LLM-mediated insertion / merge decision. \\
Prompt template & $p_{\rm prompt}$ & Frozen system + user prompts (Appendix). \\
Decoding & $d_{\rm decode}$ & Backbone-default temperature; greedy where supported. \\
Horizon & $T$ & $T{=}5$ (ICPC Refine@5); $T{=}1$ (LCB-Pro Pass@1). \\
Memory snapshot & $M^\star$ & Built from $D_{\rm hist}^{\rm mem}$; read-only on $D_{\rm cal}\cup D_{\rm test}$. \\
Write mode & $\omega_{\rm write}$ & $\textsf{audit-only}$ on certificate-bearing runs. \\
\midrule
\multicolumn{3}{l}{\textit{Manifest class $\Pi_\alpha$ (5 controller variants, Theorem~\ref{thm:finite-class-selection})}} \\
\multicolumn{3}{p{0.93\columnwidth}}{
Base: backbone only (no SV/HP/TA, no $M^\star$); admission is the trivial all-pass gate. \quad
+SV: Base + Solution Validator (admission gated by SV). \quad
+SV+HP: above + Hypothesis Prober. \quad
+SV+HP+TA: above + Test Augmentation. \quad
Full (=$\pi^\star$): above + frozen Experience retrieval ($M^\star, k_{\rm gen}, \ell_{\rm alg}$). All five share $p_{\rm prompt}, d_{\rm decode}, T, \omega_{\rm write}$ and differ only by which manifest entries are switched on; mechanism factors of switched-off components fall back to $1$ in Eq.~\eqref{eq:false-risk} and to the trivial gate in $\Gamma^{\rm obs}_{1:T}$.
} \\
\bottomrule
\end{tabular}
\caption{Frozen controller manifest for certificate-bearing evaluation. Each row maps a manifest entry of Definition~\ref{def:controller-manifest} to its declared value; the bottom block defines the five controller variants in $\Pi_\alpha$. No entry is tuned on $D_{\rm test}$.}
\label{tab:controller_choices}
\end{table}

\subsection{Main Results}
From Table~\ref{tab:mainresults}: (1) \model improves consistently across backbones, with best gains on DeepSeek-V3.2-Chat of $+11.0$ (ICPC-Eval Refine@5) and $+22.7$ (LCB-Pro Pass@1, $81/167$ vs.\ $43/167$, Wilson 95\% CIs non-overlapping); (2) it matches or beats same-backbone baselines and equals DeepSeek-V3.2-Reasoner on ICPC-Eval at $33.9$ Refine@5 under frozen memory; (3) gains scale with backbone strength.  Section~\ref{sec:certificate_results} audits Corollary~\ref{thm:structural-certificate}'s clean component.

\subsection{Trace-Calibrated Certificate and Mechanism Evidence}
\label{sec:certificate_results}

We instantiate the certificate on a held-out calibration split using the logging schema in Definition~\ref{def:stopped-process}.  The pre-declared manifest class is $\Pi_\alpha=\{\text{Base},\,\text{+SV},\,\text{+SV+HP},\,\text{+SV+HP+TA},\,\text{Full}\}$ ($|\Pi_\alpha|=5$); each controller is evaluated on $D_{\rm cal}$ with the per-controller-split Clopper--Pearson allocations of the Numerical instantiation paragraph, and Eq.~\eqref{eq:controller-select} selects the one with the largest $C_T^{\rm ctrl}(\pi)$ for final evaluation.  Table~\ref{tab:calibrated_certificate} reports the theorem-level products and certificate values; compact display rates are derived from these products and are not homogeneous-step assumptions.

\begin{table}[t]
  \centering
  \scriptsize
  \setlength{\tabcolsep}{2.5pt}
  \renewcommand{\arraystretch}{0.95}
  \resizebox{\linewidth}{!}{%
  \begin{tabular}{@{}c c c c c c c@{}}
  \toprule
  \multirow{2}{*}{\makecell{Config\\(SV,HP,TA,Exp)}} &
  \multicolumn{3}{c}{Calibration certificate ($D_{\rm cal}$, $T{=}5$)} &
  \multicolumn{3}{c}{Empirical (target split)} \\
  \cmidrule(lr){2-4}\cmidrule(lr){5-7}
   & $\prod(1{-}\bar q)$ & $\prod(1{-}\underline h)$ & $C_5^{\rm ctrl}$ & \makecell{LCB-Pro\\Pass@1} & \makecell{ICPC-Eval\\R@5 (\#T)} & Cost \\
  \midrule
  \cfg{0}{0}{0}{0} & 63.9\% & 60.3\% & 3.6\%  & 25.8 & \rt{22.9}{1.2} & \$0.002 \\
  \cfg{1}{0}{0}{0} & 69.5\% & 63.4\% & 6.1\% & 43.7 & \rt{31.4}{1.3} & \$0.023 \\
  \cfg{1}{1}{0}{0} & 73.4\% & 61.9\% & 11.5\% & 44.3 & \rt{32.2}{1.3} & \$0.027 \\
  \cfg{1}{1}{1}{0} & 76.5\% & 60.3\% & 16.2\% & 46.1 & \rt{33.1}{1.3} & \$0.043 \\
  \cfg{1}{1}{1}{1} & \textbf{79.5\%} & \textbf{54.1\%} & \textbf{25.4\%} & \textbf{48.5} & \textbf{\rt{33.9}{1.3}} & \$0.035 \\
  \bottomrule
  \end{tabular}%
  }
  \caption{Trace-calibrated stopping-time certificate and component ablation on $\Pi_\alpha$ (DeepSeek-V3.2-Chat). \emph{Calibration certificate}: theorem-level products and $C_5^{\rm ctrl}$ on the 166-problem held-out $D_{\rm cal}$; a missing mechanism factor falls back to the admitted-failure UCB. \emph{Empirical}: LCB-Pro Pass@1 ($T{=}1$, 167-problem 2025Q2 split), transferred ICPC-Eval Refine@5 with mean \#refinements (transfer per \S\ref{sec:calibration_protocol}), and per-problem API cost. $C_5^{\rm ctrl}$ lower-bounds $\Pr(\tau_S\le 5,\tau_S<\tau_F)$; Pass@1 is cross-reference, not an empirical bound on $C_5^{\rm ctrl}$.}
  \label{tab:calibrated_certificate}
\end{table}

Table~\ref{tab:calibrated_certificate} shows that the certificate ranking matches the ablation trend: the Full temporal frozen-memory controller has the largest $C^{\rm ctrl}_5(\pi)$, and $C^{\rm ctrl}_5(\pi)$ lies below the matched-horizon ICPC-Eval Refine@5 in every row, as required.  Mechanism calibration on the 166-problem calibration split gives: HP odds contraction $\bar\sigma_t^{\rm bad}/\underline\sigma_t^{\rm good}=0.46/0.71=0.65$, with $\bar\sigma_t^{\rm bad}$ the Clopper--Pearson UCB on $\sigma_t^{\rm bad}$ (Proposition~\ref{prop:probe-odds}) and $\underline\sigma_t^{\rm good}$ the LCB on the analogous good-plan rate $\Pr(W_t^{\rm probe}{=}1\mid\neg B_t^\epsilon,A_t)$; TA evidence increment $\Delta I_t(m{=}32)=-\log\bar b_t(m{=}32)=1.00$ nats with post-survival bug rate $11/58=19.0\%$ on the 117-candidate diagnostic pool. For Experience, Table~\ref{tab:exp_steps} reports cumulative judge-acceptance with and without the frozen snapshot $M^\star$ over the first three refinement steps on both benchmarks. The refinement step $t$ is the response--revision attempt of Definition~\ref{def:stopped-process}, and we operationalize it per benchmark: on LCB-Pro, $t$ indexes the $t$-th Solution Validator admission within a trajectory---the per-step counterpart of the single-shot Pass@1 protocol ($T{=}1$ in Table~\ref{tab:controller_choices})---while on ICPC-Eval, $t$ indexes the $t$-th official judge submission under the Refine@5 protocol. Each cell is the fraction of $D_{\rm cal}$ problems whose first judge-accepted submission arrives by step $t$; both \emph{w/\ Exp} rows read the same pre-built $M^\star$, so the comparison isolates Experience and is order-independent. The lift concentrates at $t{=}1$ ($+9.6$~pp on LCB-Pro and $+3.4$~pp on ICPC-Eval) and shrinks at later steps, consistent with $\textsf{ContextAcquire}$ acting on the active-state hazard $\underline h_t(M)$ in Proposition~\ref{prop:experience-hazard} rather than on $\bar\rho_t^{\rm probe}$ or $\mathcal I_t$.

\begin{table}[t]
\centering
\scriptsize
\setlength{\tabcolsep}{4pt}
\renewcommand{\arraystretch}{0.95}
\begin{tabular}{@{}lccc@{}}
\toprule
\makecell[l]{Cumulative judge-acceptance (\%)\\
\scriptsize\textit{LCB-Pro: $t$-th SV admission; ICPC-Eval: $t$-th official submission}}
 & $t{=}1$ & $t{=}2$ & $t{=}3$ \\
\midrule
LCB-Pro w/o Exp                       & 50.6 & 61.4 & 68.7 \\
LCB-Pro w/\phantom{o} Exp             & 60.2 & 67.5 & 69.9 \\
\rowcolor{cpagentblue}\hspace{0.5em}$\Delta$ (Exp~$-$~no~Exp)
                                      & \textbf{+9.6} & +6.1 & +1.2 \\
\midrule
ICPC-Eval w/o Exp                     & 18.6 & 22.0 & 27.1 \\
ICPC-Eval w/\phantom{o} Exp           & 22.0 & 25.4 & 28.8 \\
\rowcolor{cpagentblue}\hspace{0.5em}$\Delta$ (Exp~$-$~no~Exp)
                                      & \textbf{+3.4} & +3.4 & +1.7 \\
\bottomrule
\end{tabular}
\caption{Cumulative judge-acceptance by refinement step $t$ (response--revision attempt of Def.~\ref{def:stopped-process}; not the interaction step of Fig.~\ref{fig:verification_tool_freq}) on $D_{\rm cal}$, with vs.\ without the frozen Experience snapshot $M^\star$ (DeepSeek-V3.2-Chat). Per-benchmark operationalization of $t$ is in the text. The lift concentrates at $t{=}1$, matching $\textsf{ContextAcquire}$ raising the active-state hazard $\underline h_t(M)$ (Prop.~\ref{prop:experience-hazard}) rather than $\bar\rho_t^{\rm probe}$ or $\mathcal I_t$.}
\label{tab:exp_steps}
\end{table}

\subsection{Efficiency Analysis}
\label{sec:efficiency}
We analyze cost--performance trade-offs with the variable-returns-to-scale Banker--Charnes--Cooper DEA model~\cite{charnes1978measuring,banker1984some} (input $=$ \textsc{Cost}, output $=$ \textsc{Pass@1}; $\theta{=}1$ marks the BCC-efficient frontier).
Figure~\ref{fig:efficiency} shows that \model lies on the frontier for all three backbones (DeepSeek-V3.2-Chat, GPT-4o, Qwen3-235B), whereas baselines such as Reflexion are on the frontier only on some backbones (DeepSeek-V3.2-Chat) and dominated on others (Qwen3-235B). On DeepSeek-V3.2-Chat, agentic baselines and \model all sit in the $10^{-2}$\,USD per-problem range (\model: \$$0.035$, Table~\ref{tab:calibrated_certificate}); the frontier is therefore driven by \model's higher \textsc{Pass@1} at comparable cost rather than by an order-of-magnitude cost gap. Tracked components account for $92.4\%$ of total cost (HP $69.7\%$, SV $17.6\%$, Experience $3.4\%$, TA $1.7\%$).
\subsection{Tool-Use Analysis}
We analyze how \model's tool use affects the three quantities identified by the theory: false-admission risk, calibrated admission evidence, and experience-conditioned success hazard.
\paragraph{Dual-Granularity Verification.}
Figure~\ref{fig:verification_tool_freq} (LCB-Pro, DeepSeek-V3.2-Chat) shows that, indexed by \emph{interaction step} (the $i$-th in-trajectory LLM/tool event, distinct from refinement step $t$ in Definition~\ref{def:stopped-process}), HP peaks at the first few steps and decays as the approach stabilizes, while SV stays roughly uniform---early hypothesis checks, then solution-level validation on executable candidates.

\paragraph{Test Augmentation.}
Let $\mathrm{GenRate}_{\rm aug}$ be the fraction of problems with a non-empty TA suite, and call problem $i$ \emph{rescued} if a public-passing candidate that initially failed augmented tests later passes them and is accepted. Figure~\ref{fig:aug} reports these on the 167-problem target split: GPT-4o reaches $\mathrm{GenRate}_{\rm aug}{=}39.5\%$ vs.\ Qwen3-235B-A22B $68.3\%$ and DeepSeek-V3.2-Chat $70.1\%$, mainly reflecting brute-force quality; $\mathrm{RescueRate}_{\rm aug}$ is substantial across backbones ($9.3$--$23.9\%$), with DeepSeek-V3.2-Chat highest, indicating that both test strength and the backbone's ability to act on feedback matter.
Both diagnostics map to the evidence variable of Proposition~\ref{prop:evidence-acquire}: $\mathrm{GenRate}_{\rm aug}$ is the availability of the evidence channel ($\mathcal T_{\rm aug}\neq\emptyset$, required for $\bar b_t^\pi(m)$ to be informative), and $\mathrm{RescueRate}_{\rm aug}$ tracks the ``counterexample-found'' rate that tightens $\bar b_t^\pi(m)$; \S\ref{sec:certificate_results} converts this to $\Delta I_t(m{=}32){\approx}1.00$ nats in Table~\ref{tab:calibrated_certificate}.
\begin{figure}[tbp]
    \centering
    \newcommand{\subw}{0.5\linewidth}

    \begin{subfigure}[t]{\subw}
        \centering
        \includegraphics[width=\linewidth]{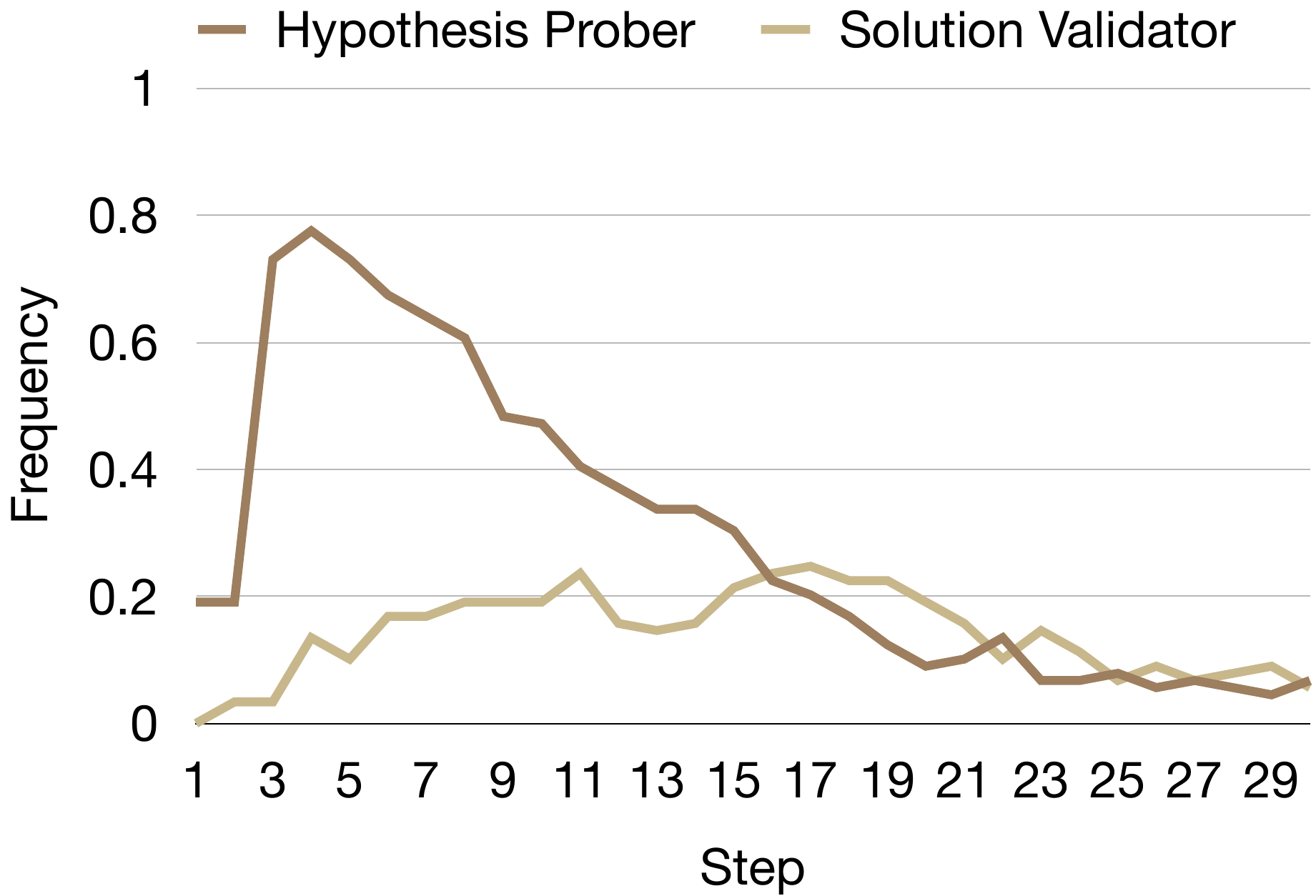}
        \caption{Verification tool usage across \emph{interaction steps}.}
        \Description{A line plot for DeepSeek-V3.2-Chat on LiveCodeBench Pro. The x-axis is the within-trajectory interaction-step index from 1 to about 29, and the y-axis is the call frequency of each verification tool. Two curves are shown: the Hypothesis Prober curve starts high at the first few interaction steps and decays monotonically toward zero as the trajectory progresses, while the Solution Validator curve stays roughly flat across all interaction steps. The pattern indicates that hypothesis-level probes concentrate early in reasoning and solution-level validation is invoked roughly uniformly once executable candidates are available.}
        \label{fig:verification_tool_freq}
    \end{subfigure}\hfill
    \begin{subfigure}[t]{\subw}
        \centering
        \includegraphics[width=\linewidth]{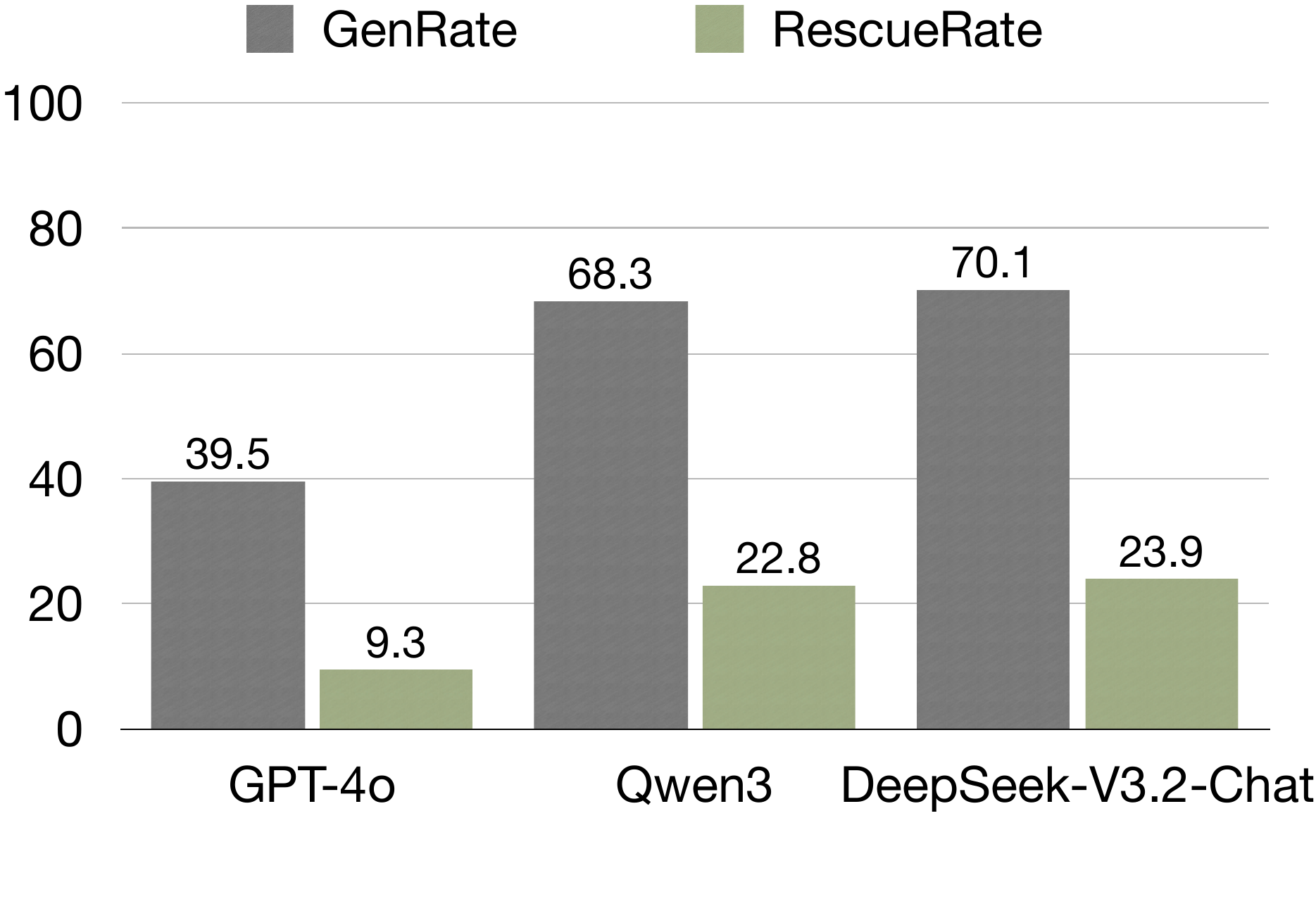}
        \caption{Test Augmentation statistics across backbones (y-axis: \%).}
        \Description{A grouped bar chart over the 167-problem LiveCodeBench Pro target split with three backbone groups along the x-axis: GPT-4o, Qwen3-235B-A22B-Instruct, and DeepSeek-V3.2-Chat. The y-axis is percentage from 0 to about 80. Each backbone has two bars: GenRate_aug, the fraction of problems for which Test Augmentation produces a non-empty consensus-filtered suite, and RescueRate_aug, the fraction of public-passing candidates whose hidden-wrong label is caught by the augmented suite. GenRate_aug is 39.5 percent for GPT-4o, 68.3 percent for Qwen3-235B-A22B-Instruct, and 70.1 percent for DeepSeek-V3.2-Chat. RescueRate_aug ranges from about 9.3 to 23.9 percent across backbones, with DeepSeek-V3.2-Chat the highest.}
        \label{fig:aug}
    \end{subfigure}

    \caption{Tool-Use Analysis on LCB-Pro. (a) HP/SV frequency per \emph{interaction step} on DeepSeek-V3.2-Chat. Panel~(a)'s x-axis indexes interaction steps ($1\ldots\!29$ per trajectory), \emph{not} refinement step $t\in\{1,\ldots,T\}$ of Definition~\ref{def:stopped-process} and Table~\ref{tab:exp_steps}. (b) $\mathrm{GenRate}_{\rm aug}$ = fraction of problems with non-empty TA suite; $\mathrm{RescueRate}_{\rm aug}$ = fraction of public-passing candidates whose hidden-wrong label is caught by TA, the operational counterpart of $\bar b_t^\pi(m)$ in Prop.~\ref{prop:evidence-acquire}.}
    \Description{A two-panel composite figure for tool-use analysis on LiveCodeBench Pro with DeepSeek-V3.2-Chat. Panel (a) is a line plot of Hypothesis Prober and Solution Validator call frequencies as a function of within-trajectory interaction step from 1 to about 29: HP decays from a high early-step peak toward zero while SV stays approximately uniform. Panel (b) is a grouped bar chart of two Test-Augmentation diagnostics, GenRate_aug and RescueRate_aug, across the three backbones GPT-4o, Qwen3-235B-A22B-Instruct, and DeepSeek-V3.2-Chat: GenRate_aug grows from 39.5 to 70.1 percent across backbones and RescueRate_aug ranges roughly between 9 and 24 percent, peaking on DeepSeek-V3.2-Chat.}
    \label{fig:freq_aug}
\end{figure}

\paragraph{Experience-Driven Self-Evolving}
Certificate-bearing benchmark results use the temporal frozen snapshot protocol in Section~\ref{sec:calibration_protocol}: historical problems build $M^\star$, calibration audits the controller with $M^\star$ fixed, and final testing reads the same snapshot.  Deployment-mode growth diagnostics in the full study show that general experiences saturate quickly while algorithmic memories keep expanding as new problem types appear; the curriculum effect follows our write-in criterion, because experiences are stored only after a failure is repaired.  These diagnostics explain snapshot construction on $D_{\rm hist}^{\rm mem}$, while all reported certificates freeze the resulting snapshot before $D_{\rm cal}$ and $D_{\rm test}$.

\subsection{Ablation Study}
\label{sec:ablation}
Table~\ref{tab:calibrated_certificate} also serves as the component ablation on DeepSeek-V3.2-Chat: SV and HP constitute Dual-Granularity Verification, TA generates extra tests beyond public cases, and Exp retrieves reusable repair knowledge from past successful trajectories.  Vanilla LLM yields $22.9$ Refine@5 on ICPC-Eval and $25.8$ Pass@1 on LCB-Pro at \$0.002 per problem.  Adding SV markedly improves both ($31.4 / 43.7$ at \$0.023); SV+HP further improves to $32.2 / 44.3$, supporting their complementarity.  Adding TA on top of Dual-Granularity Verification brings additional Pass@1 gains at higher cost ($33.1 / 46.1$ at \$0.043).  Finally, historical frozen Experience achieves the best performance ($33.9 / 48.5$) and reduces per-problem cost back to \$0.035, indicating that retrieved repair experiences reduce ineffective trial-and-error without using calibration- or test-split write-ins.  The certificate column $C_5^{\rm ctrl}$ rises monotonically with the same component additions, so the ablation ranking reproduces the certificate ranking on $D_{\rm cal}$.

\paragraph{Additional diagnostics.}
TA reliability on 117 public-pass candidates: TA flags 40/51 hidden-wrong programs while letting 47/66 hidden-correct survive, leaving the $11/58{=}19.0\%$ post-survival bug rate used by Prop.~\ref{prop:evidence-acquire}; sampling $k$ tests per suite raises detection to $74.5\%$ at $k{=}15$.
\section{Related Work}
\paragraph{Competitive programming agents.}
Among tool-augmented LLM agents that ground reasoning on deterministic execution feedback~\cite{yao2022react,madaan2023self,gao2023pal,chen2022program,gou2024tora,wu2025agentic,das2024mathsensei}, work on CP benchmarks~\cite{hendrycks2021measuring,chen2021evaluating,jain2025livecodebench,zheng2026livecodebench,xu2026icpc} mostly raises accuracy by scale: dedicated training or massive sampling~\cite{li2022competition,leblond2023alphacode,el2025competitive,guo2025deepseek,liu2025deepseek}.  Inference-time pipelines such as MapCoder, AlphaCodium, and CodeSim avoid retraining by iterating generation and verification, but are session-local and discard repair experience after each problem~\cite{islam2024mapcoder,ridnik2024code,islam2025codesim,yuhua2025oi,li2025algosimbench}.  Reflexion-style self-refinement~\cite{shinn2023reflexion} keeps verbal memory across problems but offers no admission rule for when retrieved memory should enter the loop.  \model adds neither training nor sampling: a calibrated controller ties generation, verification, and a frozen memory snapshot to a single held-out audit, closing the gap to a stronger reasoner~\cite{liu2025deepseek} on ICPC-Eval Refine@5.

\paragraph{Testing, risk control, and frozen memory.}
Randomized testing and generated checks have long driven program repair and synthesis~\cite{claessen2000quickcheck,solar2006combinatorial,chen2022codet,he2025hardtests}, but are usually consumed as raw signals.  We treat SV and TA as calibrated evidence channels via binomial calibration, conformal-style risk control, and anytime-valid evidence~\cite{clopper1934use,angelopoulos2024conformal,howard2021time,vovk2021values}, so coverage and post-survival risk are read off held-out traces.  Experience enters as a temporal frozen memory prior with PAC-Bayes/MDL complexity control~\cite{mcallester1999pac,mcallester1999pac}: tools and memory are fixed on $D_{\rm hist}^{\rm mem}$ before $D_{\rm cal}$ and $D_{\rm test}$ (Section~\ref{sec:calibration_protocol}) so the audit remains valid; memory may still grow in deployment.
\section{Conclusion}
A calibrated risk-controlled controller substantially improves contest-level programming without additional training. The theory derives certificate quantities and abstract action classes that are audited on held-out traces and frozen before evaluation; \model instantiates them as Dual-Granularity Verification, Test Augmentation, and Experience-Driven Self-Evolving, matches DeepSeek-V3.2-Reasoner on ICPC-Eval Refine@5, and provides trace-calibrated mechanism evidence for each action class.

\bibliographystyle{ACM-Reference-Format}
\bibliography{refs}
\clearpage
\appendix
\section{Tool-Orchestrated Pseudocode}
\label{app:orchestration}

Algorithm~\ref{alg:orchestration} formalizes \S\ref{sec:tool-orchestration} with two nested loops: the outer loop indexes refinement steps $t{=}1,\ldots,T$ (one admission each, matched to Refine@$K$/Pass@1), and the inner loop indexes interaction steps within $t$---HP probes consume interaction steps but not refinement steps, so a single $t$ (even $T{=}1$ on LCB-Pro) can absorb many HP calls before $\textsf{Commit}$. Memory writes follow Def.~\ref{def:controller-manifest} and uncalled channels record their gate as neutral (Def.~\ref{def:stopped-process}).

\begin{algorithm}[H]
\scriptsize
\caption{\model: tool-orchestrated refinement on a problem $q_i$. The outer \textbf{for} loop indexes refinement steps $t{=}1,\ldots,T$ (Def.~\ref{def:stopped-process}, matched to Refine@$K$/Pass@1); the inner \textbf{repeat}\,$\ldots$\,\textbf{until} loop indexes interaction steps---HP probes within step $t$, which do \emph{not} consume $t$ (cf.\ Fig.~\ref{fig:verification_tool_freq}). The deterministic SV $\to$ optional TA $\to$ admission cascade closes each step.}
\label{alg:orchestration}
{
\begin{algorithmic}[1]
\Require $q_i$, $M{=}(\mathcal{M}^{\mathrm{gen}},\mathcal{M}^{\mathrm{alg}})$, $T$, $(k_{\rm gen},\ell_{\rm alg})$, $\omega_{\rm write}$
\State $\textsf{SysPrompt}_i \gets \textsf{BaseSysPrompt}\,\Vert\,\mathrm{Format}(\mathcal{M}^{\mathrm{gen}}, k_{\rm gen})$ \Comment{$\textsf{ContextAcquire}$}
\State $(\hat\tau_1,\hat\tau_2)\gets\mathrm{ClassifyAlg}(q_i)$; $E_i^{\mathrm{alg}}\gets\mathrm{RetrieveAlgExp}(\hat\tau_1,\hat\tau_2;\ell_{\rm alg},M)$
\State $\mathcal{T}\gets\text{public tests}$; $\mathcal{T}_{\rm aug}\gets\emptyset$; $h_{i,0}\gets E_i^{\mathrm{alg}}$
\For{$t=1,\ldots,T$}
\State $X_{i,t}\gets h_{i,t-1}$; $G^{\mathrm{HP}}_{i,t}\gets$ neutral
\Repeat
\State $a\gets\mathrm{NextAction}(q_i,X_{i,t},M),\ a\in\{\textsf{Probe},\textsf{Commit}\}$
\If{$a=\textsf{Probe}$}
\State run HP on $X_{i,t}$; fold pass/fail into $X_{i,t}$; AND outcome into $G^{\mathrm{HP}}_{i,t}$ \Comment{Prop.~\ref{prop:probe-odds}}
\EndIf
\Until{$a=\textsf{Commit}$}
\State $s_{i,t}\gets\mathrm{Gen}(q_i,M,X_{i,t})$ \Comment{Eq.~\eqref{eq:gen_with_memory}}
\State $(v_{i,t},f_{i,t})\gets\Phi(q_i,s_{i,t};\mathcal{T})$; $G^{\mathrm{SV}}_{i,t}\gets v_{i,t}$ \Comment{public-test screen}
\If{$v_{i,t}=1$}
\If{$\mathcal{T}_{\rm aug}=\emptyset$}
\State $\mathcal{T}_{\rm aug}\gets\mathrm{Augment}(q_i)$ \Comment{TA once per $q_i$ (Tab.~\ref{tab:controller_choices})}
\EndIf
\State $(v_{i,t},f_{i,t})\gets\Phi(q_i,s_{i,t};\mathcal{T}\cup\mathcal{T}_{\rm aug})$; $G^{\mathrm{TA}}_{i,t}\gets v_{i,t}$ \Comment{Prop.~\ref{prop:evidence-acquire}}
\Else
\State $G^{\mathrm{TA}}_{i,t}\gets$ neutral
\EndIf
\State $\Gamma_{i,t}\gets G^{\mathrm{HP}}_{i,t}\wedge G^{\mathrm{SV}}_{i,t}\wedge G^{\mathrm{TA}}_{i,t}$; log $(Z_{i,t},A_{i,t},\Gamma_{i,t},S_{i,t},F_{i,t})$
\If{$\Gamma_{i,t}=1$}
\Return $s_{i,t}$ \Comment{$\textsf{Stop}$ admits}
\EndIf
\State $h_{i,t}\gets\mathrm{Update}(X_{i,t},f_{i,t})$ \Comment{$\textsf{Refine}$ to $t{+}1$}
\EndFor
\If{$\omega_{\rm write}=\textsf{audit-only}$}
\State append experience to $\widetilde M$ \Comment{Eq.~\eqref{eq:audit_memory}}
\Else
\State update memory via Eq.~\eqref{eq:memory_update}
\EndIf
\end{algorithmic}}
\end{algorithm}

\section{Notation Glossary}
\label{app:glossary}

Table~\ref{tab:glossary} consolidates every symbol used in Section~\ref{sec:theory}.

\begin{table}[!t]
\centering
\scriptsize
\setlength{\tabcolsep}{4pt}
\renewcommand{\arraystretch}{0.95}
\begin{tabular}{@{}p{0.30\linewidth}p{0.66\linewidth}@{}}
\toprule
Symbol & Meaning \\
\midrule
\multicolumn{2}{l}{\textit{Trajectory events and stopping times}} \\
$T$ & Refinement horizon. \\
$A_t$ & Active-trajectory event $\{\tau_S>t{-}1,\tau_F>t{-}1\}$. \\
$\Gamma_t,\,\Gamma_t^\alpha$ & Observable admission gate / risk-score-thresholded gate. \\
$B_t^\epsilon$ & Hidden-judge $\epsilon$-significant failure event for the admitted candidate. \\
$F_t,\,S_t$ & Indicators of a false admission / clean success at step $t$. \\
$\tau_F,\,\tau_S$ & Stopping times of false admission and clean success. \\
\midrule
\multicolumn{2}{l}{\textit{Certificate state, risk score, and abstract actions}} \\
$X_t,\,\phi$ & Implementation state and projection $Z_t=\phi(X_t)$. \\
$Z_t$ & Certificate summary state $(\rho_t,I_t,\underline h_t,t)$. \\
$\rho_t$ & Upper envelope on residual bad-candidate mass (calibrated as $\bar\rho_t^{\mathrm{probe},\pi}$). \\
$I_t$ & Accumulated program-level evidence in nats (calibrated as $\mathcal I_t^{\mathrm{evid},\pi}$). \\
$\underline h_t$ & Lower envelope on clean-success hazard (calibrated as $\underline h_t^\pi(M)$). \\
$\rho_{t+},\,I_{t+},\,\underline h_{t+}$ & Post-action certificate components after a single action. \\
$r_t(Z_t)$ & Residual admission-risk score $\rho_t\exp(-I_t)$. \\
$\hat\theta_{t,\alpha},\,\Theta_t$ & Calibrated admission threshold for $r_t$ and its finite grid. \\
$\textsf{Stop}$ & Commit current candidate through gate $\Gamma_t^\alpha$. \\
$\textsf{Refine}$ & Advance to next candidate distribution. \\
$\textsf{RiskProbe}$ & Action contracting $\rho_t$ via $\lambda_t(Z_t)\in[0,1]$. \\
$\lambda_t(Z_t)$ & Probe-induced contraction factor on $\rho_t$. \\
$\textsf{EvidenceAcquire}(m)$ & Action adding increment $\Delta I_t(m,Z_t)$ to $I_t$. \\
$\textsf{ContextAcquire}(k)$ & Action adding $\Delta h_t(k,Z_t)$ to $\underline h_t$ via memory snapshot. \\
\midrule
\multicolumn{2}{l}{\textit{Manifest entries (Definition~\ref{def:controller-manifest})}} \\
$\pi,\,\Pi_\alpha,\,\pi^\star$ & Frozen controller; finite pre-declared class; selected controller. \\
$\Gamma^{\rm obs}_{1:T}$ & Observable admission gate trajectory. \\
$\mathcal U_{\rm probe}$ & $\textsf{RiskProbe}$ tool interface (HP). \\
$m,\,\gamma_{\rm evid}$ & TA generation intensity and consensus aggregation rule. \\
$k_{\rm gen},\,\ell_{\rm alg}$ & General-bank injection budget and algorithmic-bucket retrieval budget (joint realization of $k$). \\
$R_{\rm route},\,D_{\rm dedup}$ & Retrieval routing rule and deduplication rule. \\
$p_{\rm prompt},\,d_{\rm decode}$ & Prompt template and decoding parameters. \\
$M^\star,\,\widetilde M$ & Frozen retrievable memory snapshot and audit-only log. \\
$\omega_{\rm write}$ & Memory write mode (\textsf{audit-only} on certificate-bearing runs). \\
\midrule
\multicolumn{2}{l}{\textit{Calibration counts and bounds}} \\
$n_t^\pi,\,f_t^\pi,\,s_t^\pi$ & Active calibration traces, false admissions, clean successes at step $t$ under $\pi$. \\
$n_t(\theta),\,f_t(\theta),\,\bar R_t(\theta)$ & Active admissions, hidden-rejected admissions, and Clopper--Pearson UCB at threshold $\theta$. \\
$f_t^{\mathrm{probe},\pi}$ & Calibration count of bad-and-probe-surviving traces. \\
$q_t^\pi,\,h_t^\pi$ & Deployment active-state false / clean hazards. \\
$\bar q_{t,\mathrm{raw}}^\pi$ & Clopper--Pearson UCB on $q_t^\pi$ from admitted failures. \\
$\underline h_t^\pi(M)$ & Clopper--Pearson LCB on $h_t^\pi$ under memory $M$; $\underline h_t^\pi:=\underline h_t^\pi(\varnothing)$. \\
$W_t^{\rm probe},\,W_t^{\rm evid}$ & Survival indicators at probe / evidence stages. \\
$\rho_t^\pi$ & Deployment bad-prior $\Pr_\pi(B_t^\epsilon\mid A_t)$. \\
$\sigma_t^{\rm bad,\pi},\,\sigma_t^{\rm good,\pi}$ & Probe-survival rates conditional on $B_t^\epsilon$ / $\neg B_t^\epsilon$. \\
$\bar\rho_t^{\mathrm{probe},\pi}$ & UCB on joint bad-and-probe-survive mass (Prop.~\ref{prop:probe-odds}). \\
$\bar b_t^\pi(m)$ & UCB on evidence-gate survival given bad and probe-survive. \\
$\Delta I_t(m),\,\mathcal I_t^{\mathrm{evid},\pi}$ & Per-step evidence $-\log\bar b_t^\pi(m)$ and cumulative evidence. \\
$\bar q_{t,\mathrm{mech}}^\pi$ & Mechanism-factorized UCB $\bar\rho_t^{\mathrm{probe},\pi}\exp[-\mathcal I_t^{\mathrm{evid},\pi}]$. \\
$\bar q_{t,\mathrm{ctrl}}^\pi$ & $\min\{\bar q_{t,\mathrm{raw}}^\pi,\bar q_{t,\mathrm{mech}}^\pi\}$. \\
$C_T(\pi),\,C_T^{\rm ctrl}(\pi)$ & Clean-before-false lower bound (raw / mechanism-tightened). \\
\midrule
\multicolumn{2}{l}{\textit{Confidence allocations and data splits}} \\
$\alpha,\,\alpha_t$ & False-admission risk budget (overall / per-step). \\
$\delta$ & Calibration failure probability. \\
$\delta_{\mathrm{raw}},\,\delta_{\mathrm{gate}},\,\delta_{\mathrm{probe}},\,\delta_{\mathrm{evid}}$ & Per-channel allocations of $\delta$ (each $=0.025$). \\
$D_{\rm hist}^{\rm mem}$ & Historical split used only for memory construction. \\
$D_{\rm cal}$ & Held-out split used for calibration and selection. \\
$D_{\rm test}$ & Target evaluation split (read-only $M^\star$). \\
\bottomrule
\end{tabular}
\caption{Consolidated notation for Section~\ref{sec:theory}. Each symbol used in Definitions~\ref{def:stopped-process}--\ref{def:controller-manifest}, Theorems~\ref{thm:risk-gate} and~\ref{thm:finite-class-selection}, Propositions~\ref{prop:probe-odds}--\ref{prop:experience-hazard}, and Corollary~\ref{thm:structural-certificate} appears here.}
\label{tab:glossary}
\end{table}

\input{appendix_proofs}

\section{Additional Empirical Analyses}
\label{sec:more_analysis}
This appendix probes the main results along seven diagnostic axes: whether Test Augmentation produces a useful signal at all (\S\ref{app:ta_diagnostics}); whether Dual-Granularity Verification absorbs the TA noise we measure (\S\ref{app:dgv_robustness}); whether the gains are an artifact of weak backbones (\S\ref{app:frontier_transfer}); whether the cost--accuracy position is genuinely on the efficiency frontier (\S\ref{app:efficiency_bcc}); where the tokens actually go (\S\ref{app:component_breakdown}); how the tool budget shifts with problem difficulty (\S\ref{app:tool_use_difficulty}); and whether Experience's hazard lift translates into observable per-problem savings in refinement steps and cost (\S\ref{app:experience_efficiency}). All seven analyses reuse the frozen manifest $\pi^\star$ from Section~\ref{sec:experiment}, leaving $\Pi_\alpha$ and $C_T^{\rm ctrl}(\pi^\star)$ unchanged; the detailed proofs of all theorems, propositions, and corollaries are provided in Appendix~\ref{app:proofs}.

\subsection{Test-Augmentation Diagnostics}
\label{app:ta_diagnostics}

This analysis evaluates whether Test Augmentation (TA) provides useful candidate-level evidence beyond public tests. To avoid circular validation, the augmented tests were generated by \model with the DeepSeek-V3.2-Chat backbone, while the evaluated candidate solutions were produced independently by DeepSeek-V3.2-Reasoner on the same LCB-Pro problems.

\paragraph{TA label accuracy.}
The diagnostic pool contains 117 problems where TA successfully generated augmented tests and the candidate solution had already passed all public tests. Among them, $51$ candidates are actually incorrect and $66$ are actually correct according to the online judge. A positive prediction means the augmented tests flag the candidate as buggy.

\begin{table}[h]
\footnotesize
\centering
\setlength{\tabcolsep}{4pt}
\begin{tabular}{lccc}
\toprule
 & Pred.\ Buggy & Pred.\ Correct & Total \\
\midrule
Actually Incorrect & TP $=40$ & FN $=11$ & 51 \\
Actually Correct   & FP $=19$ & TN $=47$ & 66 \\
\midrule
Total              & 59       & 58       & 117 \\
\bottomrule
\end{tabular}

\begin{tabular}{lc}
\toprule
Metric & Value \\
\midrule
Accuracy    & $87/117 = 74.4\%$ \\
Precision   & $40/59 = 67.8\%$ \\
Recall      & $40/51 = 78.4\%$ \\
F1          & $72.7\%$ \\
Specificity & $47/66 = 71.2\%$ \\
FPR         & $19/66 = 28.8\%$ \\
FNR         & $11/51 = 21.6\%$ \\
\bottomrule
\end{tabular}
\caption{Confusion matrix and derived metrics for the quality of TA-generated augmented test cases. Augmented tests are produced by \model with the DeepSeek-V3.2-Chat backbone; the evaluated candidates are produced independently by DeepSeek-V3.2-Reasoner on the same LCB-Pro problems.}
\label{tab:ta_label_accuracy}
\end{table}

Although TA is an imperfect verifier, it improves the quality of the accepted set in this diagnostic pool. Before TA filtering, the hidden-bug rate is $51/117 = 43.6\%$. Among candidates that survive TA, the hidden-bug rate drops to $11/58 = 19.0\%$, so the post-survival correctness rate is $47/58 = 81.0\%$.

\paragraph{Test-case diversity.}
We estimate diversity by randomly sampling $k$ generated test cases from each TA suite and measuring whether at least one sampled case exposes the incorrect candidate. The evaluation is conducted on the $51$ actually-incorrect candidates from the diagnostic pool. Results are mean $\pm$ standard deviation over $100$ random draws.

\begin{table}[h]
\footnotesize
\centering
\setlength{\tabcolsep}{8pt}
\begin{tabular}{cc}
\toprule
$k$ (sampled cases) & Bug detection rate \\
\midrule
1  & 37.6\% $\pm$ 5.8\% \\
2  & 51.4\% $\pm$ 5.1\% \\
3  & 58.9\% $\pm$ 4.4\% \\
5  & 65.7\% $\pm$ 3.3\% \\
9  & 70.0\% $\pm$ 2.0\% \\
15 & 74.5\% $\pm$ 0.8\% \\
\bottomrule
\end{tabular}
\caption{Bug-detection rate when randomly sampling $k$ augmented test cases per problem (LCB-Pro, DeepSeek-V3.2-Chat backbone).}
\label{tab:ta_diversity}
\end{table}

Detection grows steadily with $k$, from $37.6\%$ at $k{=}1$ to $74.5\%$ at $k{=}15$. The fact that additional samples keep adding bug-detection coverage---rather than saturating after the first one or two cases---indicates that TA cases cover complementary failure modes rather than only duplicating the same signal.

\subsection{Dual-Granularity Verification Robustness}
\label{app:dgv_robustness}

This analysis targets the main failure mode exposed by the TA diagnostic: verifier-level false positives. A false positive occurs when a correct candidate fails generated augmented tests because the generated expected outputs or test conditions are noisy.

\paragraph{Counterfactual study.}
We examine the $19$ verifier-level false-positive cases from the TA label-accuracy diagnostic. The comparison starts from the same intermediate candidate/state and runs the full \model refinement loop with TA versus without TA.

\begin{table}[h]
\footnotesize
\centering
\begin{tabular}{lc}
\toprule
Counterfactual outcome on 19 FP cases & Count \\
\midrule
Successful with or without TA & 14 \\
Failed with or without TA     & 5 \\
Succeeded without TA but failed with TA & 0 \\
\bottomrule
\end{tabular}
\caption{End-to-end counterfactual analysis on the $19$ verifier-level TA false-positive cases.}
\label{tab:dgv_counterfactual}
\end{table}

In these cases, verifier-level TA false positives do not produce observed end-to-end harm: no case succeeds without TA but fails with TA. Among the $14$ cases that remain successful in both conditions, TA introduces $5.64$ additional steps on average to recover.

\paragraph{Why recovery is possible.}
Dual-Granularity Verification (DGV) gives \model two independent executable feedback channels:
\begin{itemize}
\item Solution Validator (SV) and TA surface program-level failure signals.
\item Hypothesis Prober (HP) checks local claims, brute-force comparisons, edge cases, and alternative failure hypotheses.
\end{itemize}
When these channels disagree, \model arbitrates the evidence rather than blindly trusting a single noisy verifier.

\paragraph{Case study: Codeforces 2098B.}
\begin{figure}[h]
\centering
\includegraphics[width=\columnwidth]{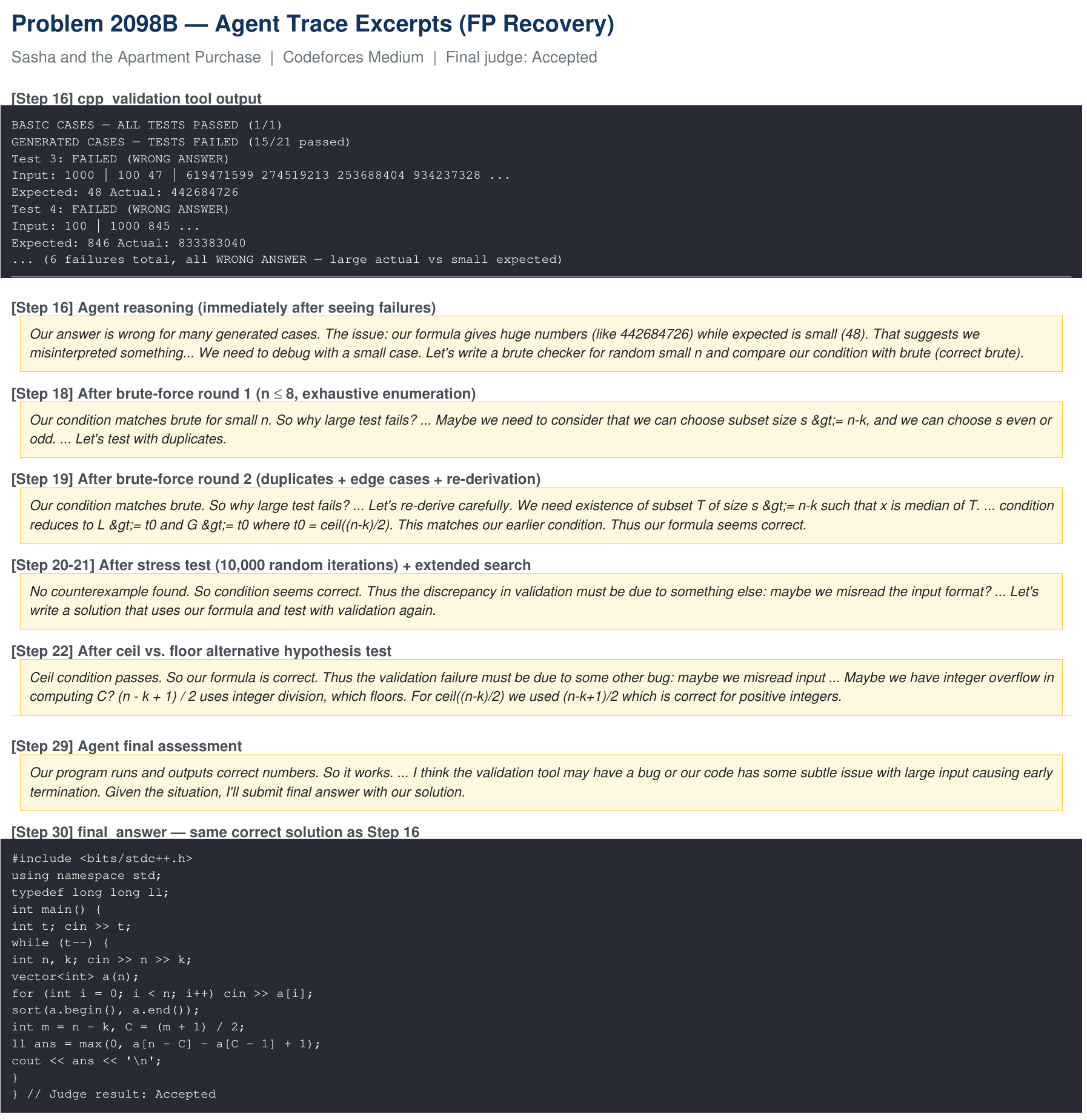}
\caption{Problem 2098B (``Sasha and the Apartment Purchase'') case study: a noisy TA alarm triggers extra HP validation but does not cause the agent to discard the correct formula-based solution.}
\label{fig:case_2098B}
\end{figure}
In Problem 2098B, ``Sasha and the Apartment Purchase'', a correct formula-based solution passes the public/basic case but fails several generated TA cases. The agent treats the signal as a possible bug and performs multiple independent HP checks:
\begin{itemize}
\item Step 18: first brute-force check on small instances.
\item Step 19: second brute-force round with duplicates, edge cases, and re-derivation.
\item Steps 20--21: stress testing and extended counterexample search.
\item Step 22: additional hypotheses such as ceil-vs-floor behavior and arithmetic issues.
\item Step 29: no genuine counterexample is found, so the agent retains the formula-based solution.
\item Step 30: the final submission is accepted.
\end{itemize}
This illustrates the practical role of DGV: a noisy TA alarm can trigger extra validation without causing the agent to discard a correct solution.

\subsection{Frontier-Backbone Transfer}
\label{app:frontier_transfer}

This experiment applies \model to the stronger DeepSeek-V3.2-Reasoner backbone. The goal is to test whether \model only helps weaker chat backbones or remains useful when the base model already has stronger reasoning ability.

\begin{table}[h]
\footnotesize
\centering
\setlength{\tabcolsep}{3pt}
\begin{tabular}{lcccccc}
\toprule
Model & ICPC R@5 & \#T & Easy & Med. & Hard & Avg. \\
\midrule
Reasoner (base)   & 33.9 & 1.6 & 84.5 & 24.6 & 7.7 & 58.1 \\
\model + Reasoner & 43.2 & 1.8 & 91.8 & 31.6 & 7.7 & 64.7 \\
$\Delta$          & +9.3 & --  & +7.3 & +7.0 & 0.0 & +6.6 \\
\bottomrule
\end{tabular}
\caption{\model applied on the frontier reasoning backbone DeepSeek-V3.2-Reasoner. ICPC R@5 is Refine@5 on ICPC-Eval; \#T is the average number of refinement turns; Easy/Med./Hard/Avg.\ are Pass@1 on LCB-Pro.}
\label{tab:frontier_reasoner}
\end{table}

\model improves the Reasoner backbone by $+9.3$ Refine@5 on ICPC-Eval and $+6.6$ average Pass@1 on LCB-Pro. Gains concentrate on Easy and Medium problems; Hard remains unchanged, consistent with the conclusion that hard problems require stronger high-level algorithmic discovery beyond feedback utilization alone.

\subsection{Efficiency and BCC Derivation}
\label{app:efficiency_bcc}

We analyze cost--performance trade-offs with Data Envelopment Analysis (DEA)~\cite{charnes1978measuring}, using the variable-returns-to-scale Banker--Charnes--Cooper (BCC) model~\cite{banker1984some}. The input is average API cost per problem, and the output is Pass@1. A method is BCC-efficient when no observed method, nor any convex combination of observed methods, can achieve at least the same output with no more input. The multi-backbone Pass@1-versus-cost view supporting this analysis is shown in the main-text Figure~\ref{fig:efficiency}.

\paragraph{Input-oriented BCC program.}
For each decision-making unit (DMU) $k$ with $(x_k,y_k)$, the input-oriented BCC efficiency is
\begin{align*}
\min_{\theta,\{\lambda_j\}} \quad & \theta \\
\text{s.t.}\quad
& \sum_{j} \lambda_j x_j \le \theta\, x_k, \\
& \sum_{j} \lambda_j y_j \ge y_k, \\
& \sum_{j} \lambda_j = 1, \\
& \lambda_j \ge 0,\ \forall j.
\end{align*}
The optimum $\theta_k^\star \in (0,1]$. $\theta_k^\star = 1$ means the method lies on the BCC frontier; $\theta_k^\star < 1$ means its input could be proportionally reduced by $(1-\theta_k^\star)$ while preserving at least the same output.

\paragraph{DeepSeek-V3.2-Chat data.}
Let $(x_j,y_j)$ denote $(\text{Cost},\text{Pass@1})$.

\begin{table}[h]
\footnotesize
\centering
\setlength{\tabcolsep}{4pt}
\begin{tabular}{lcccc}
\toprule
Method & Cost $x$ & Pass@1 $y$ & In Tok. & Out Tok. \\
\midrule
DeepSeek-V3.2-Chat (base) & 0.001161236 & 25.8 & 892.3     & 2{,}671.9  \\
Reflexion                 & 0.006798761 & 37.1 & 29{,}631.6  & 13{,}100.9 \\
CodeSim                   & 0.014984270 & 40.1 & 41{,}883.2  & 31{,}314.0 \\
AlphaCodium               & 0.023046801 & 26.9 & 89{,}478.1  & 45{,}552.7 \\
\model                    & 0.035377184 & 48.5 & 582{,}547.1 & 23{,}549.4 \\
\bottomrule
\end{tabular}
\caption{Per-method cost and Pass@1 on the DeepSeek-V3.2-Chat backbone.}
\label{tab:bcc_data}
\end{table}

AlphaCodium is Pareto-dominated by Reflexion because Reflexion has lower cost and higher Pass@1. The VRS/BCC frontier for the remaining DeepSeek points is
\[
\text{DeepSeek-V3.2-Chat (base)} \rightarrow \text{Reflexion} \rightarrow \model,
\]
and these three frontier points have $\theta = 1$.

\paragraph{Closed-form one-input one-output computation.}
With one input, one output, and the BCC convexity constraint, any feasible reference point is a convex combination of observed DMUs. For a target output $y$ between two frontier outputs $y_1 \le y \le y_2$, the minimal attainable input on the frontier is obtained by linear interpolation:
\begin{align*}
\lambda(y) &= \frac{y - y_1}{y_2 - y_1}, \\
x^\star(y) &= (1-\lambda(y))\,x_1 + \lambda(y)\,x_2, \\
\theta_k   &= \frac{x^\star(y_k)}{x_k}.
\end{align*}

\paragraph{CodeSim.}
CodeSim has $y_S = 40.1$, between Reflexion ($37.1$) and \model ($48.5$):
\begin{align*}
\lambda_S    &= \tfrac{40.1 - 37.1}{48.5 - 37.1} = \tfrac{3}{11.4} \approx 0.2631579, \\
x^\star(y_S) &= 0.7368421 \cdot 0.006798761 + 0.2631579 \cdot 0.035377184 \\
             &= 0.014319399, \\
\theta_S     &= \tfrac{0.014319399}{0.014984270} \approx 0.9556.
\end{align*}
CodeSim is slightly inefficient under VRS/BCC and could reduce cost by about $4.44\%$ while maintaining the same Pass@1.

\paragraph{AlphaCodium.}
AlphaCodium has $y_A = 26.9$, between DeepSeek-V3.2-Chat base ($25.8$) and Reflexion ($37.1$):
\begin{align*}
\lambda_A    &= \tfrac{26.9 - 25.8}{37.1 - 25.8} = \tfrac{1.1}{11.3} \approx 0.0973451, \\
x^\star(y_A) &= 0.9026549 \cdot 0.001161236 + 0.0973451 \cdot 0.006798761 \\
             &= 0.00171002162, \\
\theta_A     &= \tfrac{0.00171002162}{0.023046801} \approx 0.0742.
\end{align*}

\paragraph{Final DeepSeek VRS/BCC efficiencies.}
\begin{table}[h]
\footnotesize
\centering
\begin{tabular}{lcc}
\toprule
Method & BCC $\theta$ & Frontier status \\
\midrule
DeepSeek-V3.2-Chat (base) & 1.0000 & Frontier \\
Reflexion                 & 1.0000 & Frontier \\
CodeSim                   & 0.9556 & Inefficient \\
AlphaCodium               & 0.0742 & Inefficient \\
\model                    & 1.0000 & Frontier \\
\bottomrule
\end{tabular}
\caption{Final input-oriented VRS/BCC efficiencies for the DeepSeek-V3.2-Chat methods.}
\label{tab:bcc_efficiencies}
\end{table}

The multi-backbone frontier (Fig.~\ref{fig:efficiency}) extends this analysis to DeepSeek-V3.2-Chat, GPT-4o, and Qwen3-235B. \model lies on the BCC frontier for all three backbones, whereas some baselines are frontier-efficient only under particular backbones.

\subsection{Per-Component Token and Cost Breakdown}
\label{app:component_breakdown}

This table attributes \model's API token usage and cost across its four main components on LCB-Pro. Results use the DeepSeek-V3.2-Chat backbone and average over $167$ problems.

HP and SV are non-LLM tools for code execution and validation; their token counts measure the associated LLM reasoning overhead, including accumulated multi-turn conversation context. TA and Experience are separate LLM calls with self-contained prompts. Costs are computed using DeepSeek API pricing.

\begin{table}[h]
\footnotesize
\centering
\setlength{\tabcolsep}{3pt}
\begin{tabular}{lrrrr}
\toprule
Component & In Tok. & Out Tok. & Cost & \% Total \\
\midrule
HP (Hypothesis Prober)  & 392{,}076 & 17{,}898 & \$0.025 & 69.7\% \\
SV (Solution Validator) & 115{,}532 & 2{,}826  & \$0.006 & 17.6\% \\
Exp (Experience)        & 20{,}477  & 706      & \$0.001 & 3.4\%  \\
TA (Test Augmentation)  & 9{,}078   & 471      & \$0.001 & 1.7\%  \\
\midrule
Total (tracked)         & 537{,}163 & 21{,}901 & \$0.033 & 92.4\% \\
\bottomrule
\end{tabular}
\caption{Per-component API token usage of \model on LCB-Pro (DeepSeek-V3.2-Chat backbone, averaged over $167$ problems).}
\label{tab:component_breakdown}
\end{table}

The tracked components account for $92.4\%$ of total API cost. The remaining $\sim 7.6\%$ comes from problem reading, system prompts, and submission formatting. HP dominates cost because it is invoked repeatedly inside the multi-turn refinement loop, while TA and Experience together account for only $5.1\%$ of total cost.

\subsection{Tool Use vs.\ Problem Difficulty}
\label{app:tool_use_difficulty}

Figure~\ref{fig:verification_tool_freq} analyzes verification-tool usage along the \emph{interaction-step} axis on a single backbone. This subsection complements that view with the orthogonal \emph{difficulty} axis: how much tool-mediated effort \model spends per problem as problems become harder. The diagnostic uses the DeepSeek-V3.2-Chat backbone on LCB-Pro and counts the average number of HP/SV/TA invocations per problem, broken down by the benchmark's Easy/Medium/Hard difficulty bands.

\begin{figure}[h]
\centering
\includegraphics[width=0.85\columnwidth]{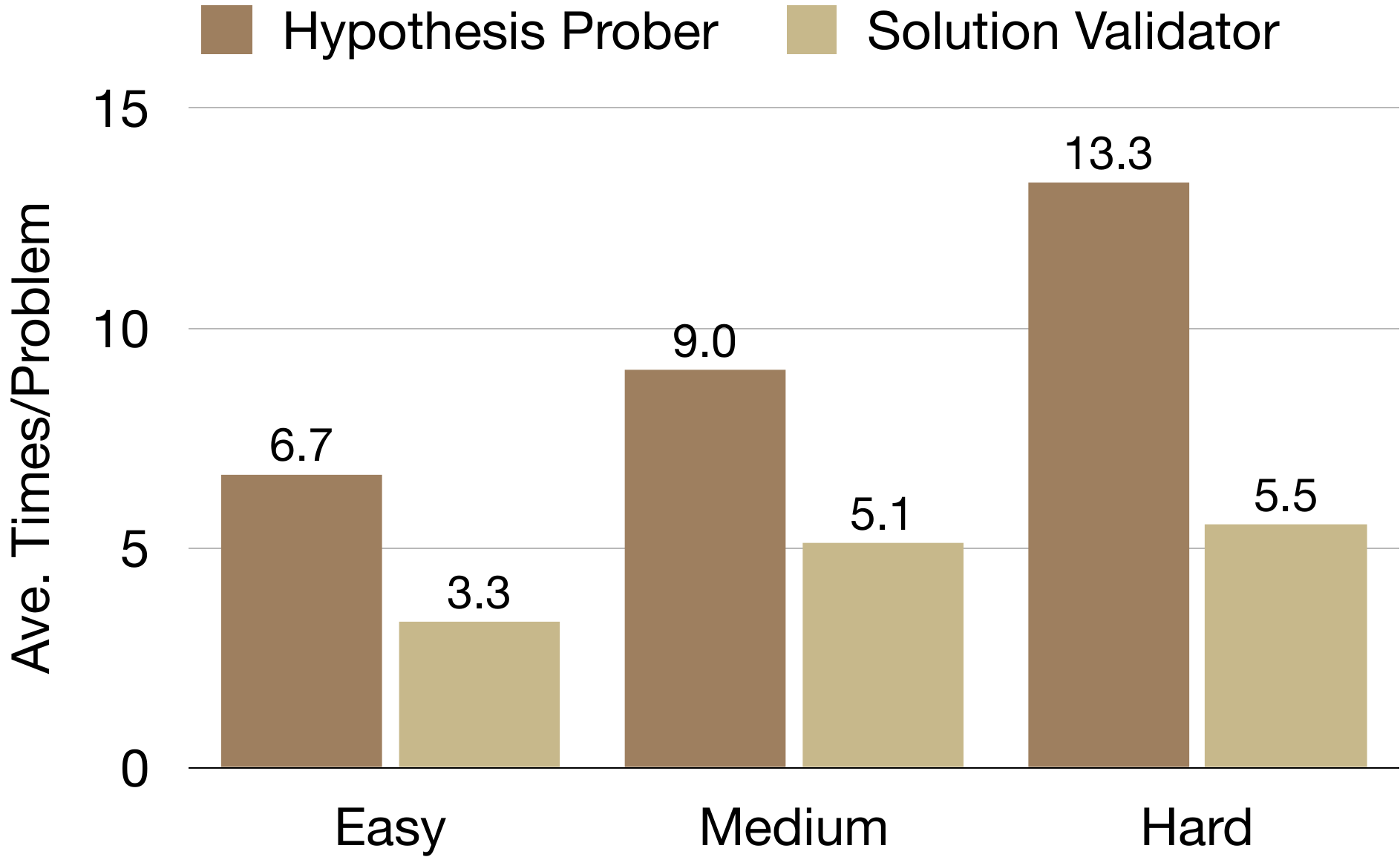}
\caption{Average tool-use frequency (calls per problem) on LCB-Pro by difficulty band (Easy/Medium/Hard), broken down by tool channel (HP, SV, TA). DeepSeek-V3.2-Chat backbone. See body for the resulting per-band totals and their interpretation.}
\label{fig:tool_use_difficulty}
\end{figure}

\paragraph{Adaptive scaling on Easy/Medium.}
Compared with a fixed-pipeline baseline (each method calls every tool a constant number of times), \model issues $5.5$ tool calls per problem on Easy and $5.1$ on Medium. The slight drop from Easy to Medium reflects two competing effects under the certificate machinery of Section~\ref{sec:theory}: (i) Medium problems trigger longer chain-of-thought before an executable candidate is produced, so HP and SV start firing later in the trajectory; (ii) once an executable candidate is produced, the admission gate $\Gamma_t$ continues to require evidence, so tool calls remain non-trivial. Empirically the two effects roughly cancel, and \model still relies heavily on external feedback on Easy and Medium problems.

\paragraph{Why Hard drops to $3.3$.}
On Hard problems, average tool calls drop to $3.3$. This is consistent with two structural facts already documented in the main paper:

\begin{itemize}
\item \emph{Few executable candidates are reached.} Hard problems often consume the budget on high-level algorithmic discovery before any candidate is produced. SV is only meaningful on executable candidates, and TA is only invoked when SV first returns pass on public tests; both channels stay neutral on trajectories that never reach this state. This is the same structural barrier that produces the unchanged Hard column in Table~\ref{tab:frontier_reasoner}: refinement-loop tools cannot substitute for missing high-level ideas.
\item \emph{Public-test-only filter.} The augmented suite $\mathcal T_{\rm aug}$ is generated at most once per problem (Algorithm~\ref{alg:orchestration}), conditional on the candidate first passing public tests. On Hard problems this trigger fires rarely, suppressing the TA contribution to the per-problem call count.
\end{itemize}

In other words, the drop on Hard is not a controller-level decision to invoke fewer tools, but a downstream consequence of the candidate stream itself becoming sparse: the gates $G^{\mathrm{HP}}_{i,t}, G^{\mathrm{SV}}_{i,t}, G^{\mathrm{TA}}_{i,t}$ from Algorithm~\ref{alg:orchestration} record neutral on inactive channels (Definition~\ref{def:stopped-process}). Combined with Fig.~\ref{fig:verification_tool_freq}, which shows HP concentrated in early interaction steps, this gives a coherent picture: \model spends tool budget where evidence is actionable, and remains efficient elsewhere by leaving inactive channels neutral rather than firing them spuriously.

\subsection{Experience-Driven Efficiency}
\label{app:experience_efficiency}

Section~\ref{sec:memory} attributes the Experience module to $\textsf{ContextAcquire}(k)$, whose effect in Proposition~\ref{prop:experience-hazard} is to lift the active-state success hazard $\underline h_t(M)$. Beyond the per-step acceptance lift reported in Table~\ref{tab:exp_steps}, a higher $\underline h_t$ should also let \model commit earlier in the refinement loop, with downstream effects on both the average number of refinement steps and the per-problem cost. We isolate this efficiency channel by an internal ablation of $\pi^\star$ that toggles only the frozen Experience snapshot $M^\star$ on the LCB-Pro target split (DeepSeek-V3.2-Chat backbone, manifest of Table~\ref{tab:controller_choices}).

\begin{figure}[h]
\centering
\includegraphics[width=0.85\columnwidth]{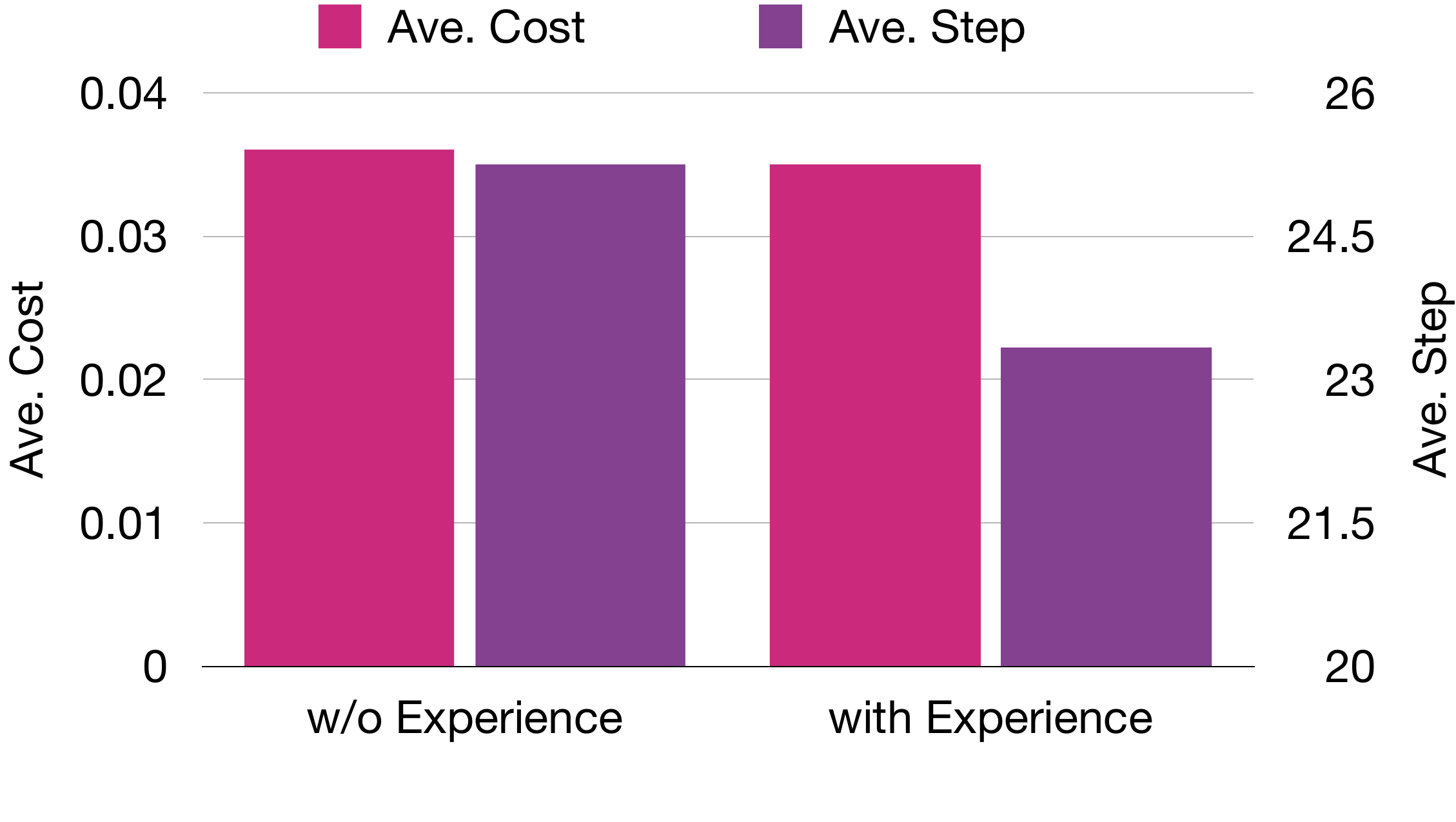}
\caption{Average refinement steps and average per-problem API cost on LCB-Pro with and without the frozen Experience snapshot $M^\star$ (DeepSeek-V3.2-Chat backbone, $\pi^\star$ from Table~\ref{tab:controller_choices}). Both axes are averaged over all $167$ LCB-Pro problems.}
\label{fig:experience_efficiency}
\end{figure}

\paragraph{Both axes move in the right direction.}
With Experience, the average refinement step count drops and the average per-problem cost drops jointly (Fig.~\ref{fig:experience_efficiency}). The same direction shows up in the component ablation of Table~\ref{tab:calibrated_certificate}: adding Experience on top of $+\textsf{SV}+\textsf{HP}+\textsf{TA}$ moves \model from \$$0.043$ to \$$0.035$ per problem and improves the only-solved $\#T$ on ICPC-Eval from $2.2$ to $1.8$, while raising LCB-Pro Pass@1 from $46.1$ to $48.5$. Two measurement axes---LCB-Pro all-problem average in Fig.~\ref{fig:experience_efficiency} and ICPC-Eval only-solved $\#T$ in Table~\ref{tab:calibrated_certificate}---thus agree on the sign of the effect, while differing in absolute magnitude due to the averaging convention.

\paragraph{Mechanism interpretation.}
The cost reduction is not driven by Experience being a cheaper module: Table~\ref{tab:component_breakdown} shows Experience itself consumes only $\sim\!3.4\%$ of the total tracked API cost. Instead, the savings act through the certificate state. A higher $\underline h_t(M)$ lets the admission gate $\Gamma_t$ close earlier on the trajectories that would otherwise have continued, which in turn truncates the HP/SV invocations downstream of those refinement steps. Since HP and SV together account for $87.3\%$ of tracked cost (Table~\ref{tab:component_breakdown}), the saved HP/SV calls dominate the additional retrieval cost incurred by Experience. In the language of Section~\ref{sec:theory}, $\textsf{ContextAcquire}$ converts hazard lift into refinement-step savings, and the savings on the expensive verification channels outweigh the retrieval overhead---giving the qualitative claim in Section~\ref{sec:experiment} that ``retrieved repair experiences reduce ineffective trial-and-error'' a concrete per-problem efficiency signature.

\section{Datasets}
\label{app:datasets}
\paragraph{ICPC-Eval~\cite{xu2026icpc}} is a competitive-programming benchmark designed to reflect real contest settings. It is built from problems drawn from $11$ recent ICPC contests (collected via QOJ and VJudge): starting from $139$ problems, the dataset is cleaned and finalized to $118$ problems. It filters out problems that rely on crucial non-textual figures, interactive problems, or problems without an official reference solution. For tasks requiring special judging, ICPC-Eval retains them and additionally implements $12$ custom special judges (e.g., for floating-point outputs or multiple valid answers). All problem statements are standardized into LaTeX for consistent parsing and model input. The benchmark matches ICPC-style problem diversity and difficulty, and provides annotations over eight major algorithmic areas (e.g., DP, mathematics, graph algorithms, and data structures), with many problems spanning multiple areas to capture the compositional and challenging nature of ICPC problems.

\paragraph{LiveCodeBench Pro~\cite{zheng2026livecodebench}} is a challenging benchmark for evaluating large language models on competitive programming. Its central goal is to mitigate overestimation caused by training-data contamination by maintaining a continuously updated problem set. The benchmark curates problems from top-tier contests such as Codeforces, ICPC, and IOI. Problems are collected as soon as possible after they appear in real contests, further reducing the risk of data leakage.

\section{Metrics}
\label{app:metrics}
\paragraph{Refine@K~\cite{xu2026icpc}} is an execution-feedback-based evaluation metric designed to measure an LLM's ability to iteratively improve a solution under a bounded turn budget. In ICPC-Eval, each problem is solved through up to $K$ response--refinement chances, and the metric asks whether the model can ``pass the test within response and refinement chances'' (i.e., succeed within $K$ turns). Concretely, the evaluation is multi-turn: in the first turn, the model receives the full problem statement; in subsequent turns, the model is additionally provided with its previous response and the corresponding evaluation feedback. Thus, a model is counted as successful under Refine@K if any of its first $K$ attempts (conditioned on prior code + feedback) produces a correct program. Compared with Pass@K (independent resampling), Refine@K better captures the behavior of reasoning/reflection-capable models because it explicitly evaluates whether a model can use feedback to correct mistakes rather than merely getting lucky with another sample.

\paragraph{\# Turn~\cite{xu2026icpc}} represents the average number of correct response turns, computed only over problems that are successfully solved within the refinement budget. Specifically, for each solved problem, we record the turn index at which the model first produces a correct solution, and then average this index across the solved subset. A smaller \#T indicates the model typically fixes its solution within fewer turns, suggesting more efficient use of execution feedback, whereas a larger \#T implies the model often needs more trial-and-error refinement rounds before arriving at a correct program.

\paragraph{Pass@1} measures the single-try success rate of code generation: for each problem, the model produces exactly one candidate solution (one response/submission). If this first solution passes the evaluation (e.g., all test cases), the problem is counted as solved. Pass@1 is the fraction of problems solved on the first attempt:
\begin{align*}
\mathrm{Pass@1} = \frac{1}{N}\sum_{i=1}^{N}\mathbb{I}\big[
&\text{the model's first solution for}\\[-2pt]
&\text{problem } i \text{ passes}\big].
\end{align*}
Intuitively, Pass@1 reflects the model's \emph{initial answer quality} without any retries or feedback-driven refinement, in contrast to iterative metrics such as Refine@K.

\input{appendix_prompts}

{%
\renewcommand{\promptboxsep}{3pt}%
\renewcommand{\promptboxfont}{\ttfamily\fontsize{6.4pt}{7.2pt}\selectfont}%
\renewcommand{\promptverbfont}{\fontsize{6.4pt}{7.2pt}\selectfont}%
\input{appendix_case_study}}

\section{Extended Discussion on Related Work}
\label{app:related}
\paragraph{LLM Agents.}
Large Language Models have been increasingly deployed as autonomous agents capable of interacting with external environments and solving multi-step tasks. In the domain of web navigation, researchers have developed agents to autonomously browse and interact with online interfaces~\cite{ning2025survey, he2024webvoyager, lai2024autowebglm, zhang2025webpilot}. In the field of software engineering, agents such as SWE-Agent have been proposed to handle repository-level maintenance and development tasks~\cite{luo2024repoagent, gao2025trae}. Furthermore, the application of agents has extended to high-level intellectual domains, including automated scientific discovery~\cite{toledo2026ai, baek2025researchagent, lu2024ai} and data science analysis, where agents assist in data visualization and interpretation~\cite{zhu2025survey, yang2024matplotagent, jiang2024siriusbi}. These advancements collectively demonstrate the potential of LLM agents to generalize across diverse and complex workflows.

In contrast to the extensive body of work on autonomous agents interacting with dynamic external environments, a more focused stream of research investigates Agentic Reasoning---the use of deterministic tools to enhance the model's internal cognitive processes. Unlike autonomous agents that use tools as ``effectors'' to change the environment state, this paradigm employs tools---primarily code interpreters (e.g., Python) and calculators---as ``cognitive prosthetics'' to ground the reasoning process itself. Early works such as PAL~\cite{gao2023pal} and PoT~\cite{chen2022program} demonstrated that decoupling reasoning from computation significantly reduces hallucinations in mathematical and symbolic tasks. Advancing this trajectory, recent frameworks~\cite{wu2025agentic, gou2024tora, das2024mathsensei, shang2025rstar2} establish a more sophisticated interleaved reasoning--action loop. By integrating code execution directly into the chain of thought, these methods utilize deterministic feedback for iterative step verification and error correction, enabling the model to navigate complex logical paths with higher reliability. Although fewer in number compared to general agent papers, these studies highlight a critical distinction: for challenging reasoning tasks, the value of a tool lies not in its complexity or connectivity, but in its ability to provide deterministic, verifiable feedback that constrains the model's probabilistic output.

\paragraph{Competitive Programming.}
Program synthesis has evolved significantly from early deductive~\cite{manna1980deductive} and inductive~\cite{spector2002genetic} synthesis to neural approaches like DeepCoder~\cite{balog2016deepcoder}. However, the release of challenging benchmarks~\cite{hendrycks2021measuring, jain2025livecodebench, zheng2026livecodebench, xu2026icpc, sun2026algbench} revealed a harsh reality: early models failed near-completely on competition-level problems. AlphaCode~\cite{li2022competition} marked a major milestone in this domain by leveraging massive sampling for competitive programming problems, while AlphaCode2~\cite{leblond2023alphacode} and OpenAI's o1-ioi~\cite{el2025competitive} later adopted reinforcement learning and heuristic inference strategies to achieve gold-medal performance. Most recently, OpenAI o3~\cite{openaio3} demonstrated that general-purpose reasoning models could reach Grandmaster status on Codeforces without domain-specific hand-crafting.

Since training such reasoning models is computationally prohibitive, recent research has pivoted to inference-time enhancements using hand-crafting LLM workflows like MapCoder~\cite{islam2024mapcoder}, AlphaCodium~\cite{ridnik2024code}, and CODESIM~\cite{islam2025codesim}. These methods simulate human problem-solving through iterative ``generate-and-verify'' cycles. However, they mainly rely on fixed, stateless workflows, which suffer from rigidity: they lack the adaptivity to scale reasoning effort within a problem, and crucially, they fail to transfer experience across problems, discarding valuable debugging insights after each session. This limitation highlights the need for more dynamic and long-term evolving mechanisms.

\end{document}

%% file: appendix_proofs.tex
\section{Section 2 Proofs}
\label{app:proofs}
\begingroup
\small
\sloppy

This appendix contains the proofs for Section 2 ("A Calibrated Risk-Controlled Theory of Feedback Control") of the paper, covering the calibration estimators after Definition 2.4, Theorem 2.5, Theorem 2.6, Propositions 2.7--2.9, Corollary 2.10, and the validity of the mechanism factorization in Eq. (9).

\subsection*{0. Setting and Conventions}

\subsubsection*{0.1 Probability space and stopped events}
\mbox{}\par\nobreak

Fix a frozen controller $\pi$. A trajectory is generated by sampling one problem from the benchmark distribution and running $\pi$ for at most $T$ refinement steps. All probabilities $\Pr_\pi(\cdot)$ are taken with respect to a fresh evaluation trajectory; calibration probabilities are taken with respect to held-out calibration trajectories.

The stopping times are

\[
\tau_F := \inf\{t : F_t = 1\}, \qquad
\tau_S := \inf\{t : S_t = 1\},
\]

and the active event before step $t$ is

\[
A_t := \{\tau_S > t-1,\; \tau_F > t-1\}.
\]

By Definition 2.1,

\[
F_t := \mathbf 1\{A_t,\,\Gamma_t = 1,\,B_t^\epsilon\}, \qquad
S_t := \mathbf 1\{A_t,\,\Gamma_t = 1,\,\neg B_t^\epsilon\},
\]

so $F_t$ and $S_t$ are mutually exclusive at every fixed step $t$ via $B_t^\epsilon$. Hence $\tau_S$ and $\tau_F$ cannot both be finite and equal (they may, however, both equal $+\infty$ when neither stop occurs by $T$, with the convention $\inf\varnothing := +\infty$). The certified event is

\[
\{\tau_S \le T,\; \tau_S < \tau_F\}.
\]

We adopt the stopped-process convention: once a clean success or false admission occurs at some step $t^\star$, $S_t$ and $F_t$ are set to $0$ at all subsequent steps $t > t^\star$ and the trajectory is inactive. As a consequence, if $\tau_S < \infty$ then $F_t = 0$ for all $t > \tau_S$, so $\tau_F = +\infty$ on $\{\tau_S < \tau_F\}$; symmetrically, $\tau_S = +\infty$ on $\{\tau_F < \tau_S\}$. We further assume $\Pr_\pi(A_1) = 1$, i.e., the trajectory begins active at step $1$.

\subsubsection*{0.2 Admission-gate semantics required by the mechanism factorization}
\mbox{}\par\nobreak

Eq. (9) and Section 5 below use the containment

\[
\{F_t = 1\} \subseteq \{B_t^\epsilon = 1,\; W_t^{\mathrm{probe}} = 1,\; W_t^{\mathrm{evid}} = 1\} \quad \text{on } A_t.
\]

This is equivalent to the requirement that the observable admission gate satisfies

\[
\Gamma_t = 1 \;\Longrightarrow\; W_t^{\mathrm{probe}} = 1 \;\text{and}\; W_t^{\mathrm{evid}} = 1 \quad \text{on } A_t,
\]

i.e., any candidate released by $\Gamma_t$ has first passed the risk-probe and evidence stages (with the convention $W^{\mathrm{probe}} = W^{\mathrm{evid}} = 1$ when the corresponding stage is not invoked). We treat this implication as part of the manifest contract on $\pi$; if $\Gamma_{1:T}^{\mathrm{obs}}$ does not chain the upstream screens, the mechanism UCB is invalid and only the raw UCB $\bar q_{t,\mathrm{raw}}^\pi$ applies.

\subsubsection*{0.3 Held-out exchangeability}
\mbox{}\par\nobreak

Under Assumption 2.3, calibration and evaluation trajectories are independent draws from the same benchmark distribution. Independence is required across problem trajectories at a fixed step $t$, not across refinement steps within a trajectory. For fixed $t$ and $\pi$, the active calibration rows yield conditional Bernoulli samples; a random active denominator does not invalidate the bounds, because we condition on the denominator first and then mix.

\subsubsection*{0.4 Zero-probability conditioning}
\mbox{}\par\nobreak

For any conditional probability $\Pr(H \mid E)$ with $\Pr(E) = 0$, we set $\Pr(H \mid E) := 0$. This convention never inflates an upper bound; in chain-rule decompositions, the corresponding event has probability $0$ and contributes nothing.

\subsubsection*{0.5 Clopper--Pearson UCB / LCB}
\mbox{}\par\nobreak

Let $Y_1, \dots, Y_n \overset{\text{i.i.d.}}{\sim} \mathrm{Bernoulli}(p)$ and $K = \sum_i Y_i$. Write

\[
U(K, n, \eta) = \mathrm{UCB}_{\mathrm{Binom}}(K, n, \eta), \qquad
L(K, n, \eta) = \mathrm{LCB}_{\mathrm{Binom}}(K, n, \eta)
\]

for the exact one-sided Clopper--Pearson bounds, so that

\[
\Pr_p\{p \le U(K, n, \eta)\} \ge 1 - \eta, \qquad
\Pr_p\{p \ge L(K, n, \eta)\} \ge 1 - \eta.
\]

If the denominator $N$ is random with $K \mid N = n \sim \mathrm{Binomial}(n, p)$, then

\[
\Pr\{p > U(K, N, \eta)\}
= \sum_n \Pr(N = n)\,\Pr\{p > U(K, n, \eta) \mid N = n\}
\le \eta,
\]

and the analogous bound holds for the LCB. We use the conservative conventions

\[
\mathrm{UCB}_{\mathrm{Binom}}(0, 0, \eta) = 1, \qquad
\mathrm{LCB}_{\mathrm{Binom}}(0, 0, \eta) = 0.
\]

\subsubsection*{0.6 Pre-declaration of score-defining objects}
\mbox{}\par\nobreak

The proof of Theorem 2.5 requires that the score function $r_t(Z_t) = \rho_t \exp(-I_t)$, the finite threshold grid $\Theta_t$ (with $1 \le |\Theta_t| < \infty$), and the gate rule $\mathbf 1\{r_t \le \theta\}$ be determined by $\pi$ alone, using only the observable trajectory state. Concretely:

\begin{enumerate}
\item The measurable map $\phi_\pi : X_t \mapsto Z_t = (\rho_t, I_t, \underline h_t, t)$ depends only on observable mechanism outputs (probe survival, evidence-gate survival, retrieval slots, step index), not on hidden labels $B_t^\epsilon$.
\item The finite grid $\Theta_t$ is determined by $\pi$ with $1 \le |\Theta_t| < \infty$.
\item The gate rule $\mathbf 1\{r_t \le \theta\}$ is determined by $\pi$.
\end{enumerate}

If $\phi_\pi$, $\Theta_t$, the aggregation rule $\gamma_{\mathrm{evid}}$, the memory snapshot $M^\star$, or any calibrated scalar in $r_t$ (such as $\bar\rho_t^{\mathrm{probe},\pi}$ or $\bar b_t^\pi$) is tuned on $D_{\mathrm{cal}}$, every candidate value must be enumerated in the finite, pre-declared manifest class $\Pi_\alpha$ (with $1 \le |\Pi_\alpha| < \infty$) in advance, so that the union bound of Theorem 2.6 absorbs both the score selection and the gate threshold selection. Throughout Theorem 2.5 and Theorem 2.6 we treat this pre-declaration as a standing hypothesis.

\subsection*{1. Validity of the Calibration Estimators in Eq. (7)}

For fixed $\pi$ and step $t$, let $n_t^\pi$ be the number of active calibration rows, $f_t^\pi$ the count of false admissions, and $s_t^\pi$ the count of clean successes. Define

\[
\bar q_{t,\mathrm{raw}}^\pi := \mathrm{UCB}_{\mathrm{Binom}}\!\left(f_t^\pi, n_t^\pi, \tfrac{\delta}{2T}\right),
\qquad
\underline h_t^\pi := \mathrm{LCB}_{\mathrm{Binom}}\!\left(s_t^\pi, n_t^\pi, \tfrac{\delta}{2T}\right).
\tag{7}
\]

The deployment hazards are $q_t^\pi := \Pr_\pi(F_t = 1 \mid A_t)$ and $h_t^\pi := \Pr_\pi(S_t = 1 \mid A_t)$.

\textbf{Proof.} Fix $t$ and $\pi$. Let $A_{i,t}$ be the active indicator for calibration trajectory $i$. On $\{A_{i,t} = 1\}$, $F_{i,t}$ is Bernoulli with parameter $q_t^\pi$. By Assumption 2.3, different trajectories are independent at fixed $t$, so the active rows yield i.i.d. $\mathrm{Bernoulli}(q_t^\pi)$ samples; the random denominator $n_t^\pi$ is handled by Section 0.5. The Clopper--Pearson UCB gives

\[
\Pr\{q_t^\pi \le \bar q_{t,\mathrm{raw}}^\pi\} \ge 1 - \tfrac{\delta}{2T}.
\]

The same argument applied to $S_{i,t}$ gives $\Pr\{h_t^\pi \ge \underline h_t^\pi\} \ge 1 - \tfrac{\delta}{2T}$. A union bound over the $2T$ one-sided events $t = 1, \dots, T$ yields

\[
\Pr\!\left\{\forall t \le T:\;
q_t^\pi \le \bar q_{t,\mathrm{raw}}^\pi
\;\text{and}\;
h_t^\pi \ge \underline h_t^\pi\right\} \ge 1 - \delta. \qquad\square
\]

The argument requires only across-trajectory independence at fixed $t$, not within-trajectory step independence.

\subsection*{2. Theorem 2.5: Calibrated Admission Gate}

\textbf{Statement.} Assume the pre-declaration of Section 0.6: $\phi_\pi$, the score function $r_t(Z_t) = \rho_t \exp(-I_t)$, the finite non-empty grid $\Theta_t$ (with $1 \le |\Theta_t| < \infty$), and the gate rule are fixed by $\pi$ before label counting; any calibrated scalar entering $r_t$ is enumerated in the finite, non-empty manifest $\Pi_\alpha$ (with $1 \le |\Pi_\alpha| < \infty$). On held-out calibration traces, let $n_t(\theta)$ be the number of active candidate admissions with $r_t(Z_t) \le \theta$ and $f_t(\theta)$ the count of such admissions rejected by hidden evaluation. Define

\[
\bar R_t(\theta) := \mathrm{UCB}_{\mathrm{Binom}}\!\left(f_t(\theta),\, n_t(\theta),\, \tfrac{\delta_{\mathrm{gate}}}{T |\Theta_t|}\right)
\]

and select

\[
\hat\theta_{t,\alpha} = \max\{\theta \in \Theta_t : \bar R_t(\theta) \le \alpha_t\},
\tag{8}
\]

with $\hat\theta_{t,\alpha} := -\infty$ and $\Gamma_t^\alpha \equiv 0$ when the set is empty. Then, with calibration probability at least $1 - \delta_{\mathrm{gate}}$, the gate $\Gamma_t^\alpha = \mathbf 1\{r_t(Z_t) \le \hat\theta_{t,\alpha}\}$ satisfies

\[
\Pr(F_t = 1 \mid A_t,\, \Gamma_t^\alpha = 1) \le \alpha_t
\]

simultaneously for all $t \le T$ and all data-dependent $\hat\theta_{t,\alpha} \in \Theta_t$.

\textbf{Proof.} Section 0.6 asserts that $\phi_\pi$, $\Theta_t$, and the gate rule are pre-declared (or absorbed into $\Pi_\alpha$ by Theorem 2.6). Fix $t$ and $\theta \in \Theta_t$, and define the deployment conditional rejection rate

\[
p_{t,\theta} := \Pr_\pi(F_t = 1 \mid A_t,\, r_t(Z_t) \le \theta).
\]

The conditioning event $\{A_t,\, r_t(Z_t) \le \theta\}$ is determined by observables. By Assumption 2.3, the calibration rows satisfying this event yield i.i.d. $\mathrm{Bernoulli}(p_{t,\theta})$ samples for hidden rejection, with a random denominator $n_t(\theta)$ handled by Section 0.5. The Clopper--Pearson UCB gives

\[
\Pr\{p_{t,\theta} \le \bar R_t(\theta)\} \ge 1 - \tfrac{\delta_{\mathrm{gate}}}{T |\Theta_t|}.
\]

Union-bounding over $t \le T$ and $\theta \in \Theta_t$ defines the good event

\[
G_{\mathrm{gate}} := \{\forall t \le T,\, \forall \theta \in \Theta_t:\; p_{t,\theta} \le \bar R_t(\theta)\},
\quad
\Pr(G_{\mathrm{gate}}) \ge 1 - \delta_{\mathrm{gate}}.
\]

On $G_{\mathrm{gate}}$, consider the data-dependent choice $\hat\theta_{t,\alpha}$.

\emph{Case 1: the selection set is non-empty.} Then $\hat\theta_{t,\alpha} \in \Theta_t$ and $\bar R_t(\hat\theta_{t,\alpha}) \le \alpha_t$. Since $G_{\mathrm{gate}}$ holds at every grid point,

\[
p_{t,\hat\theta_{t,\alpha}} \le \bar R_t(\hat\theta_{t,\alpha}) \le \alpha_t,
\]

and $\Gamma_t^\alpha = 1$ is equivalent to $r_t(Z_t) \le \hat\theta_{t,\alpha}$, so

\[
\Pr(F_t = 1 \mid A_t,\, \Gamma_t^\alpha = 1) = p_{t,\hat\theta_{t,\alpha}} \le \alpha_t.
\]

\emph{Case 2: the selection set is empty.} Then $\Gamma_t^\alpha \equiv 0$, and $\Pr(A_t,\, \Gamma_t^\alpha = 1) = 0$; by Section 0.4 the conditional probability is $0 \le \alpha_t$. $\qquad\square$

The post-selection validity follows from establishing the UCB simultaneously over the finite grid before any data-dependent choice is made.

\subsection*{3. Proposition 2.7: RiskProbe Bad-and-Survive UCB}

\textbf{Statement.} Let $W_t^{\mathrm{probe}} \in \{0, 1\}$ be the probe-survival indicator at step $t$ (set to $1$ when the probe is not invoked). On held-out calibration traces, let $n_t^\pi$ be the number of active rows at step $t$ under $\pi$, and let $f_t^{\mathrm{probe},\pi}$ count those rows where both $B_t^\epsilon = 1$ and $W_t^{\mathrm{probe}} = 1$. Define

\[
\bar\rho_t^{\mathrm{probe},\pi} := \mathrm{UCB}_{\mathrm{Binom}}\!\left(f_t^{\mathrm{probe},\pi},\, n_t^\pi,\, \tfrac{\delta_{\mathrm{probe}}}{T}\right).
\]

Then, with calibration probability at least $1 - \delta_{\mathrm{probe}}$, simultaneously for $t \le T$,

\[
\bar\rho_t^{\mathrm{probe},\pi} \ge \Pr_\pi(B_t^\epsilon,\, W_t^{\mathrm{probe}} = 1 \mid A_t),
\]

and

\[
\Pr_\pi(B_t^\epsilon,\, W_t^{\mathrm{probe}} = 1 \mid A_t)
= \sigma_t^{\mathrm{bad},\pi}\,\rho_t^\pi,
\]

where $\sigma_t^{\mathrm{bad},\pi} := \Pr_\pi(W_t^{\mathrm{probe}} = 1 \mid B_t^\epsilon,\, A_t)$ and $\rho_t^\pi := \Pr_\pi(B_t^\epsilon \mid A_t)$.

\textbf{Proof.} Fix $t$. For calibration trajectory $i$, let

\[
Y_{i,t}^{\mathrm{probe}} := \mathbf 1\{A_{i,t} = 1,\, B_{i,t}^\epsilon = 1,\, W_{i,t}^{\mathrm{probe}} = 1\}
\]

and consider it only on active rows. On $\{A_{i,t} = 1\}$,

\[
\mathbf 1\{B_{i,t}^\epsilon = 1,\, W_{i,t}^{\mathrm{probe}} = 1\}
\]

is Bernoulli with parameter $p_t^{\mathrm{probe},\pi} := \Pr_\pi(B_t^\epsilon,\, W_t^{\mathrm{probe}} = 1 \mid A_t)$. By Assumption 2.3, the active rows yield i.i.d. samples; Section 0.5 handles the random denominator. The Clopper--Pearson UCB gives

\[
\Pr\{p_t^{\mathrm{probe},\pi} \le \bar\rho_t^{\mathrm{probe},\pi}\} \ge 1 - \tfrac{\delta_{\mathrm{probe}}}{T},
\]

and a union bound over $t = 1, \dots, T$ yields the simultaneous statement with confidence $1 - \delta_{\mathrm{probe}}$.

The factorization follows from the chain rule of conditional probability (no independence is required):

\[
\begin{aligned}
\Pr_\pi(B_t^\epsilon,\, W_t^{\mathrm{probe}} = 1 \mid A_t)
&= \Pr_\pi(W_t^{\mathrm{probe}} = 1 \mid B_t^\epsilon,\, A_t)
   \Pr_\pi(B_t^\epsilon \mid A_t) \\
&= \sigma_t^{\mathrm{bad},\pi}\,\rho_t^\pi. \qquad\square
\end{aligned}
\]

If $\sigma_t^{\mathrm{bad},\pi} < 1$, the probe contracts the bad-and-survive mass strictly below the bad prior; if additionally $\sigma_t^{\mathrm{bad},\pi} < \sigma_t^{\mathrm{good},\pi}$, the probe also shrinks the bad-versus-good odds.

\subsection*{4. Proposition 2.8: Candidate-Level Evidence from EvidenceAcquire}

\textbf{Statement.} Let $W_t^{\mathrm{evid}} \in \{0, 1\}$ denote the program-level evidence-gate survival at step $t$ (set to $1$ when the gate is not invoked). On calibration rows that are simultaneously active, bad, and probe-surviving, let $\bar b_t^\pi(m)$ be the one-sided Clopper--Pearson UCB on

\[
\Pr_\pi(W_t^{\mathrm{evid}} = 1 \mid B_t^\epsilon,\, W_t^{\mathrm{probe}} = 1,\, A_t)
\]

at level $\delta_{\mathrm{evid}}/T$. Set $\Delta I_t(m) := -\log \bar b_t^\pi(m)$. Then, with calibration probability at least $1 - \delta_{\mathrm{evid}}$, simultaneously for $t \le T$,

\[
\Pr_\pi(W_t^{\mathrm{evid}} = 1 \mid B_t^\epsilon,\, W_t^{\mathrm{probe}} = 1,\, A_t)
\le \exp[-\Delta I_t(m)].
\]

\textbf{Proof.} Fix $t$ and restrict the calibration set to rows with $A_t = 1$, $B_t^\epsilon = 1$, and $W_t^{\mathrm{probe}} = 1$. Let $n_t^{\mathrm{evid},\pi}$ be the number of such rows and $k_t^{\mathrm{evid},\pi}$ the number with $W_t^{\mathrm{evid}} = 1$. Conditional on this restriction, $W_t^{\mathrm{evid}}$ is Bernoulli with parameter

\[
b_t^\pi(m) := \Pr_\pi(W_t^{\mathrm{evid}} = 1 \mid B_t^\epsilon,\, W_t^{\mathrm{probe}} = 1,\, A_t).
\]

By Assumption 2.3 and Section 0.5, the Clopper--Pearson UCB gives

\[
\Pr\{b_t^\pi(m) \le \bar b_t^\pi(m)\} \ge 1 - \tfrac{\delta_{\mathrm{evid}}}{T},
\]

and a union bound over $t = 1, \dots, T$ produces the good event

\[
G_{\mathrm{evid}} := \{\forall t \le T:\; b_t^\pi(m) \le \bar b_t^\pi(m)\}, \quad \Pr(G_{\mathrm{evid}}) \ge 1 - \delta_{\mathrm{evid}}.
\]

On $G_{\mathrm{evid}}$, $\exp[-\Delta I_t(m)] = \bar b_t^\pi(m)$, so

\[
\Pr_\pi(W_t^{\mathrm{evid}} = 1 \mid B_t^\epsilon,\, W_t^{\mathrm{probe}} = 1,\, A_t)
= b_t^\pi(m) \le \bar b_t^\pi(m) = \exp[-\Delta I_t(m)].
\]

If the conditioning subset is empty, $\bar b_t^\pi(m) = 1$ by Section 0.5 and the bound is trivial. $\qquad\square$

\subsubsection*{4.1 Suite-level interpretation}
\mbox{}\par\nobreak

The bound treats the all-pass event over the consensus-filtered suite as a single Bernoulli channel,

\[
\begin{aligned}
W_t^{\mathrm{evid}} = 1
\iff{}& \text{a bad-and-probe-surviving candidate} \\
& \text{passes the entire suite}.
\end{aligned}
\]

We calibrate the suite-level rate $b_t^\pi(m)$ directly; the $m$ generated tests need not be independent, and $m$ is not treated as a sample size.

\subsubsection*{4.2 Channel-level accumulation}
\mbox{}\par\nobreak

If multiple evidence channels $j = 1, \dots, J_t$ act on the same admission decision at step $t$, with survival events $E_{t,j}$ calibrated under sequential conditional survival

\[
\Pr_\pi(E_{t,j} = 1 \mid B_t^\epsilon,\, W_t^{\mathrm{probe}} = 1,\, A_t,\, E_{t,1} = 1, \dots, E_{t,j-1} = 1) \le b_{t,j},
\]

where each $b_{t,j}$ is a Clopper--Pearson UCB at level $\delta_{\mathrm{evid}} / (T \cdot J_{\max} \cdot |\Pi_\alpha|)$ on its conditional subset (with $J_{\max} = \max_t J_t$), then the chain rule gives

\[
\Pr_\pi(W_t^{\mathrm{evid}} = 1 \mid B_t^\epsilon,\, W_t^{\mathrm{probe}} = 1,\, A_t)
\le \prod_{j=1}^{J_t} b_{t,j}.
\]

In this setting we define the cumulative evidence by channel index,

\[
\mathcal I_t^{\mathrm{evid},\pi} := \sum_{j=1}^{J_t} \Delta I_{t,j},
\qquad
\Delta I_{t,j} := -\log b_{t,j},
\]

yielding

\[
\Pr_\pi(W_t^{\mathrm{evid}} = 1 \mid B_t^\epsilon,\, W_t^{\mathrm{probe}} = 1,\, A_t)
\le \exp[-\mathcal I_t^{\mathrm{evid},\pi}].
\]

The accumulation is over channels acting on the current step's admission decision, not over refinement steps $s \le t$. Across refinement steps, the candidates differ, no sequential conditional survival is calibrated, and the chain-rule decomposition does not apply.

For CP-Agent, the implementation uses a single TA channel per admission decision ($J_t \equiv 1$), and $\mathcal I_t^{\mathrm{evid},\pi} = \Delta I_t(m)$.

\subsection*{5. Validity of the Mechanism Factorization in Eq. (9)}

Eq. (9) defines

\[
\bar q_{t,\mathrm{mech}}^\pi = \bar\rho_t^{\mathrm{probe},\pi}\,\exp[-\mathcal I_t^{\mathrm{evid},\pi}],
\qquad
\bar q_{t,\mathrm{ctrl}}^\pi = \min\{\bar q_{t,\mathrm{raw}}^\pi,\, \bar q_{t,\mathrm{mech}}^\pi\}.
\tag{9}
\]

We show that on the good events of Propositions 2.7 and 2.8, $\bar q_{t,\mathrm{ctrl}}^\pi$ is a valid UCB on $q_t^\pi := \Pr_\pi(F_t = 1 \mid A_t)$.

\textbf{Proof.} Under the gate-semantics requirement of Section 0.2, on $A_t$,

\[
\{F_t = 1\} \subseteq \{B_t^\epsilon = 1,\, W_t^{\mathrm{probe}} = 1,\, W_t^{\mathrm{evid}} = 1\}.
\]

The chain rule of conditional probability gives

\[
\begin{aligned}
q_t^\pi
&\le \Pr_\pi(B_t^\epsilon = 1,\, W_t^{\mathrm{probe}} = 1,\, W_t^{\mathrm{evid}} = 1 \mid A_t) \\
&= \Pr_\pi(B_t^\epsilon = 1,\, W_t^{\mathrm{probe}} = 1 \mid A_t) \\
&\quad \times \Pr_\pi(W_t^{\mathrm{evid}} = 1 \mid B_t^\epsilon = 1,\, W_t^{\mathrm{probe}} = 1,\, A_t).
\end{aligned}
\]

On the good event of Proposition 2.7,

\[
\Pr_\pi(B_t^\epsilon = 1,\, W_t^{\mathrm{probe}} = 1 \mid A_t) \le \bar\rho_t^{\mathrm{probe},\pi}.
\]

On the good event of Proposition 2.8 (or Section 4.2 when multiple channels are present),

\[
\Pr_\pi(W_t^{\mathrm{evid}} = 1 \mid B_t^\epsilon = 1,\, W_t^{\mathrm{probe}} = 1,\, A_t)
\le \exp[-\mathcal I_t^{\mathrm{evid},\pi}],
\]

where $\mathcal I_t^{\mathrm{evid},\pi}$ accumulates over channels acting on the step-$t$ admission decision (Section 4.2). Combining,

\[
q_t^\pi \le \bar\rho_t^{\mathrm{probe},\pi} \exp[-\mathcal I_t^{\mathrm{evid},\pi}] = \bar q_{t,\mathrm{mech}}^\pi.
\]

Section 1 gives $q_t^\pi \le \bar q_{t,\mathrm{raw}}^\pi$ on its own good event. On the joint good event,

\[
q_t^\pi \le \min\{\bar q_{t,\mathrm{raw}}^\pi,\, \bar q_{t,\mathrm{mech}}^\pi\} = \bar q_{t,\mathrm{ctrl}}^\pi. \qquad\square
\]

If a mechanism factor is not separately calibrated, we set it to $1$; this can only relax $\bar q_{t,\mathrm{mech}}^\pi$, so Eq. (9) falls back to $\bar q_{t,\mathrm{raw}}^\pi$ without losing validity.

\subsection*{6. Proposition 2.9: ContextAcquire as a Calibrated Hazard Envelope}

\textbf{Statement.} Under the split $D_{\mathrm{hist}}^{\mathrm{mem}} \to D_{\mathrm{cal}} \to D_{\mathrm{test}}$ and a frozen memory snapshot $M^\star$, ContextAcquire$(k)$ assigns $\underline h_{t+} \leftarrow \underline h_t^\pi(M^\star)$, where

\[
\underline h_t^\pi(M^\star) = \mathrm{LCB}_{\mathrm{Binom}}\!\left(s_t^{M^\star,\pi},\, n_t^{M^\star,\pi},\, \eta_h\right)
\]

is a Clopper--Pearson LCB estimated on calibration traces using $M^\star$. The increment $\Delta h_t(k, Z_t) := \underline h_t^\pi(M^\star) - \underline h_t^{\pi,\varnothing}$ is the snapshot-on/off LCB gap, used only as a deployment diagnostic; the certificate uses $\underline h_t^\pi(M^\star)$ directly.

\textbf{Proof.} Fix $\pi$, $t$, and $M^\star$. The split ensures that $M^\star$ is constructed from $D_{\mathrm{hist}}^{\mathrm{mem}}$ and used read-only on $D_{\mathrm{cal}}$ and $D_{\mathrm{test}}$, so the evaluation rows do not leak into memory construction. If $M^\star$ or any other manifest entry is selected from a finite set on $D_{\mathrm{cal}}$, all candidates are absorbed into the union bound of Theorem 2.6; the present statement is for a fixed $M^\star$.

Let $n_t^{M^\star,\pi}$ be the number of active calibration rows using $M^\star$ and $s_t^{M^\star,\pi}$ the count of clean successes. The deployment hazard is

\[
h_t^\pi(M^\star) := \Pr_\pi(S_t = 1 \mid A_t,\, M^\star).
\]

By Assumption 2.3 and Section 0.5, on $\{A_t = 1\}$ the indicator $S_t$ is i.i.d. Bernoulli with parameter $h_t^\pi(M^\star)$, and the Clopper--Pearson LCB satisfies

\[
\Pr\{h_t^\pi(M^\star) \ge \underline h_t^\pi(M^\star)\} \ge 1 - \eta_h.
\]

Union-bounding over steps and any manifest-level memory/controller choices yields the simultaneous LCB. Assigning $\underline h_{t+} \leftarrow \underline h_t^\pi(M^\star)$ gives $h_t^+ \ge \underline h_t^\pi(M^\star)$. The certificate (Corollary 2.10, Theorem 2.6) uses this LCB directly, not the gap $\Delta h_t(k, Z_t)$. $\qquad\square$

\subsection*{7. Corollary 2.10: Stopped Clean-Before-False Certificate}

\textbf{Statement.} Fix a frozen controller $\pi$. Assume the stopped-process convention of Section 0.1 and the initial-activity condition $\Pr_\pi(A_1) = 1$. On a calibration good event giving simultaneous active-step bounds $q_t^\pi \le \bar q_t^\pi$ and $h_t^\pi \ge \underline h_t^\pi$ for all $t \le T$, a fresh evaluation trajectory satisfies

\[
\Pr_\pi(\tau_S \le T,\, \tau_S < \tau_F)
\ge C_T(\pi)
:= \left[
\prod_{t=1}^T (1 - \bar q_t^\pi)
- \prod_{t=1}^T (1 - \underline h_t^\pi)
\right]_+ .
\tag{13}
\]

\textbf{Proof.} The argument uses only the chain rule on stopped events; no across-step independence is required.

\emph{Step 1 (no-false survival).} The event $\{F_t = 1\}$ implies the trajectory is active at $t$, hence $\{F_t = 1\} \subseteq A_t \subseteq \{\tau_F > t-1\}$. For events $E \subseteq B \subseteq C$ with $\Pr(B) > 0$,

\[
\Pr(E \mid C) = \frac{\Pr(E)}{\Pr(C)} \le \frac{\Pr(E)}{\Pr(B)} = \Pr(E \mid B),
\]

so

\[
\Pr_\pi(F_t = 1 \mid \tau_F > t-1)
\le \Pr_\pi(F_t = 1 \mid A_t) = q_t^\pi \le \bar q_t^\pi.
\]

If $\Pr(A_t) = 0$, Section 0.4 sets the conditional to $0$ and the inequality remains valid. Iterating the chain rule on $\{\tau_F > t\} = \{\tau_F > t-1\} \cap \{F_t = 0\}$,

\[
\Pr_\pi(\tau_F > T)
= \prod_{t=1}^T \Pr_\pi(F_t = 0 \mid \tau_F > t-1)
\ge \prod_{t=1}^T (1 - \bar q_t^\pi).
\tag{14}
\]

\emph{Step 2 (still-active upper bound).} Since $A_{t+1} = A_t \cap \{S_t = 0,\, F_t = 0\} \subseteq A_t \cap \{S_t = 0\}$,

\[
\Pr_\pi(A_{t+1} \mid A_t) \le \Pr_\pi(S_t = 0 \mid A_t) = 1 - h_t^\pi \le 1 - \underline h_t^\pi.
\]

Iterating, and using $\Pr_\pi(A_1) = 1$ from the Section 0.1 initial-activity assumption together with the inclusion $A_{T+1} \subseteq A_T \subseteq \cdots \subseteq A_1$,

\[
\Pr_\pi(A_{T+1}) = \prod_{t=1}^T \Pr_\pi(A_{t+1} \mid A_t)
\le \prod_{t=1}^T (1 - \underline h_t^\pi).
\tag{15}
\]

\emph{Step 3 (decomposition).} We claim the disjoint identity

\[
\{\tau_F > T\} = \{\tau_S \le T,\; \tau_S < \tau_F\} \;\dot\cup\; A_{T+1}.
\]

To verify it, fix a trajectory in $\{\tau_F > T\}$ and split on whether $\tau_S \le T$ or $\tau_S > T$:

(i) If $\tau_S \le T$ then $\tau_S \le T < \tau_F$, so the trajectory lies in $\{\tau_S \le T,\, \tau_S < \tau_F\}$.

(ii) If $\tau_S > T$ then $\tau_S > T$ and $\tau_F > T$, hence the trajectory lies in $A_{T+1} = \{\tau_S > T,\, \tau_F > T\}$.

For the reverse containment, note that the two right-hand events are disjoint: in $A_{T+1}$ we have $\tau_S > T$, contradicting $\tau_S \le T$. To see that every trajectory in $\{\tau_S \le T,\, \tau_S < \tau_F\}$ also lies in $\{\tau_F > T\}$, we invoke the stopped-process convention of Section 0.1: on $\{\tau_S < \tau_F\}$, $\tau_S$ is the first stop, so $F_t = 0$ at every active step $t < \tau_S$ (by minimality of $\tau_S$) and $F_t = 0$ at every step $t > \tau_S$ (by the stopped-process zeroing); the within-step exclusivity of $F_{\tau_S}$ and $S_{\tau_S}$ via $B_t^\epsilon$ (Section 0.1) gives $F_{\tau_S} = 0$. Hence $\tau_F = \inf\{t : F_t = 1\} = +\infty > T$, and the trajectory is in $\{\tau_F > T\}$. Without this convention the identity fails (e.g., $\tau_S = 2,\,\tau_F = 3,\,T = 5$ would satisfy $\tau_S \le T,\,\tau_S < \tau_F$ but $\tau_F \le T$).

Therefore

\[
\Pr_\pi(\tau_S \le T,\, \tau_S < \tau_F) = \Pr_\pi(\tau_F > T) - \Pr_\pi(A_{T+1}).
\]

Combining (14) and (15) and taking the positive part,

\[
\Pr_\pi(\tau_S \le T,\, \tau_S < \tau_F)
\ge \left[\prod_{t=1}^T (1 - \bar q_t^\pi) - \prod_{t=1}^T (1 - \underline h_t^\pi)\right]_+ . \qquad\square
\]

\emph{Specialization to $T = 1$.} $C_1(\pi) = [\underline h_1^\pi - \bar q_1^\pi]_+$, recovering the Pass@1 form.

\subsection*{8. Theorem 2.6: Simultaneously Valid Finite-Manifest Calibration}

\textbf{Statement.} Let $\Pi_\alpha$ be a finite, pre-declared, non-empty class of frozen controllers with $1 \le |\Pi_\alpha| < \infty$, declared before $D_{\mathrm{cal}}$ is opened. Each $\pi \in \Pi_\alpha$ specifies the stop/refine rule, risk-probe invocation policy, evidence intensity $m$, context intensity $k$, prompts, routing and deduplication rules, decoding parameters, retrievable memory snapshot $M^\star$, write mode $\omega_{\mathrm{write}}$, and (when a threshold gate is used) a finite non-empty grid $\Theta_t$ with $1 \le |\Theta_t| < \infty$ per step. With probability at least $1 - \delta$ over calibration, the UCB/LCB quantities entering

\[
C_T^{\mathrm{ctrl}}(\pi)
:= \left[\prod_{t=1}^T (1 - \bar q_{t,\mathrm{ctrl}}^\pi) - \prod_{t=1}^T (1 - \underline h_t^\pi(M))\right]_+
\tag{11}
\]

are simultaneously valid for every $\pi \in \Pi_\alpha$, so the data-dependent selection

\[
\pi^\star \in \arg\max_{\pi \in \Pi_\alpha} C_T^{\mathrm{ctrl}}(\pi)
\tag{12}
\]

inherits a post-selection-valid lower bound $C_T^{\mathrm{ctrl}}(\pi^\star)$ on $\Pr_{\pi^\star}(\tau_S \le T,\, \tau_S < \tau_F)$.

\subsubsection*{8.1 Confidence budget}
\mbox{}\par\nobreak

Choose non-negative $\delta_{\mathrm{raw}}, \delta_{\mathrm{gate}}, \delta_{\mathrm{probe}}, \delta_{\mathrm{evid}}$ with

\[
\delta_{\mathrm{raw}} + \delta_{\mathrm{gate}} + \delta_{\mathrm{probe}} + \delta_{\mathrm{evid}} \le \delta. \tag{G}
\]

For each $\pi \in \Pi_\alpha$, allocate per-row Clopper--Pearson confidences as

\[
\frac{\delta_{\mathrm{raw}}}{2T |\Pi_\alpha|}, \quad
\frac{\delta_{\mathrm{gate}}}{T |\Theta_t| |\Pi_\alpha|}, \quad
\frac{\delta_{\mathrm{probe}}}{T |\Pi_\alpha|}, \quad
\frac{\delta_{\mathrm{evid}}}{T |\Pi_\alpha| \cdot J_{\max}}
\]

for the raw, gate, probe, and evidence bounds, respectively, with $J_{\max} = \max_t J_t$ (which equals $1$ in the single-channel implementation). Substituting $\delta = 0.10$, $\delta_\bullet = 0.025$, $|\Pi_\alpha| = 5$, $T = 5$, $|\Theta_t| = 1$, $J_{\max} = 1$ gives the per-row raw confidence $0.025 / (2 \cdot 5 \cdot 5) = 5 \times 10^{-4}$ used in the main text.

\subsubsection*{8.2 Proof}
\mbox{}\par\nobreak

For each $\pi \in \Pi_\alpha$ define the good events

\[
\begin{aligned}
G_{\mathrm{raw}}(\pi) &: \forall t,\; q_t^\pi \le \bar q_{t,\mathrm{raw}}^\pi \text{ and } h_t^\pi(M) \ge \underline h_t^\pi(M), \\
G_{\mathrm{gate}}(\pi) &: \forall t,\,\forall \theta \in \Theta_t,\; p_{t,\theta}^\pi \le \bar R_t^\pi(\theta), \\
G_{\mathrm{probe}}(\pi) &: \forall t,\; \Pr_\pi(B_t^\epsilon,\, W_t^{\mathrm{probe}} = 1 \mid A_t) \le \bar\rho_t^{\mathrm{probe},\pi}, \\
G_{\mathrm{evid}}(\pi) &: \forall t,\,\forall j \le J_t,\;
\Pr_\pi(E_{t,j} = 1 \mid \cdots) \le b_{t,j}^\pi.
\end{aligned}
\]

The per-row allocations of Section 8.1 give

\[
\begin{aligned}
\Pr(G_{\mathrm{raw}}(\pi)^c) &\le \tfrac{\delta_{\mathrm{raw}}}{|\Pi_\alpha|}, &
\Pr(G_{\mathrm{gate}}(\pi)^c) &\le \tfrac{\delta_{\mathrm{gate}}}{|\Pi_\alpha|}, \\
\Pr(G_{\mathrm{probe}}(\pi)^c) &\le \tfrac{\delta_{\mathrm{probe}}}{|\Pi_\alpha|}, &
\Pr(G_{\mathrm{evid}}(\pi)^c) &\le \tfrac{\delta_{\mathrm{evid}}}{|\Pi_\alpha|}.
\end{aligned}
\]

Let

\[
G := \bigcap_{\pi \in \Pi_\alpha}\!\bigl(G_{\mathrm{raw}}(\pi) \cap G_{\mathrm{gate}}(\pi) \cap G_{\mathrm{probe}}(\pi) \cap G_{\mathrm{evid}}(\pi)\bigr).
\]

A union bound over $\Pi_\alpha$, steps, threshold grids, mechanism factors, and channels gives

\[
\Pr(G) \ge 1 - (\delta_{\mathrm{raw}} + \delta_{\mathrm{gate}} + \delta_{\mathrm{probe}} + \delta_{\mathrm{evid}}) \ge 1 - \delta.
\]

On $G$, for any $\pi \in \Pi_\alpha$ and any $t \le T$:

\begin{enumerate}
\item $q_t^\pi \le \bar q_{t,\mathrm{raw}}^\pi$ from $G_{\mathrm{raw}}(\pi)$.
\item $q_t^\pi \le \bar\rho_t^{\mathrm{probe},\pi} \exp[-\mathcal I_t^{\mathrm{evid},\pi}] = \bar q_{t,\mathrm{mech}}^\pi$ from Section 5 using $G_{\mathrm{probe}}(\pi) \cap G_{\mathrm{evid}}(\pi)$.
\item Therefore $q_t^\pi \le \bar q_{t,\mathrm{ctrl}}^\pi := \min\{\bar q_{t,\mathrm{raw}}^\pi,\, \bar q_{t,\mathrm{mech}}^\pi\}$.
\item $h_t^\pi(M) \ge \underline h_t^\pi(M)$ from $G_{\mathrm{raw}}(\pi)$.
\end{enumerate}

Corollary 2.10 then gives $\Pr_\pi(\tau_S \le T,\, \tau_S < \tau_F) \ge C_T^{\mathrm{ctrl}}(\pi)$ for every $\pi \in \Pi_\alpha$. Since $G$ holds simultaneously for all $\pi \in \Pi_\alpha$, any data-dependent $\pi^\star$ in this finite class also satisfies

\[
\Pr_{\pi^\star}(\tau_S \le T,\, \tau_S < \tau_F) \ge C_T^{\mathrm{ctrl}}(\pi^\star).
\]

If $\pi^\star$'s manifest includes a score-threshold admission rule, $G_{\mathrm{gate}}(\pi^\star)$ provides the Theorem 2.5 guarantee at the corresponding threshold. $\qquad\square$

\subsection*{9. Monotonicity of the Certificate}

In the unclipped regime, define

\[
C_T = \prod_{s=1}^T (1 - \bar q_s) - \prod_{s=1}^T (1 - \underline h_s),
\qquad
\bar q_t = \bar\rho_t^{\mathrm{probe}} \exp(-\mathcal I_t^{\mathrm{evid}}).
\]

We have

\[
\frac{\partial C_T}{\partial \bar q_t} = -\prod_{s \ne t} (1 - \bar q_s) \le 0,
\]

so any reduction in $\bar q_t$ does not decrease $C_T$. Since

\[
\frac{\partial \bar q_t}{\partial \bar\rho_t^{\mathrm{probe}}} = \exp(-\mathcal I_t^{\mathrm{evid}}) \ge 0,
\qquad
\frac{\partial \bar q_t}{\partial \mathcal I_t^{\mathrm{evid}}} = -\bar q_t \le 0,
\]

it follows that

\[
\frac{\partial C_T}{\partial \bar\rho_t^{\mathrm{probe}}} \le 0,
\qquad
\frac{\partial C_T}{\partial \mathcal I_t^{\mathrm{evid}}} \ge 0,
\qquad
\frac{\partial C_T}{\partial \underline h_t} = \prod_{s \ne t} (1 - \underline h_s) \ge 0.
\]

RiskProbe contracts $\bar\rho_t^{\mathrm{probe}}$, EvidenceAcquire increases $\mathcal I_t^{\mathrm{evid}}$, and ContextAcquire raises $\underline h_t$, each weakly increasing the certificate. The same conclusions hold in the clipped regime by subgradient monotonicity of $[\cdot]_+$.

\subsection*{10. Summary of Structural Conditions}

\begin{enumerate}
\item Across-trajectory independence at fixed $t$ is required; within-trajectory step independence is not (Section 0.3).
\item The suite-level Bernoulli treatment does not require independence across the $m$ generated tests (Section 4.1); test-level independence is needed only for sequential channel accumulation (Section 4.2).
\item The Hypothesis Prober contributes only to $\bar\rho_t^{\mathrm{probe}}$; Solution Validator and Test Augmentation contribute only to $\mathcal I_t^{\mathrm{evid}}$ (Sections 3, 4, 5).
\item Finite grids and finite manifests give post-selection validity (Theorems 2.5, 2.6).
\item Empty cells use the conservative conventions $\mathrm{UCB}(0, 0, \eta) = 1$ and $\mathrm{LCB}(0, 0, \eta) = 0$ (Section 0.5).
\item Memory must be frozen, or all candidate snapshots, routes, and retrieval budgets must be absorbed into $\Pi_\alpha$ (Section 6).
\item Corollary 2.10 lower-bounds $\Pr(\tau_S \le T,\, \tau_S < \tau_F)$, which is strictly stricter than aggregate Refine@$T$ or Pass@1.
\item $F_t$ and $S_t$ are mutually exclusive by construction; $\tau_S = \tau_F$ does not occur (Section 0.1).
\item The mechanism containment requires $\Gamma_t = 1 \Rightarrow W_t^{\mathrm{probe}} = W_t^{\mathrm{evid}} = 1$ on $A_t$; without it, only the raw UCB applies (Section 0.2).
\item The score map $\phi_\pi$, threshold grid $\Theta_t$, aggregation rule $\gamma_{\mathrm{evid}}$, memory snapshot $M^\star$, and any calibrated scalar in $r_t$ must be pre-declared or absorbed into $\Pi_\alpha$ before label counting (Section 0.6).
\item Evidence accumulates over channels acting on the same admission decision, not over refinement steps. The CP-Agent implementation uses a single channel per step, so $\mathcal I_t^{\mathrm{evid},\pi} = \Delta I_t(m)$ (Section 4.2).
\item The confidence budget satisfies $\delta_{\mathrm{raw}} + \delta_{\mathrm{gate}} + \delta_{\mathrm{probe}} + \delta_{\mathrm{evid}} \le \delta$ (Section 8.1).
\end{enumerate}

\subsection*{11. Compact Proof Skeleton}

\begin{enumerate}
\item \emph{Lemma (Clopper--Pearson with random active denominator).} Conditioning on the active-row count, the CP UCB/LCB is valid; a union bound covers steps, grids, and manifests.
\item \emph{Theorem 2.5.} For each $(t, \theta)$, calibrate $p_{t,\theta} = \Pr(F_t = 1 \mid A_t,\, r_t \le \theta)$; the finite-grid simultaneous UCB yields a post-selection-valid threshold.
\item \emph{Proposition 2.7.} On active rows, $\mathbf 1\{B_t^\epsilon,\, W_t^{\mathrm{probe}} = 1\}$ is Bernoulli; the chain rule gives the factorization $\sigma_t^{\mathrm{bad}} \rho_t$.
\item \emph{Proposition 2.8.} On bad-and-probe-surviving rows, calibrate the suite-level Bernoulli survival; $\Delta I = -\log \bar b$ requires no independence across the $m$ tests; channel-level accumulation requires sequential conditional calibration.
\item \emph{Eq. (9).} The chain rule on $\{F_t = 1\} \subseteq \{B_t^\epsilon,\, W_t^{\mathrm{probe}},\, W_t^{\mathrm{evid}}\}$ on $A_t$ gives $q_t^\pi \le \bar\rho \exp(-\mathcal I)$; the minimum with the raw UCB is again a UCB.
\item \emph{Proposition 2.9.} With frozen $M^\star$, the active-success hazard admits the CP LCB; the snapshot-on/off gap is a diagnostic, not a certificate component.
\item \emph{Corollary 2.10.} No-false survival lower bound, still-active upper bound, and disjoint decomposition $\{\tau_F > T\} = \{\tau_S \le T,\, \tau_S < \tau_F\} \dot\cup A_{T+1}$.
\item \emph{Theorem 2.6.} All UCB/LCB quantities are simultaneously valid over the finite manifest; the argmax inherits the bound.
\end{enumerate}

\endgroup

%% file: appendix_prompts.tex
\section{Prompts}
\label{app:prompt}

\subsection{System Prompt}
\begin{promptboxapp}

You are a professional competitive programming contestant who specializes in solving algorithm and data structure problems.

Please carefully read the problem statement in the input, and first analyze:

**Problem Category**  

- Determine which type(s) of algorithmic problem it mainly belongs to.  

- You may choose or combine from the following categories:  

  Use the format `"Primary Tag: [Secondary Tag1, Secondary Tag2, ...]"`  
  - Basic Algorithms: [Simulation, Enumeration, Recursion, Greedy, Divide and Conquer, Binary Search, Sorting, Doubling, Construction, Prefix Sum \& Difference Array]  
  
  - Dynamic Programming: [Memoized Search, Knapsack DP, Interval DP, DP on DAG, Tree DP, Bitmask DP, Digit DP, Counting DP, Probabilistic DP, DP Optimization]  
  
  - Strings: [String Matching (KMP/Boyer–Moore), String Hashing, Trie, Aho–Corasick Automaton, Suffix Array (SA), Suffix Automaton (SAM), Manacher, Palindromic Tree]  
  
  - Mathematics: [Number Theory, Combinatorics, Linear Algebra, Game Theory, Probability \& Expectation, Polynomials, Fast Exponentiation, Primes, Divisors, Congruence]  
  
  - Data Structures: [Stack, Queue, Linked List, Hash Table, Disjoint Set Union (DSU), Fenwick Tree (BIT), Segment Tree, Sqrt Decomposition, Balanced BST, Tree of Trees, Persistent Data Structures]  
  
  - Graph Theory: [Graph Traversal (DFS/BFS), Shortest Path, Minimum Spanning Tree, Topological Sort, Bipartite Graph, Network Flow, Strongly Connected Components (SCC), Biconnected Components, Lowest Common Ancestor (LCA), Tree Problems]  
  
  - Computational Geometry: [Vector Operations, Convex Hull, Half-Plane Intersection, Rotating Calipers, Sweep Line]  
  
  - Miscellaneous: [Two Pointers, Discretization, Randomization, Search (Pruning/Heuristics/A*)]

Next, please write concise, efficient, and correct C++ code.  

Finally, please use tools to validate the code’s correctness and debug issues.

\#\# Experience System

\#\#\# General Programming Experience (Auto-injected)
Historical general programming experiences (e.g., syntax errors, header includes, type conversions) will be shown below. Please pay attention to these common issues while coding.

{\%- if long\_term\_experiences \%}
{{long\_term\_experiences}}
{\%- endif \%}

Here are some examples of using tools: 

Task: "Solve a graph shortest path problem: N cows stand in a line, each cow has a breed attribute. You need to send information from cow 1 to cow N. Whether two breeds can communicate is determined by a matrix S, and the cost is the distance between positions." 

Thought: This is a shortest path problem with constraints, belonging to graph theory. Let me first retrieve historical experience on graph theory - shortest path, to see if there are similar solutions. 

\{\{code\_block\_opening\_tag\}\} 

experience\_data = retrieve\_algorithm\_experience(primary\_tag="Graph Theory", secondary\_tag="Shortest Path") 

print(experience\_data) \{\{code\_block\_closing\_tag\}\} 

Observation: "[Graph Theory - Shortest Path Experience] 1. When using Dijkstra, ensure all edge weights are non-negative..." Thought: This experience is useful; I need to ensure edge weights are non-negative. Let me use Hypothesis Prober to verify whether the edge weights are non-negative and whether Dijkstra is applicable. 

\{\{code\_block\_opening\_tag\}\} 

cpp\_code = r''' 

\#include <bits/stdc++.h> 

using namespace std;

int main() { 

int N = 10; 

cout << "Edges in directed complete graph: " << 1LL * N * (N - 1) << "\\n"; 

for (int i = 1; i <= N; i++) 

for (int j = 1; j <= N; j++) 

if (i != j \&\& abs(i - j) < 0) 

cout << "Negative weight?\\n"; 

cout << "All |i-j| >= 0 -> Dijkstra applicable\\n"; 

return 0; 

} 

'''; 

result = hypo\_prober(cpp\_code); 

print(result) 

\{\{code\_block\_closing\_tag\}\}

Observation: The weights are non-negative, but building the complete graph is too large and needs layering/graph-construction optimization. 

Thought: After understanding the historical experience, I will write the final code example. To avoid blow-up, do not build an O($N^2$) graph; avoid Dijkstra and instead use 0-1 BFS (deque) on a layered graph to compute shortest paths. 

\{\{code\_block\_opening\_tag\}\} 

cpp\_code = r''' 

\#include <bits/stdc++.h> 

using namespace std;

... 

int ans = dist[id\_pos(N)]; 

if (ans >= INF/2) 

cout << -1 << "\\n"; 

else cout << ans << "\\n"; 

return 0; 

'''

result = cpp\_validation(cpp\_code) 

print(result) 

\{\{code\_block\_closing\_tag\}\} 
\end{promptboxapp}

\begin{promptboxapp}

The above examples use hypothetical tools that may not exist in real tasks. Aside from executing computations inside the Python code snippets you create, you can only access the tools listed below, whose behavior is similar to regular Python functions: 

\{\{code\_block\_opening\_tag\}\} 

\{\%- for tool in tools.values() \%\} 

\{\{ tool.to\_code\_prompt() \}\} 

\{\% endfor \%\} 

\{\{code\_block\_closing\_tag\}\} 

Here are the rules you should follow:

1. C++ code must be passed to tool functions (e.g., cpp\_validation, hypo\_validator, final\_answer) as a string. Correct format example: 

Thought: [Your thought] 

Next, I will write C++ code to implement this idea. 

\{\{code\_block\_opening\_tag\}\} 

code = ''' 

\#include <iostream> 

int main() 

...

''' 

result = cpp\_validation(code) 

print(result) 

\{\{code\_block\_closing\_tag\}\} 

2. Only use variables that you have defined! 

3. C++ code must not be executed directly by the Python interpreter! 

4. Do not name any new variable as any tool name: e.g., do not name a variable "final\_answer".

7. Do not create any dummy variables in the code, because these variables appearing in logs may mislead you. 

8. Do not give up! You have no time or step limit for writing code. You are the one solving the task, not providing guidance on how to solve it.

Please start solving the problem.
\end{promptboxapp}
\subsection{Tool Descriptions}
\subsubsection{Test Case Generator}
\mbox{}\par\nobreak
\begin{promptboxapp}
You are an elite competitive-programming assistant.

Produce only valid C++17 code.

Use standard I/O (cin/cout).

Include every helper necessary for a brute-force solution.

Return only the C++ source code, wrapped in a single fenced code block: ```cpp ... ```

You are an expert competitive programming assistant.

Your task is to write a correct and brute-force C++ solution for the following problem. The program must:

- Read input from standard input and write output to standard output, strictly following the problem's input/output format.

- Implement a straightforward, brute-force (but correct) algorithm, prioritizing clarity and correctness over efficiency.

- Avoid any optimizations or advanced algorithms; use only simple, exhaustive approaches.

- Ensure the code is self-contained and compiles with GNU G++23 14.2 (64bit, msys2).

---

{Question}

---

Language and Compiler

- Language: cpp

- Compiler: GNU G++23 14.2 (64bit, msys2)

Output:

Provide only the complete C++ source code, enclosed in a single code block using the syntax ```cpp ... ```.
\end{promptboxapp}
\subsubsection{Hypothesis Prober}
\mbox{}\par\nobreak
\begin{promptboxapp}
This tool compiles and executes C++ code in a local sandbox environment.

It takes C++ code as input, compiles it, and returns the execution output. The tool can be used for quick numerical computation and verification: write small C++ programs to compute results (e.g., prefix sums, combinations, matrix operations, simulations, randomized experiments), and treat the printed output as a reliable computed reference to support your reasoning.

**When input is needed, you should hard-code the input directly into the C++ program** (e.g., embed test data in the code). This is not a tool for validating the code you intend to submit.

If any compilation or runtime errors occur, it will return the corresponding error messages, along with historical error reminders.
\end{promptboxapp}

\subsubsection{Solution Validator}
\mbox{}\par\nobreak
\begin{promptboxapp}
This tool validates your C++ code by compiling and running it against local test cases.

It automatically injects test inputs and compares the actual output with the expected output.

For failed tests, it returns the input data, the expected output, and the actual output.

The output is limited to 500 words to prevent excessive verbosity.
This tool is used to verify whether the code you intend to submit is correct, and it does not require you to provide input.

Before submitting the final code, you must use this tool at least once.
\end{promptboxapp}

\subsubsection{Experience Retriever}
\mbox{}\par\nobreak
\begin{promptboxapp}
Retrieve historical fix experiences for a specific algorithm type.

When you need to solve an algorithm-related problem, you can use this tool to obtain past experiences for that algorithm type. These experiences come from previously solved similar problems and may help with the current task.

Algorithm taxonomy:

- Basic Algorithms: Simulation, Enumeration, Recursion, Greedy, Divide and Conquer, Binary Search, Sorting, Doubling, Construction, Prefix Sum \& Difference Array

- Dynamic Programming: Memoized Search, Knapsack DP, Interval DP, DP on DAG, Tree DP, Bitmask DP, Digit DP, Counting DP, Probabilistic DP, DP Optimization

- Strings: String Matching, String Hashing, Trie, Aho–Corasick Automaton, Suffix Array, Manacher, Palindromic Tree

- Mathematics: Number Theory, Combinatorics, Linear Algebra, Game Theory, Probability \& Expectation, Fast Exponentiation, Primes, Divisors, Congruence

- Data Structures: Stack, Queue, Linked List, Hash Table, Disjoint Set Union (DSU), Fenwick Tree (BIT), Segment Tree, Sqrt Decomposition, Balanced BST

- Graph Theory: Graph Traversal (DFS/BFS), Shortest Path, Minimum Spanning Tree, Topological Sort, Bipartite Graph, Network Flow, Strongly Connected Components (SCC), Lowest Common Ancestor (LCA)

- Computational Geometry: Vector Operations, Convex Hull, Half-Plane Intersection, Rotating Calipers, Sweep Line

- Miscellaneous: Two Pointers, Discretization, Randomization, Search (Pruning/Heuristics/A*)

\end{promptboxapp}
\subsection{Experience Generation}
\begin{promptboxapp}
\#\# Your Task
As an expert in extracting code-fix experiences, you need to analyze the debugging process that turns a failed solution into a successful one, and distill reusable debugging experience and repair strategies.

\#\# Core Responsibilities
1. Accuracy: Ensure the extracted experiences are accurate and verified.
2. Reusability: The summarized experience should be applicable to similar problems.
3. Practicality: Provide concrete, actionable fix methods.
4. Conciseness: Express the key points in the most concise language.

\#\# Repair Process to Analyze

\#\#\# First Verification (Failed)
Error Type: \{error\_summary\}

Error Message:
\{failure\_result\}

Code Snippet at Failure:
\begin{verbatim}
{failure_code}
\end{verbatim}

\#\#\# Intermediate Improvement Steps
\{steps\_text\}

\#\#\# Second Verification (Succeeded)
Code Snippet After Success:
\begin{verbatim}
{success_code}
\end{verbatim}

\#\# Analysis Requirements

\#\#\# 1. Identify the Root Cause
- What is the fundamental cause of this error?
- Is it an algorithmic logic issue, a boundary-condition issue, or a syntax/library-usage issue?

\#\#\# 2. Extract the Fix Strategy
- What specific method did the agent use to fix the issue?
- What are the key steps of this fix?
- Are there any debugging tricks or thought processes worth recording?

\#\#\# 3. Summarize Reusable Experience
- If a similar issue occurs in the future, what approach should be taken?
- What was the most critical insight in this fix process?
- What general cautions or best practices apply?

\#\# Output Format Requirements

Return a JSON object with the following fields:
\begin{verbatim}
{
  "error\_context": "When/where the error occurred (what was being implemented/handled)",
  "error\_cause": "The specific cause of the error",
  "fix\_method": "The concrete modification made",
  "fix\_result": "The effect after the fix",
  "key\_insight": "A reusable key takeaway (one-sentence lesson)"
}
\end{verbatim}

\#\#\# Output Specifications

1. error\_context (required):
   - Describe the scenario and timing of the error.
   - Examples: "While implementing divide-and-conquer optimized DP", "When handling large-scale array inputs", "When enumerating subsets using bit operations".
   - Length: 10--30 words.

2. error\_cause (required):
   - Clearly state the cause of the error.
   - Examples: "Incorrect subproblem evaluation order in recursion, violating state dependencies."
   - Examples: "Using <= in the loop bound caused out-of-bounds array access."
   - Length: 15--50 words.

3. fix\_method (required):
   - Describe the concrete fix.
   - Examples: "Move the recursive call solve(l, mid-1) before computing dp[mid]."
   - Examples: "Change the loop condition from i<=n to i<n."
   - Length: 15--50 words.

4. fix\_result (required):
   - Describe the effect after applying the fix.
   - Examples: "State dependencies are satisfied; outputs match expected results."
   - Examples: "All array accesses are within bounds; no segmentation fault."
   - Length: 10--30 words.

5. key\_insight (required):
   - The most important reusable lesson.
   - It should help others avoid the same mistake at a glance.
   - Examples: "In divide-and-conquer DP optimization, subproblems that the current state depends on must be computed before evaluating dp[mid]."
   - Examples: "Array indices range from 0 to n-1; use a strict < n loop bound when iterating."
   - Length: 20--60 words.

\#\# Example Outputs

\begin{verbatim}
{
  "error\_context": "While implementing divide-and-conquer optimized DP",
  "error\_cause": "Left-half subproblems dp[l..mid-1] were not computed before dp[mid], breaking the required
  dependency order",
  "fix\_method": "Move the recursive call solve(l, mid-1) to execute before computing dp[mid]",
  "fix\_result": "Dependencies are satisfied and all test cases pass",
  "key\_insight": "For divide-and-conquer DP optimization, compute all dependent subproblems before evaluating 
  the current state"
}
\end{verbatim}

\begin{verbatim}
{
  "error\_context": "While iterating over array elements",
  "error\_cause": "The loop condition used i<=n instead of i<n, causing out-of-bounds access beyond the last 
  element",
  "fix\_method": "Change for(int i=0; i<=n; i++) to for(int i=0; i<n; i++)",
  "fix\_result": "Array access stays in range and the program runs normally",
  "key\_insight": "Array indices are 0..n-1; always use a strict <n loop bound when traversing arrays"
}
\end{verbatim}

\end{promptboxapp}
\subsection{Experience Deduplication}
\begin{promptboxapp}
 Please determine whether the new experience is a duplicate of, or highly similar to, any existing experience.

\#\# New Experience
\{new\_exp.get('fix\_summary', '')\}

\#\# Existing Experience List
\{existing\_summaries\}

\#\# Criteria
- If the problem described and the solution method are essentially the same as an existing experience (even if phrased differently), it is considered a duplicate.
- If the new experience is a special case or a generalization of an existing one, it is also considered a duplicate.
- If the new experience involves a different problem or a different solution method, it is considered non-duplicate.

\#\# Output Format
Return a JSON object:
\begin{verbatim}
{
  "is\_duplicate": true or false,
  "duplicate\_index": the index of the duplicated experience (1-based; null if not duplicate),
  "reason": "Justification for the decision"
}
\end{verbatim}
   
\end{promptboxapp}

\subsection{Experience Routing}
\begin{promptboxapp}
Please analyze the following repair experience and determine which category it belongs to.

\#\# Repair Experience
- Error Description: \{original\_error\}
- Fix Method: \{fix\_summary\}
\{f'- Code Context: \{code\_context[:500]\}' if code\_context else ''\}

\#\# Classification Rules

\#\#\# 1. General (General-Purpose Errors)
Applicable to common programming errors that are not tied to a specific algorithm, such as:
- Syntax errors (missing semicolons, mismatched braces/parentheses, etc.)
- Spelling mistakes (typos)
- Header/library import issues
- Undeclared/uninitialized variables
- Type conversion/casting errors
- Input/output format issues
- Memory management issues (not algorithm-related)

\#\#\# 2. Algorithm (Algorithmic Errors)
Applicable to errors tied to a particular algorithm, where you must choose the corresponding algorithm tags:
\{taxonomy\_text\}

\#\# Output Format
Return a JSON object:
\begin{verbatim}
{
  "category": "general" or "algorithm",
  "primary\_tag": "Primary tag (required only for algorithm)",
  "secondary\_tag": "Secondary tag (required only for algorithm)",
  "reason": "Brief justification"
}
\end{verbatim}

Note:
- If the category is general, set primary\_tag and secondary\_tag to null.
- If the category is algorithm, you must choose from the tags above (pick the closest match if needed).

\end{promptboxapp}

%% file: appendix_case_study.tex
\section{Case Study}
\label{app:case_study}

Given the substantial length of both the inputs and outputs of \model, we limit the presentation below to the full problem description, \model's tool-use behavior, and the final submission.

\textbf{Problem}

\begin{promptboxapp}
\paragraph{Alea Iacta Est}
\url{https://qoj.ac/contest/1662/problem/8684}

\noindent\textbf{Time limit:} 10000\,ms \\
\textbf{Memory limit:} 1024\,MB

\paragraph{Description}
You play a game with multiple fair six-sided dice. Each die's face displays a single symbol. The objective of the game is to roll the dice and create a valid word from the symbols on top of each die. If you cannot form a word, you may reroll the dice for another attempt.

Suppose there are five dice: one of them contains letters A, B, C, D, E, and P (abbreviated as \texttt{ABCDEP}), and the other dice contain \texttt{AEHOXU}, \texttt{AISOLR}, \texttt{ABCDEF}, and \texttt{ABCSCC}. The first roll yields the following letters on the tops of respective dice: $P, X, R, E$, and $S$. As it is impossible to arrange these letters into a valid word, you decide to keep the $P, S$, and $E$, and reroll the other dice, in an attempt to make words like \texttt{PARSE}, \texttt{PAUSE}, \texttt{PHASE}, \texttt{POISE}, \texttt{PROSE}, \texttt{PULSE}, or \texttt{PURSE}. The two dice yield $E$ and $A$, resulting in the following five letters: $P, E, A, E$, and $S$. You still cannot think of a valid word, so you decide to keep four letters and reroll only the last die, which has three sides with letter \texttt{C}. By doing so, there is a $50\%$ chance that it will be possible to make a final valid word: \texttt{PEACE}, as shown in Figure K.1.

When you roll a die, it lands on any one of its faces with equal probability. What is the expected number of rolls needed to make a valid word, assuming you use an optimal strategy?

\paragraph{Input}
The first line of input contains two numbers $d$ and $w$, where $d\ (1 \le d \le 6)$ is the number of dice and
$w\ (1 \le w \le 2 \cdot 10^5)$ is the number of valid words in the dictionary.
The following $d$ lines each have $6$ symbols, one for each face of the die.
The final $w$ lines contain $w$ distinct valid words in the dictionary.
Every word has exactly $d$ symbols.

All symbols in the input are either uppercase letters (\texttt{A--Z}) or digits (\texttt{0--9}).

\paragraph{Output}
If it is possible to make a valid word, output the expected number of rolls needed to make a valid word when using an optimal strategy. Otherwise, output \texttt{impossible}. Your answer should have an absolute or relative error of at most $10^{-6}$.

\paragraph{Sample Input 1}
\begin{verbatim}
5 8
ABCDEP
AEHOXU
AISOLR
ABCDEF
ABCSCC
PARSE
PAUSE
PHASE
POISE
PROSE
PULSE
PURSE
PEACE
\end{verbatim}

\paragraph{Sample Output 1}
\begin{verbatim}
9.677887141
\end{verbatim}

\paragraph{Sample Input 2}
\begin{verbatim}
2 1
AAAAAA
BBBBBB
AB
\end{verbatim}

\paragraph{Sample Output 2}
\begin{verbatim}
1.0
\end{verbatim}

\paragraph{Sample Input 3}
\begin{verbatim}
3 1
123456
123456
123456
666
\end{verbatim}

\paragraph{Sample Output 3}
\begin{verbatim}
10.555444555
\end{verbatim}

\paragraph{Sample Input 4}
\begin{verbatim}
2 1
ABCDEF
GHI234
AB
\end{verbatim}

\paragraph{Sample Output 4}
\begin{verbatim}
impossible
\end{verbatim}

\end{promptboxapp}
\textbf{Solution}

\begin{promptboxapp}
Generating test case...

\noindent\textbf{Thought} \\
I need to carefully analyze this problem. This is a game involving dice and words, and we need to compute the expected number of rolls under an optimal strategy.

Problem Analysis: \\
\ldots

\noindent Let us retrieve prior experience on similar problems.

\noindent This problem involves minimizing an expectation and searching over a state space, which likely falls under dynamic programming or a Markov decision process. Let me retrieve experience under the tag ``DP -- Probability DP''.

\begin{verbatim}
<code>
exp = retrieve_algorithm_experience(primary_tag="Dynamic programming", secondary_tag="Probabilistic DP")
print(exp)
</code>
\end{verbatim}
{\color{green!55!black}\ttfamily
Tool\_Response: When solutions based on complex state transitions (e.g., graph DP) are difficult to debug, one should revisit the essence of the problem and try to find a simpler mathematical or counting structure that may completely bypass the original complicated logic.
}

\noindent\textbf{Thought:} The experience may be useful later. Let me continue analyzing the problem. Since $d \le 6$, the state space is finite, so I can try value iteration, but the computation needs to be optimized.

\noindent I will first write a small program to validate some ideas and compute the number of states and possible actions.
\noindent\textbf{Thought:} The C++ code contains a syntax error. Let me fix it.

\begin{verbatim}
cpp_code = '''
#include <bits/stdc++.h>
using namespace std;

int main() {
    int d = 6;
    long long states = 1;
    for (int i = 0; i < d; i++) states *= 6;
    cout << "Maximum number of configs: " << states << endl;
    
    // Number of possible retain subsets (non-empty complement)
    // R can be any subset except the full set (all dice retained)
    // Actually R can be any subset, including the empty set (reroll all)
    // But if R = full set, then no reroll; the game stops if winning,
    // otherwise it is impossible. Hence, for a non-winning configuration,
    // R cannot be the full set.
    // So the number of actions = 2^d - 1 (excluding the full set).
    int actions = (1 << d) - 1;
    cout << "Maximum number of actions per config: " << actions << endl;
    
    // For each action, the number of successor configurations is 6^{|U|},
    // where U is the complement of R.
    // The average |U| = d/2, so the average number of successors is
    // 6^{d/2} \approx 6^3 = 216.
    // Total computations per iteration:
    // states * actions * successors
    // 46656 * 63 * 216 
    // 6.35 \times 10^8,
    // which is large but may be manageable with optimization.
    double avg_successors = pow(6, d / 2.0);
    cout << "Approximate total successors per iteration: "
         << states * actions * avg_successors << endl;
    
    return 0;
}
'''
result = cpp_interpreter(cpp_code)
print(result)
\end{verbatim}

\begin{verbatim}
<code>
cpp_code = '''
#include <bits/stdc++.h>
using namespace std;

int main() {
    int d = 6;
    long long states = 1;
    for (int i = 0; i < d; i++) states *= 6;
    cout << "Maximum number of configs: " << states << endl;
    
    // Number of possible retain subsets (non-empty complement)
    // R can be any subset except full set (all dice retained)
    // Actually R can be any subset, including empty set (reroll all)
    // But if R = full set, then no reroll, game stops if winning, else impossible
    // So for non-winning config, R cannot be full set.
    // So number of actions = 2^d - 1 (excluding full set)
    int actions = (1 << d) - 1;
    cout << "Maximum number of actions per config: " << actions << endl;
    
    // For each action, number of successor configs = 6^{|U|}
    // where U = complement of R
    // Average |U| = d/2, so average successors = 6^{d/2} 
    // 6^3=216
    // Total computations per iteration: states * actions * successors
    // 46656 * 63 * 216 
    // 635M, too large but maybe manageable with optimization?
    cout << "Approximate total successors per iteration: " 
         << states * actions * pow(6, d/2.0) << endl;
    
    return 0;
}
'''
result = cpp_interpreter(cpp_code)
print(result)
</code>
\end{verbatim}

{\color{green!55!black}\ttfamily
Tool\_Response: Error: Sandbox error: \\
Compile stderr: /tmp/tmpx95y45qt/tmp6iulsg2f.cpp: In function ‘int main()’: \\
/tmp/tmpx95y45qt/tmp6iulsg2f.cpp:20:10: error: expected primary-expression before ‘<<’ token \\
20 \ \ \ \ \ \ \ \ \ \ \textasciicircum{}
}

\noindent\textbf{Thought:} The C++ code contains a syntax error. Let me fix it.

\begin{verbatim}
<code>
cpp_code = '''
#include <bits/stdc++.h>
using namespace std;

int main() {
    int d = 6;
    long long states = 1;
    for (int i = 0; i < d; i++) states *= 6;
    cout << "Maximum number of configs: " << states << endl;
    
    // Number of possible retain subsets (non-empty complement)
    // R can be any subset except the full set (all dice retained)
    // Actually R can be any subset, including the empty set (reroll all)
    // But if R = full set, then no reroll; the game stops if winning,
    // otherwise it is impossible. Hence, for a non-winning configuration,
    // R cannot be the full set.
    // So the number of actions = 2^d - 1 (excluding the full set).
    int actions = (1 << d) - 1;
    cout << "Maximum number of actions per config: " << actions << endl;
    
    // For each action, the number of successor configurations is 6^{|U|},
    // where U is the complement of R.
    // The average |U| = d/2, so the average number of successors is
    // 6^{d/2} \approx 6^3 = 216.
    // Total computations per iteration:
    // states * actions * successors
    // \approx 46656 * 63 * 216 \approx 6.35 \times 10^8,
    // which is large but may be manageable with optimization.
    double avg_successors = pow(6, d / 2.0);
    cout << "Approximate total successors per iteration: "
         << states * actions * avg_successors << endl;
    
    return 0;
}
'''
result = cpp_interpreter(cpp_code)
print(result)
</code>
\end{verbatim}
{\color{green!55!black}\ttfamily
Tool\_Response:
Maximum number of configs: 46656\\
Maximum number of actions per config: 63\\
Approximate total successors per iteration: 6.34895e+08
}

\noindent\textbf{Thought:} $6.35\times 10^8$ computations per iteration is likely too slow for value iteration. We need optimization.

...

\noindent\textbf{Implementation plan:}
\begin{enumerate}
  \item Enumerate all full configurations ($6^d$). For each configuration id, store its face array (characters) and a win/lose flag.
  \item Precompute $\mathrm{sum}[\mathrm{mask}][\mathrm{proj}]$: for each mask $\mathrm{mask}\in[0,2^d-1]$, the projection $\mathrm{proj}$ is a length-$|\mathrm{mask}|$ array of face indices. Encode $\mathrm{proj}$ as an integer in base-$6$: for the positions $i$ in \texttt{mask} (in order), with face index $f_i\in\{0,\dots,5\}$, encode as a base-$6$ number.
  \item Initialize \texttt{dp}: winning configurations have \texttt{dp}=0; losing ones have \texttt{dp}=\texttt{INF} (or a large number).
  \item Value iteration:
  \begin{enumerate}
    \item Update \texttt{sum} from the current \texttt{dp}.
    \item For each non-winning configuration $c$, compute
    \begin{align*}
    \mathrm{dp\_new}[c] = \min_{\mathrm{mask}\neq \mathrm{full}}
    \!\!\bigg[
      1 \!+\! \tfrac{\mathrm{sum}[\mathrm{mask}][\mathrm{proj}_c]}{6^{d-|\mathrm{mask}|}}
    \!\bigg].
    \end{align*}
    \item Check convergence.
  \end{enumerate}
\end{enumerate}

\noindent\textbf{Complexity.}
Each iteration updates \texttt{sum}: for each configuration $c$ and each mask ($2^d$), compute the projection and accumulate $\mathrm{dp}[c]$. This costs $6^d\cdot 2^d = 46656\cdot 64 \approx 3\times 10^6$ operations, acceptable.

Then for each non-winning configuration $c$, computing \texttt{dp\_new} scans all masks ($2^d-1$): another $46656\cdot 63\approx 3\times 10^6$ operations.

So each iteration is about $6\times 10^6$ ops. If value iteration converges in tens of iterations, the total is a few hundred million ops, acceptable.

\noindent\textbf{Memory.}
\texttt{sum} has size $7^d=117649$ doubles. \texttt{dp} has size $46656$.

\noindent Now I need to validate the algorithm. I will first implement a small-$d$ test.

\noindent First, I need to handle word matching. For each full configuration $c$, I need to check whether its multiset of symbols matches any word. Since $d\le 6$, I can sort the symbols of the configuration and compare against sorted strings of the words.

\noindent\textbf{Preprocessing:} Convert each word into its sorted string and store them in an \texttt{unordered\_set}.

\noindent For each configuration $c$, collect the $d$ symbols, sort them, concatenate into a string, and test membership in the set.

    \noindent Now let me write a C++ test to verify the bookkeeping logic.

\begin{verbatim}
<code>
cpp_code = '''
#include <bits/stdc++.h>
using namespace std;

int main() {
    // Test small case: d=1
    int d = 1;
    int num_configs = pow(6, d); // 6
    int num_masks = 1 << d; // 2
    int num_partial = pow(7, d); // 7
    
    cout << "d=" << d << ": configs=" << num_configs;
    cout << ", masks=" << num_masks;
    cout << ", partial=" << num_partial << endl;
    
    d = 2;
    num_configs = pow(6, d); // 36
    num_masks = 1 << d; // 4
    num_partial = pow(7, d); // 49
    cout << "d=" << d << ": configs=" << num_configs;
    cout << ", masks=" << num_masks;
    cout << ", partial=" << num_partial << endl;
    
    d = 3;
    num_configs = pow(6, d); // 216
    num_masks = 1 << d; // 8
    num_partial = pow(7, d); // 343
    cout << "d=" << d << ": configs=" << num_configs;
    cout << ", masks=" << num_masks;
    cout << ", partial=" << num_partial << endl;
    
    d = 6;
    num_configs = pow(6, d); // 46656
    num_masks = 1 << d; // 64
    num_partial = pow(7, d); // 117649
    cout << "d=" << d << ": configs=" << num_configs;
    cout << ", masks=" << num_masks;
    cout << ", partial=" << num_partial << endl;
    
    return 0;
}
'''
result = cpp_interpreter(cpp_code)
print(result)
</code>
\end{verbatim}

{\color{green!55!black}\ttfamily
Tool\_Response: d=1: configs=6, masks=2, partial=7\\
d=2: configs=36, masks=4, partial=49\\
d=3: configs=216, masks=8, partial=343\\
d=6: configs=46656, masks=64, partial=117649
}

\noindent\textbf{Thought:} In Sample 4, the output is $6001$ instead of ``impossible''. This happens because in my algorithm, even when the game is impossible, it still returns a finite value. The issue is that when it is impossible, the \texttt{dp} value converges to a very large number, but my threshold $10^{12}$ is not large enough. Since Sample 4 returns $6001 < 10^{12}$, it was not recognized as impossible.

\noindent However, according to the statement, Sample 4 should output ``impossible''. Let me analyze this:
\ldots

\noindent I need to modify the algorithm to correctly handle the impossible case.
\ldots

\noindent Therefore, I need to detect whether there exists at least one winning configuration. If not, directly output ``impossible''.

\noindent Let me add this check.

\begin{verbatim}
<code>
cpp_code = '''
#include <bits/stdc++.h>
using namespace std;

...
        cout << "Sample 4: ";
        if (ans >= INF/2) {
            cout << "impossible" << endl;
        } else {
            cout << ans << endl;
        }
        cout << "Expected: impossible" << endl;
    }
    
    return 0;
}
'''
result = cpp_interpreter(cpp_code)
print(result)
</code>
\end{verbatim}

    {\color{green!55!black}\ttfamily
Tool\_Response: Sample 1: 9.677887141 (expected 9.677887141)\\
Sample 2: 1.000000000 (expected 1.0)\\
Sample 3: 10.555444555 (expected 10.555444555)\\
Sample 4: impossible\\
Expected: impossible
}

\noindent\textbf{Thought:} Perfect! All samples pass. Now I need to write the final solution, handle I/O, and ensure efficiency.

\noindent Let me summarize the final algorithm:
\ldots

\noindent Details to note:
\begin{itemize}
  \item Use \texttt{double}; a precision of $10^{-6}$ is sufficient.
  \item At most $5000$ value-iteration steps should be enough to converge.
  \item Memory: \texttt{dp} has size $6^d \le 46656$, and \texttt{sum} has size $7^d \le 117649$, which is acceptable.
  \item Time complexity: each iteration is $O(6^d\cdot 2^d)\approx 3\times 10^6$, multiplied by the number of iterations (about $150$--$200$), giving roughly $5\times 10^8$ operations, which should fit within a 10-second limit after C++ optimization.
\end{itemize}

\noindent Now I will write the final code and test it with the validator.

\begin{verbatim}
<code>
final_code = '''
#include <bits/stdc++.h>
using namespace std;

const double INF = 1e18;
const double EPS = 1e-12;

int main() {
...
        // Compute dp_new
        max_diff = 0.0;
        for (int c = 0; c < total_configs; c++) {
            if (win[c]) {
                dp_new[c] = 0.0;
                continue;
            }
            double best = INF;
            int full_mask = (1 << d) - 1;
            for (int R = 0; R < full_mask; R++) {
                int U = full_mask ^ R;
                int k = popcount[U];
                int proj_id = getProjId(c, R);
                int idx = offset[R] + proj_id;
                double expected = 1.0 + sum[idx] / pow6[k];
                if (expected < best) best = expected;
            }
            dp_new[c] = best;
            max_diff = max(max_diff, fabs(dp_new[c] - dp[c]));
        }
        
        swap(dp, dp_new);
        iter++;
    } while (max_diff > EPS && iter < 5000);
    
    // Compute initial expected rolls
    double total = 0.0;
    for (int c = 0; c < total_configs; c++) {
        total += dp[c];
    }
    double ans = 1.0 + total / pow6[d];
    
    cout << fixed << setprecision(9) << ans << endl;
    return 0;
}
'''

result = cpp_validation(final_code)
print(result)
</code>
\end{verbatim}

{\color{green!55!black}\ttfamily
Tool\_Response\\
ALL TESTS PASSED (4/4)\\
==================================================\\
Test 1: PASSED\\
Test 2: PASSED\\
Test 3: PASSED\\
Test 4: PASSED
}

\noindent\textbf{Thought:} All tests passed! Now I can submit the final answer.

\begin{verbatim}
<code>
final_answer(final_code)
</code>
\end{verbatim}
\end{promptboxapp}
\textbf{Submission}
\begin{promptboxapp}
\begin{verbatim}
#include <bits/stdc++.h>
using namespace std;

const double INF = 1e18;
const double EPS = 1e-12;

int main() {
    ios::sync_with_stdio(false);
    cin.tie(nullptr);
    
    int d, w;
    cin >> d >> w;
    
    vector<string> dice(d);
    for (int i = 0; i < d; i++) {
        cin >> dice[i];
        // each dice[i] has exactly 6 characters
    }
    
    vector<string> words(w);
    for (int i = 0; i < w; i++) {
        cin >> words[i];
    }
    
    // Preprocess words: sort characters
    unordered_set<string> wordSet;
    for (string& wrd : words) {
        string sorted = wrd;
        sort(sorted.begin(), sorted.end());
        wordSet.insert(sorted);
    }
    
    int total_configs = pow(6, d);
    
    // faces[config][i] = face index of die i
    vector<vector<int>> faces(total_configs, vector<int>(d));
    vector<bool> win(total_configs, false);
    bool any_win = false;
    
    for (int id = 0; id < total_configs; id++) {
        int tmp = id;
        vector<char> symbols(d);
        for (int i = 0; i < d; i++) {
            int f = tmp % 6;
            tmp /= 6;
            faces[id][i] = f;
            symbols[i] = dice[i][f];
        }
        string sorted_sym(symbols.begin(), symbols.end());
        sort(sorted_sym.begin(), sorted_sym.end());
        if (wordSet.count(sorted_sym)) {
            win[id] = true;
            any_win = true;
        }
    }
    
    if (!any_win) {
        cout << "impossible" << endl;
        return 0;
    }
    
    // DP arrays
    vector<double> dp(total_configs, 0.0);
    vector<double> dp_new(total_configs, 0.0);
    for (int i = 0; i < total_configs; i++) {
        dp[i] = win[i] ? 0.0 : 1000.0; // initial guess
    }
    
    // Precompute popcount and powers of 6
    vector<int> popcount(1 << d, 0);
    for (int m = 0; m < (1 << d); m++) {
        popcount[m] = __builtin_popcount(m);
    }
    vector<double> pow6(d + 1, 1.0);
    for (int i = 1; i <= d; i++) {
        pow6[i] = pow6[i - 1] * 6.0;
    }
    
    // Precompute offsets for sum array
    int total_partial = pow(7, d);
    vector<double> sum(total_partial, 0.0);
    
    vector<int> offset(1 << d, 0);
    int cur_offset = 0;
    for (int m = 0; m < (1 << d); m++) {
        offset[m] = cur_offset;
        cur_offset += pow(6, popcount[m]);
    }
    \end{verbatim}

    \begin{verbatim}
    // Helper: compute projection id for config id and mask
    auto getProjId = [&](int config_id, int mask) -> int {
        int proj_id = 0;
        int base = 1;
        for (int i = 0; i < d; i++) {
            if (mask >> i & 1) {
                int face = faces[config_id][i];
                proj_id += face * base;
                base *= 6;
            }
        }
        return proj_id;
    };
    
    // Value iteration
    int iter = 0;
    double max_diff;
    do {
        // Update sum array
        fill(sum.begin(), sum.end(), 0.0);
        for (int c = 0; c < total_configs; c++) {
            double val = dp[c];
            for (int m = 0; m < (1 << d); m++) {
                int proj_id = getProjId(c, m);
                int idx = offset[m] + proj_id;
                sum[idx] += val;
            }
        }
        
        // Compute dp_new
        max_diff = 0.0;
        for (int c = 0; c < total_configs; c++) {
            if (win[c]) {
                dp_new[c] = 0.0;
                continue;
            }
            double best = INF;
            int full_mask = (1 << d) - 1;
            for (int R = 0; R < full_mask; R++) {
                int U = full_mask ^ R;
                int k = popcount[U];
                int proj_id = getProjId(c, R);
                int idx = offset[R] + proj_id;
                double expected = 1.0 + sum[idx] / pow6[k];
                if (expected < best) best = expected;
            }
            dp_new[c] = best;
            max_diff = max(max_diff, fabs(dp_new[c] - dp[c]));
        }
        
        swap(dp, dp_new);
        iter++;
    } while (max_diff > EPS && iter < 5000);
    
    // Compute initial expected rolls
    double total = 0.0;
    for (int c = 0; c < total_configs; c++) {
        total += dp[c];
    }
    double ans = 1.0 + total / pow6[d];
    
    cout << fixed << setprecision(9) << ans << endl;
    return 0;
}
\end{verbatim}
\end{promptboxapp}